\documentclass[12pt]{article}
\usepackage{amsmath}
\usepackage{times}
\usepackage[dvipdfmx]{graphicx}
\usepackage{color}
\usepackage{multirow}
\usepackage[authoryear]{natbib}
\usepackage{rotating}
\usepackage{bbm}
\usepackage{latexsym}
\usepackage{setspace}

\newcommand{\captionfonts}{\normalsize}

\makeatletter  
\long\def\@makecaption#1#2{%
  \vskip\abovecaptionskip
  \sbox\@tempboxa{{\captionfonts #1: #2}}%
  \ifdim \wd\@tempboxa >\hsize
    {\captionfonts #1: #2\par}
  \else
    \hbox to\hsize{\hfil\box\@tempboxa\hfil}%
  \fi
  \vskip\belowcaptionskip}
\makeatother   

\usepackage{subfigure} 

\usepackage{algorithm}
\usepackage{algorithmic}
\usepackage{comment}
\usepackage{amsmath,amsthm,amssymb}
\usepackage{lineno}
\usepackage{bm}
\usepackage{comment}
 
\newtheorem{theo}{Theorem}
\newtheorem{coro}{Corollary}
\newtheorem{lemma}{Lemma}

\newtheorem{assumption}{Assumption}
 
\newcommand{\cd}{\cdot}
\newcommand{\tx}{\tilde{x}}
\newcommand{\ty}{\tilde{y}}

\newcommand{\X}{{\mathcal{X}}}
\newcommand{\Y}{{\mathcal{Y}}}

\newcommand{\E}{{\bf{E}}}

\newcommand{\R}{\mathbb{R}}
\renewcommand{\H}{{\mathcal{H}}}

\newcommand{\tX}{{\tilde{X}}}

\renewcommand{\hm}{\hat{m}}
\newcommand{\bX}{\bar{X}}
\newcommand{\bY}{\bar{Y}}

\newcommand{\argmax}{\mathop{\rm arg~max}\limits}
\newcommand{\argmin}{\mathop{\rm arg~min}\limits}

\begin{document}
\hspace{13.9cm}1

\ \vspace{20mm}\\


{\LARGE Filtering with State-Observation Examples\\ via Kernel Monte Carlo Filter}

\ \\
{\bf \large Motonobu Kanagawa$^{\displaystyle 1, 2}$},\ {\bf \large Yu Nishiyama$^{\displaystyle 3}$},\ {\bf \large Arthur Gretton$^{\displaystyle 4}$},\\ {\bf \large Kenji Fukumizu$^{\displaystyle 1,\displaystyle 2}$}\\
{$^{\displaystyle 1}$SOKENDAI (The Graduate University for Advanced Studies), Tokyo.}\\
{$^{\displaystyle 2}$The Institute of Statistical Mathematics, Tokyo.}\\
{$^{\displaystyle 3}$The University of Electro-Communications, Tokyo.}\\
{$^{\displaystyle 4}$Gatsby Computational Neuroscience Unit, University College London.}\\
%

{\bf Keywords:} state-space models, filtering, kernel methods, kernel mean embeddings

\thispagestyle{empty}
\markboth{}{NC instructions}
\ \vspace{-0mm}\\
%
\begin{center} {\bf Abstract} \end{center}
This paper addresses the problem of filtering with a state-space model. 
Standard approaches for filtering assume that a probabilistic model for observations (i.e.\ the observation model) is given explicitly or at least parametrically. 
We consider a setting where this assumption is not satisfied; we assume that the knowledge of the observation model is only provided by examples of state-observation pairs.
This setting is important and appears when state variables are defined as quantities that are very different from the observations.
We propose Kernel Monte Carlo Filter, a novel filtering method that is focused on this setting.
Our approach is based on the framework of kernel mean embeddings, which enables nonparametric posterior inference using the state-observation examples.
The proposed method represents state distributions as weighted samples, propagates these samples by sampling, estimates the state posteriors by Kernel Bayes' Rule, and resamples by Kernel Herding. 
In particular, the sampling and resampling procedures are novel in being expressed using kernel mean embeddings, so we theoretically analyze their behaviors.
We reveal the following properties, which are similar to those of corresponding procedures in particle methods: (1) the performance of sampling can degrade if the effective sample size of a weighted sample is small; (2)  resampling improves the sampling performance by increasing the effective sample size.
We first demonstrate these theoretical findings by synthetic experiments. 
Then we show the effectiveness of the proposed filter by artificial and real data experiments, which include vision-based mobile robot localization.

\section{Introduction} \label{sec:intro}
Time-series data are ubiquitous in science and engineering. We often wish to extract useful information from such time-series data.
{\em State-space models} have been one of the most successful approaches for this purpose (see, e.g.,\ \cite{DurKoo12}).
Suppose that we have a sequence of observations $y_1,\dots,y_t,\dots,y_T$.
A state-space model assumes that for each observation $y_t$, there is a hidden state $x_t$ that generates $y_t$, and that these states $x_1, \dots, x_t,\dots,x_T$ follow a Markov process (see Figure \ref{fig:problem_setting}).
Therefore the state-space model is characterized by two components: (1) {\em observation model} $p(y_t|x_t)$,  the conditional distribution of an observation given a state, and (2) {\em transition model} $p(x_t|x_{t-1})$, the conditional distribution of a state given the previous one.

\begin{figure}[t]
\begin{center}
	\includegraphics[width=0.50\columnwidth, clip]{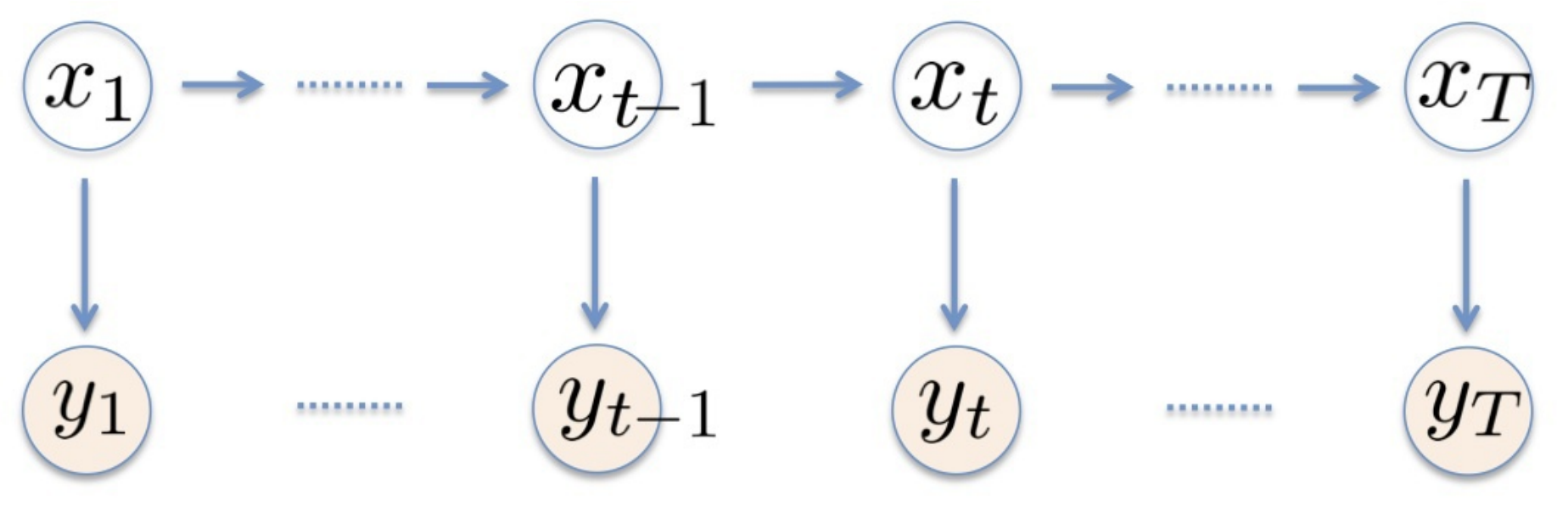}
\caption{Graphical representation of a state-space model: $y_1,\dots,y_T$ denote observations, and $x_1,\dots,x_T$ denote states. The states are hidden, and to be estimated from the observations.}
\label{fig:problem_setting}
\end{center}
\end{figure}

This paper addresses the problem of {\em filtering}, which has been a central topic in the literature on state-space models.
The task is to estimate a posterior distribution of the state for each time $t$,  based on observations up to that time:
\begin{equation} \label{eq:posterior_intro}
 p(x_t | y_1,\dots,y_t),\quad t=1,2,\dots, T.
\end{equation}
The estimation is to be done online (sequentially),  as each $y_t$ is received.
For example, a tracking problem can be formulated as filtering, where $x_t$ is the position of an object to be tracked, and $y_t$ is a noisy observation of $x_t$ \citep{RisAruGor04}.

As an inference problem, the starting point of filtering is that the observation model $p(y_t | x_t) $ and the transition model $p(x_t | x_{t-1})$ are {\em given} in some form.
The simplest form is a liner-Gaussian state-space model, which enables analytic computation of the posteriors; this is the principle of the classical Kalman filter \citep{Kal60}.
The filtering problem is more difficult if the observation and transition models involve nonlinear-transformation and non-Gaussian noise.
Standard solutions for such situations include Extended and Unscented Kalman filters \citep{AndMoo79,JulUhl97,JulUhl04} and particle filters \citep{GorSalSmi93,Doucet2001,DouJoh11}.
Particle filters in particular have wide applicability since they only require that (i)  (unnormalized) density values of the observation model are computable, and that (ii) sampling with the transition model is possible.
Thus particle methods are applicable to basically any nonlinear non-Gaussian state-space models, and have been used in various fields such as computer vision, robotics, computational biology, and so on (see, e.g.,\ \cite{Doucet2001}).

However, it can even be restrictive to assume that the observation model $p(y_t | x_t)$ is given as a probabilistic model. 
An important point is that in practice, we may define the states $x_1,\dots,x_T$ arbitrarily as quantities that we wish to estimate from available observations $y_1,\dots,y_T$.
Thus if these quantities are very different from the observations, the observation model may not admit a simple parametric form.
For example, in location estimation problems in robotics, states are locations in a map, while observations are sensor data, such as camera images and signal strength measurements of a wireless device \citep{Vlassis2001,WolBurBur05,Ferris2006}.
In brain computer interface applications, states are defined as positions of a device to be manipulated, while observations are brain signals \citep{Pistohl_etal2008,WanJiMilSch11}.
In these applications, it is hard to define the observation model as a probabilistic model in parametric form.

For such applications where the observation model is very complicated, information about the relation between states and observations is rather given as {\em examples} of state-observation pairs $\{ (X_i,Y_i )\}$; such examples are often available {\em before} conducting filtering in test phase.
For example, one can collect location-sensor examples for the location estimation problems, by making use of more expensive sensors than those for filtering \citep{QuiStaCoaThr10}.
The brain computer interface problems also allow us to obtain training samples for the relation between device positions and brain signals \citep{Schalk_etal2007}.
However, making use of such examples for learning the observation model is not straightforward.
If one relies on a parametric approach, it would require exhaustive efforts for designing a parametric model to fit the complicated (true) observation model.  
Nonparametric methods such as kernel density estimation \citep{Sil86}, on the other hand, suffer from the curse of dimensionality when applied to high-dimensional observations.
Moreover, observations may be suitable to be represented as {\em structured} (non-vectorial) data, as for the cases of image and text. 
Such situations are not straightforward for either approach, since they usually require that data is given as real vectors. 

\paragraph{Kernel Monte Carlo Filter.}
In this paper, we propose a filtering method that is focused on the above situations where the information of the observation model $p(y_t|x_t)$ is only given through the state-observation examples $\{ (X_i,Y_i) \}$.
The proposed method, which we call the {\em Kernel Monte Carlo Filter (KMCF)}, is applicable when the following are satisfied:
\begin{enumerate}
\item Positive definite kernels (reproducing kernels) are defined on the states and observations.
Roughly, a positive definite kernel is a similarity function that takes two data points as input, and outputs their similarity value.

\item Sampling with the transition model $p(x_t | x_{t-1})$ is possible. This is the same assumption as for standard particle filters: the probabilistic model can be arbitrarily nonlinear and non-Gaussian. 
\end{enumerate}

The last decades of research on kernel methods have yielded numerous kernels, not only for real vectors, but also for structured data of various types \citep{SchSmo02,HofSchSmo08}. Examples include kernels for images in computer vision \citep{Lazebnik2006}, graph structured data in bioinformatics \citep{SchTsuVer04}, and genomic sequences \citep{Sch10a,Sch10b}.
Therefore we can apply KMCF to such structured data by making use of the kernels developed in these fields. 
On the other hand, this paper assumes that the transition model is given explicitly: we do not discuss parameter learning (for the case of a parametric transition model), and assume that parameters are fixed.


KMCF is based on probability representations provided by the framework of {\em kernel mean embeddings}, which is a recent development in the fields of kernel methods \citep{SmoGreSonSch07,SriGreFukSchetal10,SonFukGre13}.
In this framework, any probability distribution is represented as a uniquely associated function in a reproducing kernel Hilbert space (RKHS), which is known as a {\em kernel mean}.
This representation enables us to estimate a distribution of interest, by alternatively estimating the corresponding kernel mean.
One significant feature of kernel mean embeddings is {\em Kernel Bayes' Rule} \citep{FukSonGre11,FukSonGre13}, by which KMCF estimates posteriors based on the state-observation examples.
Kernel Bayes' Rule has the following properties: 
(a) It is theoretically grounded and is proven to get more accurate as the number of the examples increases; 
(b) It requires neither parametric assumptions nor heuristic approximations for the observation model;
(c) Similarly to other kernel methods in machine learning, Kernel Bayes' Rule is empirically known to perform well for high-dimensional data, when compared to classical nonparametric methods.
KMCF inherits these favorable properties.

KMCF sequentially estimates the RKHS representation of the posterior (\ref{eq:posterior_intro}), in the form of weighted samples. This estimation consists of three steps of {\em prediction}, {\em correction} and {\em resampling}.
Suppose that we already obtained an estimate for the posterior of the previous time. 
In the prediction step, this previous estimate is propagated forward by sampling with the transition model, in the same manner as the sampling procedure of a particle filter. The propagated estimate is then used as a prior for the current state.
In the correction step, Kernel Bayes' Rule is applied to obtain a posterior estimate, using the prior and the state-observation examples $\{(X_i,Y_i)\}_{i=1}^n$.
Finally, in the resampling step, an approximate version of {\em Kernel Herding} \citep{CheWelSmo10} is applied, to obtain pseudo samples from the posterior estimate.
Kernel Herding is a greedy optimization method to generate pseudo samples from a given kernel mean, and searches for those samples from the entire space $\X$.
Our resampling algorithm modifies this, and searches for pseudo samples from a finite candidate set of the state samples $\{X_1,\dots,X_n \} \subset \X$.
The obtained pseudo samples are then used in the prediction step of the next iteration.

While the KMCF algorithm is inspired by particle filters, there are several important differences: 
(i) A weighted sample expression in KMCF is an estimator of the RKHS representation of a probability distribution, while that of a particle filter represents an empirical distribution. This difference can be seen in the following fact: weights of KMCF can take negative values, while weights of a particle filter are always positive. 
(ii) To estimate a posterior, KMCF uses the state-observation examples $\{(X_i,Y_i)\}_{i=1}^n$ and does not require the observation model itself, while a particle filter makes use of the observation model to update weights. 
In other words, KMCF involves nonparametric estimation of the observation model, while a particle filter does not.
(iii) KMCF achieves resampling based on Kernel Herding, while a particle filter uses a standard resampling procedure with an empirical distribution. We use Kernel Herding because the resampling procedure of particle methods is not appropriate for KMCF,  as the weights in KMCF may take negative values.

Since the theory of particle methods cannot therefore be used to justify our approach, we conduct the following theoretical analysis:
\begin{itemize}
\item We derive error bounds for the sampling procedure in the prediction step (Section \ref{sec:upper_bound}): this justifies the use of the sampling procedure with weighted sample expressions of kernel mean embeddings. The bounds are not trivial, since the weights of kernel mean embeddings can take negative values.

\item We discuss how resampling works with kernel mean embeddings (Section \ref{sec:theory_resampling}): it improves the estimation accuracy of the subsequent sampling procedure, by increasing the effective sample size of an empirical kernel mean.
This mechanism is essentially the same as that of a particle filter.

\item We provide novel convergence rates of Kernel Herding, when pseudo samples are searched from a finite candidate set (Section \ref{sec:rates_resampling}): this justifies our resampling algorithm.
This result may be of independent interest to the kernel community, as it describes how Kernel Herding is often used in practice.

\item We show the consistency of the overall filtering procedure of KMCF under certain smoothness assumptions (Section \ref{sec:consistency_KMCF}): KMCF provides consistent posterior estimates, as the number of state-observation examples $\{ (X_i,Y_i) \}_{i=1}^n$ increases.
\end{itemize}

The rest of the paper is organized as follows. 
In Section \ref{sec:related}, we review related works. 
Section \ref{sec:background} is devoted to preliminaries to make the paper self-contained; we review the theory of kernel mean embeddings.
Section \ref{sec:filter} presents Kernel Monte Carlo Filter, and
Section \ref{sec:theory} shows theoretical results.
In Section \ref{sec:experiment}, we demonstrate the effectiveness of KMCF by artificial and real data experiments.
The real experiment is on vision-based mobile robot localization, which is an example of the location estimation problems mentioned above.
Appendices include two methods for reducing computational costs of KMCF.

This paper expands on a conference paper by \cite{KanNisGreFuk14}.
The present paper differs from this earlier work in that it introduces and justifies the use of Kernel Herding for resampling.
The resampling step allows us to control the effective sample size of an empirical kernel mean, which is an important factor that determines the accuracy of the sampling procedure, as in particle methods.

\section{Related work}	\label{sec:related}

As explained, we consider the following setting: (i) the observation model $p(y_t | x_t)$ is not known explicitly or even parametrically. Instead, state-observation examples $\{ (X_i,Y_i) \}$ are available before test phase; (ii) sampling from the transition model $p(x_t | x_{t-1})$ is possible.
Note that standard particle filters cannot be applied to this setting directly, since they require that the observation model is given as a parametric model. 

As far as we know, there exist a few methods that can be applied to this setting directly \citep{Vlassis2001,Ferris2006}.
These methods learn the observation model from state-observation examples nonparametrically, and then use it to run a particle filter with a transition model.
\cite{Vlassis2001} proposed to apply conditional density estimation based on the $k$-nearest neighbors approach \citep{Sto77} for learning the observation model. 
A problem here is that conditional density estimation suffers from the curse of dimensionality if observations are high-dimensional \citep{Sil86}. 
\cite{Vlassis2001} avoided this problem by estimating the conditional density function of the state given observation, and used it as an alternative for the observation model.
This heuristic may introduce bias in estimation, however.
\cite{Ferris2006} proposed to use Gaussian Process regression for leaning the observation model.
This method will perform well if the Gaussian noise assumption is satisfied, but cannot be applied to structured observations.


\paragraph{Related settings.}
There exist related but different problem settings from ours.
One situation is that examples for state transitions are also given, and the transition model is to be learned nonparametrically from these examples.
For this setting, there are methods based on kernel mean embeddings \citep{song2009,FukSonGre11,FukSonGre13} and Gaussian Processes \citep{KoFox09,DeiHubHan09}.
The filtering method by \cite{FukSonGre11,FukSonGre13} is in particular closely related to KMCF, as it also uses Kernel Bayes' Rule.
A main difference from KMCF is that it computes forward probabilities by Kernel Sum Rule \citep{song2009,SonFukGre13}, which nonparametrically learns the transition model from the state transition examples.
While the setting is different from ours, we compare KMCF with this method in our experiments as a baseline.

Another related setting is that the observation model itself is given and sampling is possible, but computation of its values is expensive or even impossible. 
Therefore ordinary Bayes' rule cannot be used for filtering. 
To overcome this limitation, \cite{JasSinMarMcC12} and \cite{CalCze2014} proposed to apply approximate Bayesian computation (ABC) methods. For each iteration of filtering, these methods generate state-observation pairs from the observation model. Then they pick some pairs that have close observations to the test observation, and regard the states in these pairs as samples from a posterior. 
Note that these methods are not applicable to our setting, since we do not assume that the observation model is provided.
That said, our method may be applied to their setting, by generating state-observation examples from the observation model.
While such a comparison would be interesting, this paper focuses on comparison among the methods applicable to our setting.

\section{Kernel mean embeddings of distributions}
\label{sec:background}
Here we briefly review the framework of kernel mean embeddings. For details, we refer to the tutorial papers \citep{SmoGreSonSch07,SonFukGre13}. 

\subsection{Positive definite kernel and RKHS}
We begin by introducing positive definite kernels and reproducing kernel Hilbert spaces, details of which can be found in \cite{SchSmo02,Berlinet2004,SteChr2008}.

Let $\X$ be a set, and $k: \X \times \X \to \R$ be a positive definite (p.d.) kernel.\footnote{A symmetric kernel $k: \X \times \X \to \R$ is called {\em positive definite (p.d.)}, if for all $n \in \mathbb{N}$, $c_1,\dots,c_n \in \R$, and $X_1,\dots,X_n \in \X$, we have \[ \sum_{i=1}^n \sum_{j=1}^n c_i c_j k(X_i,X_j) \geq 0.\] }
Any positive definite kernel is uniquely associated with a Reproducing Kernel Hilbert Space (RKHS)   \citep{Aronszajn1950}.
Let $\H$ be the RKHS associated with $k$.
The RKHS $\H$ is a Hilbert space of functions on $\X$, which satisfies the following important properties:
\begin{enumerate}
\item ({\bf feature vector}): $k(\cd,x) \in \H$ for all $x \in \X$.
\item ({\bf reproducing property}): $f(x) = \left< f, k(\cd,x) \right>_\H$ for all $f \in \H$ and $x \in \X$,
\end{enumerate}
where $\left< \cd,\cd \right>_{\H}$ denotes the inner product equipped with $\H$, and $k(\cd,x)$ is a function with $x$ fixed. 
By the reproducing property, we have 
\[ k(x,x') = \left<k(\cd,x), k(\cd,x') \right>_\H, \quad  \forall x, x' \in \X. \]
Namely, $k(x,x')$ implicitly computes the inner product between the functions $k(\cd,x)$ and $k(\cd,x')$. From this property, $k(\cd,x)$ can be seen as an implicit representation of $x$ in $\H$. Therefore $k(\cd,x)$ is called the {\em feature vector} of $x$, and $\H$ the feature space.  It is also known that the subspace spanned by the feature vectors $\{ k(\cd, x) | x \in \X \}$ is dense in $\H$. This means that any function $f$ in $\H$ can be written as the limit of functions of the form $f_n := \sum_{i=1}^n c_i k(\cd,X_i)$, where $c_1,\dots,c_n \in \R$ and $X_1,\dots,X_n \in \X$.

For example, positive definite kernels on the Euclidian space $\X = \R^d$ include Gaussian kernel $k(x,x') = \exp(- \| x - x' \|_2^2 / 2\sigma^2)$ and Laplace kernel $k(x,x') = \exp( - \| x - x \|_1 / \sigma)$, where $\sigma > 0$ and $\| \cd \|_1$ denotes the $\ell_1$ norm.
Notably, kernel methods allow $\X$ to be a set of {\em structured data}, such as images, texts or graphs.
In fact, there exist various positive definite kernels developed for such structured data \citep{HofSchSmo08}. 
Note that the notion of positive definite kernels is {\em different} from smoothing kernels in kernel density estimation \citep{Sil86}: a smoothing kernel does not necessarily define an RKHS.

\subsection{Kernel means} \label{sec:back_kernel_mean} 
We use the kernel $k$ and the RKHS $\H$ to represent {\em probability distributions} on $\X$.
This is the framework of kernel mean embeddings \citep{SmoGreSonSch07}. 
Let $\X$ be a measurable space, and $k$ be measurable and bounded\footnote{$k$ is bounded on $\X$ if $\sup_{x \in \X} k(x,x) < \infty$.} on $\X$.
Let $P$ be an arbitrary probability distribution on $\X$.
Then the representation of $P$ in $\H$ is defined as the mean of the feature vector:
\begin{equation} \label{eq:kenrel_mean_def}
 m_P := \int k(\cd,x)dP(x) \in \H,
\end{equation}
which is called the {\bf kernel mean} of $P$.

If $k$ is {\em characteristic}, the kernel mean (\ref{eq:kenrel_mean_def}) preserves all the information about $P$; a positive definite kernel $k$ is defined to be characteristic, if the mapping $P \to m_P \in \H$ is one-to-one \citep{FukBacJor04,FukGreSunSch08,SriGreFukSchetal10}.
This means that the RKHS is rich enough to distinguish among all distributions. 
For example, the Gaussian and Laplace kernels are characteristic.
For conditions for kernels to be characteristic, see \cite{FukSriGreSch09,SriGreFukSchetal10}.
We assume henceforth that kernels are characteristic. 

An important property of the kernel mean (\ref{eq:kenrel_mean_def}) is the following: by the reproducing property, we have
\begin{equation} \label{eq:kernel_mean_inner}
 \left< m_P, f \right>_\H = \int f(x) dP(x) = \E_{X \sim P}[f(X)], \quad \forall f \in \H.
\end{equation}
That is, the expectation of any function in the RKHS can be given by the inner product between the kernel mean and that function.

\subsection{Estimation of kernel means}
Suppose that distribution $P$ is unknown, and that we wish to estimate $P$ from available samples.
This can be equivalently done by estimating its kernel mean $m_P$, since $m_P$ preserves all the information about $P$.

For example, let $X_1,\dots,X_n$ be an i.i.d.\ sample from $P$. Define an estimator of $m_P$ by the empirical mean:
\[ \hm_P := \frac{1}{n} \sum_{i=1}^n k(\cd,X_i). \]
Then this converges to $m_P$ at a rate $\| \hm_P - m_P \|_\H = O_p(n^{-1/2})$  \citep{SmoGreSonSch07}, where $O_p$ denotes the asymptotic order in probability, and $\| \cd \|_\H$ is the norm of the RKHS: $\| f \|_\H := \sqrt{\left<f,f\right>_\H}$ for all $f \in \H$.
Note that this rate is independent of the dimensionality of the space $\X$.

\paragraph{Kernel Bayes' Rule (KBR)} \label{sec:KBR}
Next we explain Kernel Bayes' Rule, which serves as a building block of our filtering algorithm.
To this end, let us introduce two measurable spaces $\X$ and $\Y$.
Let $p(x,y)$ be a joint probability on the product space $\X \times \Y$ that decomposes as $p(x,y)=p(y|x)p(x)$. 
Let $\pi(x)$ be a prior distribution on $\X$. 
Then the conditional probability $p(y|x)$ and the prior $\pi(x)$ define the posterior distribution by Bayes' rule;
\[p^\pi(x|y) \propto p(y|x) \pi(x).\]

The assumption here is that the conditional probability $p(y|x)$ is unknown. 
Instead, we are given an i.i.d.\ sample $(X_1,Y_1),\dots,(X_n,Y_n)$ from the joint probability $p(x,y)$.
We wish to estimate the posterior $p^\pi(x|y)$ using the sample.
KBR achieves this by estimating the kernel mean of $p^\pi(x|y)$.

KBR requires that kernels be defined on $\X$ and $\Y$. Let $k_\X$ and $k_\Y$ be kernels on $\X$ and $\Y$, respectively.
Define the kernel means of the prior $\pi(x)$ and the posterior $p^\pi(x|y)$:
\[ m_\pi := \int k_\X(\cd,x)\pi(x) dx,\quad m^\pi_{X|y} := \int k_\X(\cd,x) p^\pi(x|y) dx. \]
KBR also requires that $m_\pi$ be expressed as a weighted sample. Let $\hm_\pi := \sum_{j=1}^\ell \gamma_j k_\X(\cd,U_j)$ be a sample expression of $m_\pi$, where $\ell \in \mathbb{N}$, $\gamma_1,\dots,\gamma_\ell \in \R$ and $U_1,\dots,U_\ell \in \X$.
For example, suppose $U_1,\dots,U_\ell$ are i.i.d.\ drawn from $\pi(x)$. Then $\gamma_j = 1/\ell$ suffices.

Given the joint sample $\{ (X_i,Y_i) \}_{i=1}^n$ and the empirical prior mean $\hm_\pi$, KBR estimates the kernel posterior mean $m^\pi_{X|y}$ as a weighted sum of the feature vectors:
\begin{equation} \label{eq:KBR_estimator}
 \hm^\pi_{X|y} := \sum_{i=1}^n w_i k_\X(\cd,X_i),
\end{equation}
where the weights $w := (w_1, \dots, w_n )^T \in \R^n$ are given by Algorithm \ref{al:KBR_simple}. Here ${\rm diag}(v)$ for $v \in \R^n$ denotes a diagonal matrix with diagonal entries $v$.
It takes as input (i) vectors $ {\bf k}_Y = (k_\Y(y,Y_1),\dots,k_\Y(y,Y_n))^T$, ${\bf m}_\pi = (\hm_\pi(X_1),\dots, \hm_\pi(X_n) )^T \in \R^n$,  where $\hm_\pi(X_i) = \sum_{j=1}^\ell \gamma_j k_\X(X_i,U_j)$;
(ii) kernel matrices $G_X = (k_\X(X_i,X_j)), G_Y = (k_\Y(Y_i,Y_j)) \in \R^{n \times n}$; and 
(iii) regularization constants $\varepsilon,\delta>0$.
The weight vector $w := (w_{1},\dots,w_{n})^T \in \R^n$ is obtained by matrix computations involving two regularized matrix inversions.
Note that these weights can be negative.

\cite{FukSonGre13} showed that KBR is a consistent estimator of the kernel posterior mean under certain smoothness assumptions: the estimate (\ref{eq:KBR_estimator}) converges to $m^\pi_{X|y}$, as the sample size goes to infinity $n \to \infty$ and $\hm_\pi$ converges to $m_\pi$ (with $\varepsilon,\delta \to 0$ in appropriate speed).
For details, see \cite{FukSonGre13,SonFukGre13}.

\begin{algorithm}[t]
\caption{Kernel Bayes' Rule}
\label{al:KBR_simple}
\begin{algorithmic}[1]
\STATE
{\bf Input:} ${\bf k}_Y, {\bf m}_\pi \in \R^n$, $G_X,G_Y \in \R^{n \times n}$, $\varepsilon,\delta > 0$.
\STATE {\bf Output:}  $w := (w_{1},\dots,w_{n})^T \in \R^n$.
\\ \hrulefill
\STATE  $\Lambda \leftarrow  {\rm diag} ( (G_X + n \varepsilon I_n)^{-1} {\bf m}_\pi ) \in \R^{n \times n}$.
\STATE $w \leftarrow \Lambda G_Y ( (\Lambda G_Y)^2 + \delta I_n )^{-1} \Lambda {\bf k}_Y \in \R^n$.

\end{algorithmic}
\end{algorithm}

\subsection{Decoding from empirical kernel means} \label{sec:back_decoding}
In general, as shown above, a kernel mean $m_P$ is estimated as a weighted sum of feature vectors;
\begin{equation}\label{eq:weighted}
\hm_P = \sum_{i=1}^n w_i k(\cd,X_i),
\end{equation}
with samples $X_1,\dots,X_n \in \X$ and (possibly negative) weights $w_1,\dots,w_n \in \R$.
Suppose $\hm_P$ is close to $m_P$, i.e.,\ $\| \hm_P - m_P \|_\H$ is small.
Then $\hm_P$ is supposed to have accurate information about $P$, as $m_P$ preserves all the information of $P$.

How can we decode the information of $P$ from $\hm_P$? 
The empirical kernel mean (\ref{eq:weighted}) has the following property, which is due to the reproducing property of the kernel:
\begin{equation} \label{eq:weighted_inner}
\left< \hm_P, f \right>_\H =  \sum_{i=1}^n w_i f(X_i),\quad \forall f \in \H.
\end{equation}
Namely, the weighted average of any function in the RKHS is equal to the inner product between the empirical kernel mean and that function. This is analogous to the property (\ref{eq:kernel_mean_inner}) of the pupation kernel mean $m_P$. Let $f$ be any function in $\H$. From these properties (\ref{eq:kernel_mean_inner}) (\ref{eq:weighted_inner}), we have 
\[ \left| \E_{X \sim P}[f(X) ] - \sum_{i=1}^n w_i f(X_i) \right| = \left| \left< m_P - \hm_P, f \right>_\H \right| \leq \| m_P - \hm_P \|_\H \| f \|_\H, \]
where we used the Cauchy-Schwartz inequality. Therefore the left hand side will be close to $0$, if the error $\| m_P - \hm_P \|_\H$ is small. This shows that the expectation of $f$ can be estimated by the weighted average $\sum_{i=1}^n w_i f(X_i)$. Note that here $f$ is a function in the RKHS, but the same can also be shown for functions outside the RKHS under certain assumptions \citep{KanFuk14}.
In this way, the estimator of the form (\ref{eq:weighted}) provides estimators of moments, probability masses on sets and the density function (if this exists). 
This will be explained in the context of state-space models in Section \ref{sec:decode}.

\subsection{Kernel Herding} \label{sec:kernel_herding}
Here we explain Kernel Herding \citep{CheWelSmo10}, which is another building block of the proposed filter. 
Suppose the kernel mean $m_P$ is known. We wish to generate samples $x_1,x_2,\dots, x_\ell \in \X$ such that the empirical mean $\check{m}_P :=  \frac{1}{\ell} \sum_{i=1}^\ell k(\cd,x_i)$ is close to $m_P$, i.e.,\ $\| m_P - \check{m}_P \|_\H$ is small. This should be done only using $m_P$.
Kernel Herding achieves this by 
greedy optimization using 
the following update equations:
\begin{eqnarray} 
&& x_1 = \argmax_{x \in \X}\ m_P(x) \label{eq:herding_update1}, \\
&& x_\ell = \argmax_{x \in \X}\ m_P(x) - \frac{1}{\ell} \sum_{i=1}^{\ell-1} k(x,x_i),\quad (\ell \geq 2) \label{eq:herding_update2}
\end{eqnarray}
where $m_P(x)$ denotes the evaluation of $m_P$ at $x$ (recall that $m_P$ is a function in $\H$).

An intuitive interpretation of this procedure can be given if there is a constant $R > 0$ such that $ k(x,x) = R$ for all $x \in \X$ (e.g.,\ $R=1$ if $k$ is Gaussian).
Suppose that $x_1,\dots,x_{\ell-1}$ are already calculated. 
In this case, it can be shown that $x_\ell$ in (\ref{eq:herding_update2}) is the minimizer of
\begin{eqnarray} 
 \mathcal{E}_\ell &:=& \left\| m_P - \frac{1}{\ell} \sum_{i=1}^\ell k(\cd,x_i) \right\|_\H.
 \label{eq:herding_error}
\end{eqnarray}
Thus, Kernel Herding performs greedy minimization of the distance between $m_P$ and the empirical kernel mean $\check{m}_P = \frac{1}{\ell} \sum_{i=1}^\ell k(\cd,x_i)$.

It can be shown that the error $\mathcal{E}_\ell$ of (\ref{eq:herding_error}) decreases at a rate at least $O(\ell^{-1/2})$ under the assumption that $k$ is bounded \citep{BacJulObo12}.
In other words, the herding samples $x_1,\dots,x_\ell$ provide a convergent approximation of $m_P$. 
In this sense, Kernel Herding can be seen as a (pseudo) sampling method.
Note that $m_P$ itself can be an empirical kernel mean of the form (\ref{eq:weighted}).
These properties are important for our resampling algorithm developed in Section \ref{sec:resampling}.

It should be noted that $\mathcal{E}_\ell$ decreases at a faster rate $O(\ell^{-1})$ under a certain assumption \citep{CheWelSmo10}: this is much faster than the rate of $\ell$ i.i.d.\ samples $O(\ell^{-1/2})$.  
Unfortunately, this assumption only holds when $\H$ is finite dimensional \citep{BacJulObo12}, and therefore the fast rate of $O(\ell^{-1})$ has not been guaranteed for infinite dimensional cases.
Nevertheless, this fast rate motivates the use of Kernel Herding in the data reduction method in Appendix \ref{sec:subsampling} (we will use Kernel Herding for two different purposes).

\section{Kernel Monte Carlo Filter}	\label{sec:filter}
In this section, we present our Kernel Monte Carlo Filter (KMCF).
First, we define notation and review the problem setting in Section \ref{sec:notation}.
We then describe the algorithm of KMCF in Section \ref{sec:algorithm}.
We discuss implementation issues such as hyper-parameter selection and computational cost in Section \ref{sec:overview_complexity}.
We explain how to decode the information on the posteriors from the estimated kernel means in Section \ref{sec:decode}.

\subsection{Notation and problem setup} \label{sec:notation}
\begin{table}[t]
\caption{Notation}
\begin{center}
\begin{tabular}{|l|l|}
\hline
$\X$ & State space \\ \hline
$\Y$ & Observation space \\ \hline
$x_t \in \X$ & State at time $t$ \\ \hline
$y_t \in \Y$ & Observation at time $t$ \\ \hline
$p( y_t |x_t)$ & Observation model \\ \hline
$p( x_t | x_{t-1})$ & Transition model \\ \hline
$\{ (X_i,Y_i) \}_{i=1}^n$ & State-observation examples \\ \hline
$k_\X$ & Positive definite kernel on $\X$ \\ \hline
$k_\Y$ & Positive definite kernel on $\Y$ \\ \hline
$\H_\X$ & RKHS associated with $k_\X$ \\ \hline
$\H_\Y$ & RKHS associated with $k_\Y$ \\ \hline
\end{tabular}
\end{center}
\label{tb:notation}
\end{table}%
Here we formally define the setup explained in Section \ref{sec:intro}. The notation is summarized in Table \ref{tb:notation}.

We consider a state-space model (see Figure \ref{fig:problem_setting}).
Let $\X$ and $\Y$ be measurable spaces, which serve as a state space and an observation space, respectively.
Let $x_1,\dots,x_t,\dots, x_T \in \X$ be a sequence of hidden states, which follow a Markov process. Let $p(x_t | x_{t-1})$ denote a transition model that defines this Markov process.
Let $y_1,\dots,y_t,\dots,y_T \in \Y$ be a sequence of observations.
Each observation $y_t$ is assumed to be generated from an observation model $p(y_t | x_t)$ conditioned on the corresponding state $x_t$.
We use the abbreviation $y_{1:t} := y_1,\dots,y_t$.

We consider a filtering problem of estimating the posterior distribution $p(x_t | y_{1:t})$ for each time $t = 1,\dots,T$. The estimation is to be done online, as each $y_t$ is given.
Specifically, we consider the following setting (see also Section \ref{sec:intro}):
\begin{enumerate}
\item  The observation model $p(y_t|x_t)$ is not known explicitly, or even parametrically.
Instead, we are given examples of state-observation pairs $ \{ (X_i,Y_i) \}_{i=1}^n \subset \X \times \Y$ prior to the test phase.  The observation model is also assumed time-homogeneous.

\item Sampling from the transition model $p(x_t|x_{t-1})$ is possible. Its probabilistic model can be an arbitrary nonlinear non-Gaussian distribution, as for standard particle filters.
It can further depend on time. For example, control input can be included in the transition model as $p(x_t | x_{t-1}) := p(x_t | x_{t-1},u_t)$, where $u_t$ denotes control input provided by a user at time $t$.
\end{enumerate}

Let $k_\X: \X \times \X \to \R$ and $k_\Y: \Y \times \Y \to \R$ be positive definite kernels on $\X$ and $\Y$, respectively. Denote by $\H_\X$ and $\H_\Y$ their respective RKHSs.
We address the above filtering problem by estimating the {\bf kernel means of the posteriors}:
\begin{equation} \label{eq:posterior_embedding}
m_{x_t | y_{1:t}} := \int k_\X(\cd,x_t) p(x_t | y_{1:t}) dx_t \in \H_\X \quad (t=1,\dots,T).
\end{equation}
These preserve all the information of the corresponding posteriors, if the kernels are characteristic (see  Section \ref{sec:back_kernel_mean}). Therefore the resulting estimates of these kernel means provide us the information of the posteriors, as explained in Section \ref{sec:decode}

\subsection{Algorithm} \label{sec:algorithm}
KMCF iterates three steps of {\em prediction}, {\em correction} and {\em resampling} for each time $t$.
Suppose that we have just finished the iteration at time $t-1$. Then, as shown later, the resampling step yields the following estimator of (\ref{eq:posterior_embedding}) at time $t-1$:
\begin{equation} \label{eq:previous} 
\check{m}_{x_{t-1}|y_{1:t-1}} := \frac{1}{n} \sum_{i=1}^n k_\X(\cd,\bX_{t-1,i}),
\end{equation}
where $\bX_{t-1,1},\dots,\bX_{t-1,n} \in \X$.
Below we show one iteration of KMCF that estimates the kernel mean (\ref{eq:posterior_embedding}) at time $t$ (see also Figure \ref{fig:prediction_correction}).

\begin{figure}[t]
\begin{center}
	\includegraphics[width=0.95\columnwidth, clip]{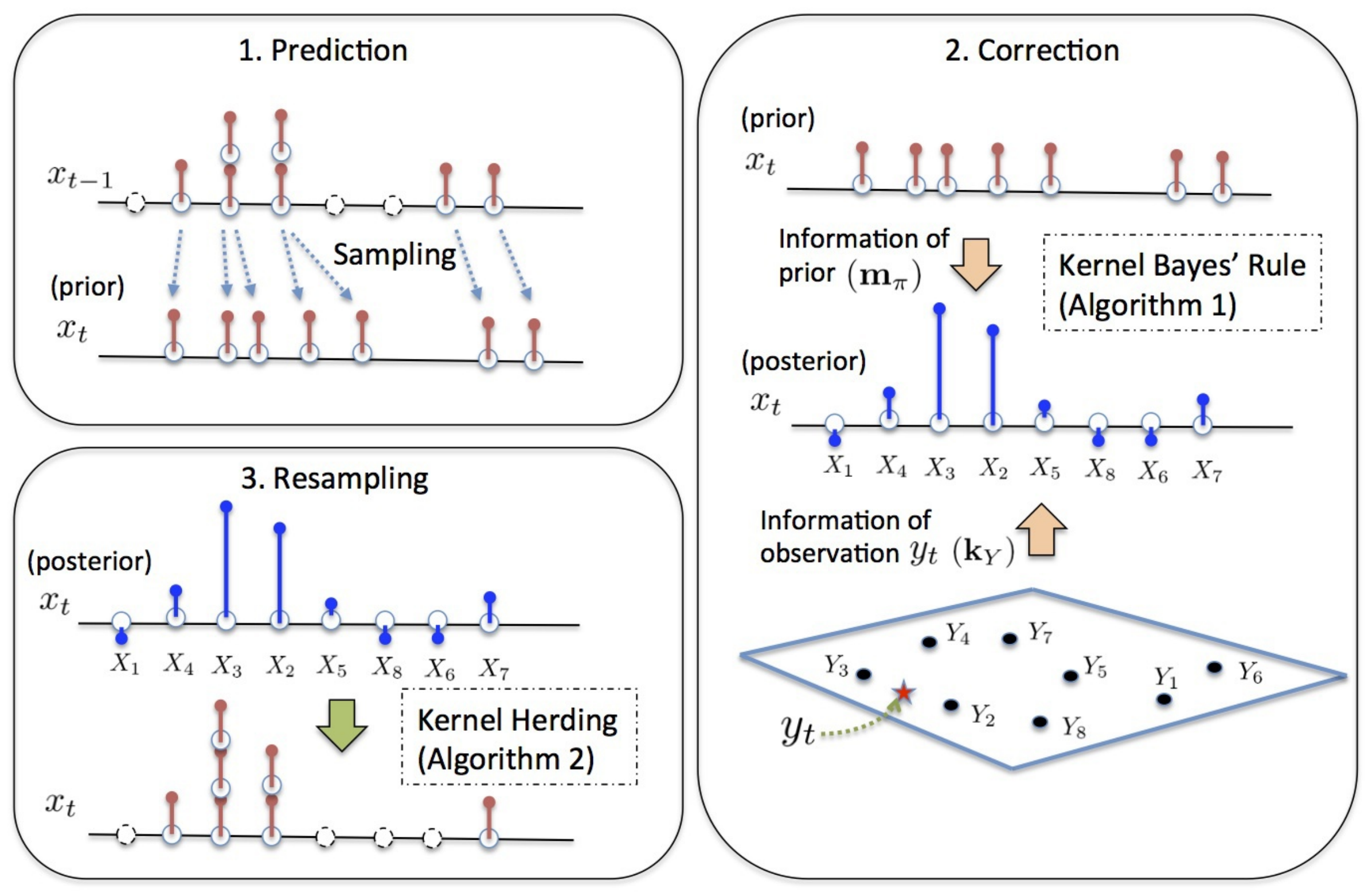}
\caption{One iteration of KMCF. Here $X_1,\dots,X_8$ and $Y_1,\dots,Y_8$ denote states and observations, respectively, in the state-observation examples $\{ (X_i, Y_i) \}_{i=1}^n$ (suppose $n=8$).
{\bf 1.\ Prediction step:} The kernel mean of the prior (\ref{eq:prior_embedding}) is estimated by sampling with the transition model $p(x_t | x_{t-1})$. 
{\bf 2.\ Correction step:}. The kernel mean of the posterior (\ref{eq:posterior_embedding}) is estimated by applying Kernel Bayes' Rule (Algorithm \ref{al:KBR_simple}).
The estimation makes use of the information of the prior (expressed as ${\bf m}_\pi := (\hm_{x_t | y_{1:t-1}}(X_i)) \in \R^8$) as well as that of a new observation $y_t$ (expressed as ${\bf k}_Y := (k_\Y(y_t,Y_i)) \in \R^8$).
The resulting estimate (\ref{eq:KBR}) is expressed as a weighted sample $\{ (w_{t,i}, X_i) \}_{i=1}^n$.
Note that the weights may be negative.
{\bf 3.\ Resampling step:} Samples associated with small weights are eliminated, and those with large weights are replicated by applying Kernel Herding (Algorithm \ref{al:resampling}). The resulting samples provide an empirical kernel mean (\ref{eq:resampled_estimate}), which will be used in the next iteration.}
\label{fig:prediction_correction}
\end{center}
\end{figure}

\paragraph{1.\ Prediction step}
The prediction step is as follows. 
We generate a sample from the transition model for each $\bX_{t-1,i}$ in (\ref{eq:previous});
\begin{equation} \label{eq:prediction_sampling}
 X_{t,i} \sim p(x_t | x_{t-1} = \bX_{t-1,i}), \quad (i = 1,\dots,n).
\end{equation}
We then specify a new empirical kernel mean;
\begin{equation}\label{eq:prior} 
	\hm_{x_t | y_{1:t-1}} := \frac{1}{n} \sum_{i=1}^n k_\X(\cd,X_{t,i}).
\end{equation}
This is an estimator of the following kernel mean of the prior;
\begin{equation} \label{eq:prior_embedding}
 m_{x_t | y_{1:t-1}} := \int k_\X(\cd,x_t) p(x_t | y_{1:t-1}) dx_t \in \H_\X,
\end{equation} 
where 
\[ p(x_t | y_{1:t-1}) = \int p(x_t | x_{t-1}) p(x_{t-1} | y_{1:t-1} ) dx_{t-1} \]
is the prior distribution of the current state $x_t$.
Thus (\ref{eq:prior}) serves as a prior for the subsequent posterior estimation.

In Section \ref{sec:theory}, we theoretically analyze this sampling procedure in detail, and provide justification of (\ref{eq:prior}) as an estimator of the kernel mean (\ref{eq:prior_embedding}).
We emphasize here that such an analysis is necessary, even though the sampling procedure is similar to that of a particle filter:
the theory of particle methods does not provide a theoretical justification of (\ref{eq:prior}) as a kernel mean estimator, since it deals with probabilities as empirical distributions.

\paragraph{2.\ Correction step}
This step estimates the kernel mean (\ref{eq:posterior_embedding}) of the posterior by using Kernel Bayes' Rule (Algorithm \ref{al:KBR_simple}) in Section \ref{sec:KBR}.
This makes use of the new observation $y_t$, the state-observation examples $\{ (X_i,Y_i) \}_{i=1}^n$ and the estimate (\ref{eq:prior}) of the prior.

The input of Algorithm \ref{al:KBR_simple} consists of (i) vectors
\begin{eqnarray*}
{\bf k}_Y &=& (k_\Y(y_t,Y_1),\dots,k_\Y(y_t,Y_n))^T \in \R^n \label{eq:kY} \\
{\bf m}_\pi &=& (\hm_{x_t | y_{1:t-1}} (X_1),\dots, \hm_{x_t | y_{1:t-1}} (X_n) )^T \nonumber \\
&=& \left( \frac{1}{n} \sum_{i=1}^n  k_\X(X_q, X_{t,i}) \right)_{q=1}^n \in \R^n \label{eq:mpi},
\end{eqnarray*} 
which are interpreted as expressions of $y_t$ and $\hm_{x_t | y_{1:t-1}}$ using the sample $\{ (X_i,Y_i) \}_{i=1}^n$,
(ii) kernel matrices $G_X = (k_\X(X_i,X_j))$, $G_Y = (k_\Y(Y_i,Y_j)) \in \R^{n \times n}$, and  (iii) regularization constants $\varepsilon,\delta>0$.
These constants $\varepsilon,\delta$ as well as kernels $k_\X, k_\Y$ are hyper-parameters of KMCF; we will discuss how to choose these parameters later.

Algorithm \ref{al:KBR_simple} outputs a weight vector $w := (w_{1},\dots,w_{n}) \in \R^n$.
Normalizing these weights\footnote{For this normalization procedure, see discussion in Section \ref{sec:overview_complexity}.} $w_t := w / \sum_{i=1}^n w_{i}$, we obtain an estimator of (\ref{eq:posterior_embedding}) as
\begin{equation} \label{eq:KBR}
 \hm_{x_t|y_{1:t}} = \sum_{i=1}^n w_{t,i} k_\X(\cd,X_i).
\end{equation}

The apparent difference from a particle filter is that the posterior (kernel mean) estimator (\ref{eq:KBR}) is expressed in terms of the samples $X_1,\dots,X_n$ in the training sample $\{ (X_i,Y_i) \}_{i=1}^n$, not with the samples from the prior  (\ref{eq:prior}). This requires that the training samples $X_1,\dots,X_n$ cover the support of posterior $p(x_t | y_{1:t})$ sufficiently well. If this does not hold, we cannot expect good performance for the posterior estimate. Note that this is also true for any methods that deal with the setting of this paper; poverty of training samples in a certain region means that we do not have any information about the observation model $p(y_t | x_t)$ in that region.

\paragraph{3. Resampling step} \label{sec:resampling}
This step applies the update equations  (\ref{eq:herding_update1}) (\ref{eq:herding_update2}) of Kernel Herding in Section \ref{sec:kernel_herding} to the estimate (\ref{eq:KBR}).
This is to obtain samples $\bX_{t,1},\dots,\bX_{t,n}$ such that 
\begin{equation} \label{eq:resampled_estimate} 
\check{m}_{x_{t} | y_{1:t}} := \frac{1}{n} \sum_{i=1}^n k_\X(\cd,\bX_{t,i})
\end{equation}
is close to (\ref{eq:KBR}) in the RKHS. Our theoretical analysis in Section \ref{sec:theory} shows that such a procedure can reduce the error of the prediction step at time $t+1$.

\begin{algorithm}[t]
\caption{Resampling with Kernel Herding}
\label{al:resampling}
\begin{algorithmic}[1]
\STATE
{\bf Input:} $\{ (w_{t,i}, X_i) \}_{i=1}^n$.
\STATE {\bf Output:} $\bX_{t,1},\dots,\bX_{t,n} \in \{ X_i \}_{i=1}^n$. 
\STATE {\bf Requirement:} $k_\X: \X \times \X \to \R$.
\\ \hrulefill
	\STATE $ \bX_{t,1} \leftarrow \argmax_{x \in \{ X_1,\dots, X_n \}} \sum_{i=1}^n w_{t,i} k_\X(x,X_i)$.
	\FOR{$p=2$ to $n$}
		\STATE $\bX_{t,p} \leftarrow \argmax_{x \in \{ X_1,\dots, X_n \}} \sum_{i=1}^n w_{t,i} k_\X(x,X_i) - \frac{1}{p} \sum_{j=1}^{p-1} k_\X(x,\bX_{t,j})$
	\ENDFOR \\
\end{algorithmic}
\end{algorithm}

The procedure is summarized in Algorithm \ref{al:resampling}.
Specifically, we generate each $\bX_{t,i}$ by searching the solution of the optimization problem in (\ref{eq:herding_update1}) (\ref{eq:herding_update2}) from a finite set of samples $\{X_1,\dots,X_n \}$ in (\ref{eq:KBR}).
We allow repetitions in $\bX_{t,1},\dots,\bX_{t,n}$.
We can expect that the resulting (\ref{eq:resampled_estimate}) is close to (\ref{eq:KBR}) in the RKHS if the samples $X_1,\dots,X_n$ cover the support of the posterior $p(x_t | y_{1:t})$ sufficiently. 
This is verified by the theoretical analysis of Section \ref{sec:rates_resampling}.

Here searching for the solutions from a finite set reduces the computational costs of Kernel Herding. 
It is possible to search from the entire space $\X$, if we have sufficient time or if the sample size $n$ is small enough; it depends on applications and available computational resources.
We also note that the size of the resampling samples is not necessarily $n$; this depends on how accurately these samples approximate (\ref{eq:KBR}).
Thus a smaller number of samples may be sufficient.
In this case we can reduce the computational costs of resampling, as discussed in Section \ref{sec:theory_resampling}.

The aim of our resampling step is similar to that of the resampling step of a particle filter (see, e.g.,\ \cite{DouJoh11}). 
Intuitively, the aim is to eliminate samples with very small weights, and replicate those with large weights (see Figures \ref{fig:prediction_correction} and \ref{fig:resampling}).
In particle methods, this is realized by generating samples from the empirical distribution defined by a weighted sample (therefore this procedure is called ``resampling").
Our resampling step is a realization of such a procedure in terms of the kernel mean embedding: we generate samples $\bX_{t,1},\dots,\bX_{t,n}$ from the empirical kernel mean (\ref{eq:KBR}).

Note that the resampling algorithm of particle methods is not appropriate for use with kernel mean embeddings. This is because it assumes that weights are positive, but our weights in (\ref{eq:KBR}) can be negative, as (\ref{eq:KBR}) is a kernel mean estimator.
One may apply the resampling algorithm of particle methods by first truncating the samples with negative weights. 
However, there is no guarantee that samples obtained by this heuristic produce a good approximation of (\ref{eq:KBR}) as a kernel mean, as shown by experiments in Section \ref{sec:exp_resampling}.
In this sense, the use of Kernel Herding is more natural since it generates samples that approximate a kernel mean.

\begin{algorithm}[t]
\caption{Kernel Monte Carlo Filter}
\label{al:KBRPF}
\begin{algorithmic}[1]
\STATE {\bf Input:} $y_1,\dots,y_T \in \Y$.
\STATE {\bf Output:} $w_1,\dots,w_T \in \R^n$.
\STATE
{\bf Requirement:} $k_\X$, $k_\Y$, $\varepsilon, \delta$, $\{ (X_i, Y_i) \}_{i=1}^n$, $p(x_t | x_{t-1})$, $p_{\rm init}$.
\\ \hrulefill
	\STATE $G_X \leftarrow (k_\X(X_i,X_j)) \in \R^{n \times n}$.
	\STATE $G_Y \leftarrow (k_\Y(Y_i,Y_j)) \in \R^{n \times n}$.

	\FOR{ $t = 1$ to $T$ }
		\IF{ $t=1$ }
			\STATE Sampling: $X_{1,1},\dots,X_{1,n} \sim  p_{\rm init}$\ i.i.d.		
		\ELSE
			\STATE $\bX_{t-1,1}, \dots, \bX_{t-1,n} \leftarrow $ Algorithm \ref{al:resampling}$(w_{t-1}, \{ X_i \}_{i=1}^n )$.			
			\STATE  Sampling: $X_{t,i} \sim p(x_t | x_{t-1} = \bX_{t-1,i})\ (i=1,\dots,n)$. 
		\ENDIF
		\STATE ${\bf m}_\pi \leftarrow ( \frac{1}{n} \sum_{i=1}^n k_\X(X_q, X_{t,i}) )_{q=1}^n \in \R^n$.
		\STATE ${\bf k}_Y \leftarrow (k_\Y(Y_q,y_t))_{q=1}^n \in \R^n$. 
		\STATE $w_t \leftarrow $ Algorithm \ref{al:KBR_simple}$( {\bf k}_Y, {\bf m}_\pi, G_X, G_Y, \varepsilon, \delta )$.
		\STATE $w_t \leftarrow w_{t} / \sum_{i=1}^n w_{t,i}$.
	\ENDFOR
\end{algorithmic}
\end{algorithm}

\paragraph{Overall algorithm.}
We summarize the overall procedure of KMCF in Algorithm \ref{al:KBRPF}, where $p_{\rm init}$ denotes a prior distribution for the initial state $x_1$.
For each time $t$, KMCF takes as input an observation $y_t$, and outputs a weight vector $w_t = (w_{t,1},\dots,w_{t,n})^T \in \R^n$. Combined with the samples $X_1,\dots,X_n$ in the state-observation examples $\{ (X_i,Y_i) \}_{i=1}^n$, these weights provide an estimator (\ref{eq:KBR}) of the kernel mean of posterior (\ref{eq:posterior_embedding}).

We first compute kernel matrices $G_X, G_Y$ (Line 4-5), which are used in Algorithm \ref{al:KBR_simple} of Kernel Bayes' Rule (Line 15).
For $t=1$, we generate an i.i.d.\ sample $X_{1,1},\dots,X_{1,n}$ from the initial distribution $p_{\rm init}$ (Line 8), which provides an estimator of the prior corresponding to (\ref{eq:prior}).
Line 10 is the resampling step at time $t-1$, and Line 11 is the prediction step at time $t$.
Line 13-16 corresponds to the correction step.

\subsection{Discussion} \label{sec:overview_complexity}
The estimation accuracy of KMCF can depend on several factors in practice. Below we discuss these issues.
\paragraph{Training samples.}
We first note that training samples $ \{ (X_i,Y_i) \}_{i=1}^n$ should provide the information concerning the observation model $p(y_t | x_t)$.
For example, $ \{ (X_i,Y_i) \}_{i=1}^n$ may be an i.i.d.\ sample from a joint distribution $p(x,y)$ on $\X \times \Y$, which decomposes as $p(x,y) = p(y|x) p(x)$. Here $p(y|x)$ is the observation model and $p(x)$ is some distribution on $\X$.  The support of $p(x)$ should cover the region where  states $x_1,\dots,x_T$ may pass in the test phase, as discussed in Section \ref{sec:algorithm}.
For example, this is satisfied when the state space $\X$ is compact, and the support of $p(x)$ is the entire $\X$. 

Note that training samples $\{ (X_i,Y_i) \}_{i=1}^n$ can also be non-i.i.d in practice. For example, we may deterministically select $X_1,\dots,X_n$ so that they cover the region of interest. In location estimation problems in robotics, for instance, we may collect location-sensor examples $\{ (X_i,Y_i) \}_{i=1}^n$ so that locations $X_1,\dots,X_n$ cover the region where location estimation is to be conducted \citep{QuiStaCoaThr10}.

\paragraph{Hyper-parameters.}
As in other kernel methods in general, the performance of KMCF depends on the choice of its hyper-parameters, which are the kernels $k_\X$ and $k_\Y$ (or parameters in the kernels, e.g.,\ the bandwidth of the Gaussian kernel) and the regularization constants $\delta, \varepsilon > 0$.
We need to define these hyper-parameters based on the joint sample $\{ (X_i,Y_i) \}_{i=1}^n$, before running the algorithm on the test data $y_1,\dots,y_T$.
This can be done by cross validation. 
Suppose that $\{ (X_i,Y_i) \}_{i=1}^n$ is given as a sequence from the state-space model.
We can then apply two-fold cross validation, by dividing the sequence into two subsequences. 
If $\{ (X_i,Y_i) \}_{i=1}^n$ is not a sequence, we can rely on the cross validation procedure for Kernel Bayes' Rule (see Section 4.2 of \cite{FukSonGre13}).

\paragraph{Normalization of weights.}
We found in our preliminary experiments that normalization of the weights (Line 16, Algorithm \ref{al:KBRPF}) is beneficial to the filtering performance. This may be justified by the following discussion about a kernel mean estimator in general.
Let us consider a consistent kernel mean estimator $\hm_P := \sum_{i=1}^n w_i k(\cd,X_i)$ such that $\lim_{n \to \infty} \| \hm_P - m_P \|_\H = 0$. Then we can show that the sum of the weights converges to $1$: $\lim_{n \to \infty} \sum_{i=1}^n w_i = 1$ under certain assumptions \citep{KanFuk14}. This could be explained as follows. Recall that the weighted average $\sum_{i=1}^n w_i f(X_i)$ of a function $f$ is an estimator of the expectation $\int f(x)dP(x)$. Let $f$ be a function that takes the value $1$ for any input: $f(x) = 1,\ \forall x \in \X$. Then we have $\sum_{i=1}^n w_i f(X_i) = \sum_{i=1}^n w_i$ and $\int f(x) dP(x) = 1$. Therefore $\sum_{i=1}^n w_i$ is as an estimator of $1$. In other words, if the error $\| \hm_P - m_P \|_\H$ is small, then the sum of the weights $\sum_{i=1}^n w_i$ should be close to $1$.
Conversely, if the sum of the weights is far from $1$, it suggests that the estimate $\hm_P$ is not accurate. Based on this theoretical observation, we suppose that normalization of the weights (this makes the sum equal to $1$) results in a better estimate.

\paragraph{Time complexity.}
For each time $t$, the naive implementation of Algorithm \ref{al:KBRPF} requires a time complexity of $O(n^3)$ for the size $n$ of the joint sample $\{ (X_i, Y_i) \}_{i=1}^n$.
This comes from Algorithm \ref{al:KBR_simple} in Line 15 (Kernel Bayes' Rule) and Algorithm \ref{al:resampling} in Line 10 (resampling).
The complexity $O(n^3)$ of Algorithm \ref{al:KBR_simple} is due to the matrix inversions. 
Note that one of the inversions $(G_X + n \varepsilon I_n)^{-1}$ can be computed before the test phase, as it does not involve the test data.
Algorithm \ref{al:resampling} also has complexity of $O(n^3)$.
In Section \ref{sec:theory_resampling}, we will explain how this cost can be reduced to $O(n^2 \ell)$ by generating only $\ell < n$ samples by resampling.


\paragraph{Speeding up methods.}
In Appendix \ref{sec:speed_up}, we describe two methods for reducing the computational costs of KMCF, both of which only need to be applied prior to the test phase.
(i) Low rank approximation of kernel matrices $G_X, G_Y$, which reduces the complexity to $O(n r^2)$, where $r$ the rank of low rank matrices:
Low rank approximation works well in practice, since eigenvalues of a kernel matrix often decay very rapidly. Indeed this has been theoretically shown for some cases; see \cite{Wid63,Wid64} and discussions in \cite{BacJor02}. 
(ii) A data reduction method based on Kernel Herding, which efficiently selects joint subsamples from the training set $\{ (X_i, Y_i) \}_{i=1}^n$: Algorithm \ref{al:KBRPF} is then applied based only on those subsamples.
The resulting complexity is thus $O(r^3)$, where $r$ is the number of subsamples. 
This method is motivated by the fast convergence rate of Kernel Herding \citep{CheWelSmo10}.

Both methods require the number $r$ to be chosen, which is either the rank for low rank approximation, or the number of subsamples in data reduction. This determines the tradeoff between the accuracy and computational time.
In practice, there are two ways of selecting the number $r$. (a) By regarding $r$ as a hyper parameter of KMCF, we can select it by cross validation. (b) We can choose $r$ by comparing the resulting approximation error; such error is measured in a matrix norm for low rank approximation, and in an RKHS norm for the subsampling method.
For details, see Appendix \ref{sec:speed_up}.

\paragraph{Transfer leaning setting.}
We assumed that the observation model in the test phase is the same as for the training samples. However, this might not hold in some situations.
For example, in the vision-based localization problem, the illumination conditions for the test and training phases might be different (e.g., the test is done at night, while the training samples are collected in the morning). 
Without taking into account such a significant change in the observation model, KMCF would not perform well in practice.

This problem could be addressed by exploiting the framework of {\em transfer learning} \citep{PanYan10}. This framework aims at situations where the probability distribution that generates test data is different from that of training samples. The main assumption is that there exist a small number of examples from the test distribution. Transfer learning then provides a way of combining such test examples and abundant training samples, thereby improving the test performance.
The application of transfer learning in our setting remains a topic for future research.




\subsection{Estimation of posterior statistics} \label{sec:decode}
By Algorithm \ref{al:KBRPF}, we obtain the estimates of the kernel means of posteriors (\ref{eq:posterior_embedding}) as
\begin{equation} \label{eq:posterior_decoding}
 \hm_{x_t | y_{1:t}}  = \sum_{i=1}^n w_{t,i} k_\X(\cd,X_i)  \quad (t=1,\dots,T).
\end{equation}
These contain the information on the posteriors $p(x_t | y_{1:t})$ (see Sections \ref{sec:back_kernel_mean} and \ref{sec:back_decoding}). 
We now show how to estimate statistics of the posteriors using these estimates (\ref{eq:posterior_decoding}).
For ease of presentation, we consider the case $\X = \R^d$.
Theoretical arguments to justify these operations are provided by \cite{KanFuk14}.

\paragraph{Mean and covariance.}
Consider the posterior mean $\int x_t p(x_t | y_{1:t} ) dx_t \in \R^d$ and the posterior (uncentered) covariance $\int x_t x_t^T p(x_t | y_{1:t}) dx_t \in \R^{d \times d}$. 
These quantities can be estimated as
\[ \sum_{i=1}^n w_{t,i} X_i \ \ ({\rm mean}). \quad  \quad \sum_{i=1}^n w_{t,i} X_i X_i^T \ \  ({\rm covariance}).  \]

\paragraph{Probability mass.}
Let $A \subset \X$ be a measurable set with smooth boundary. 
Define the indicator function $I_A(x)$ by $I_A(x) = 1$ for $x \in A$ and $I_A(x) = 0$ otherwise. 
Consider the probability mass $\int I_A(x) p(x_t | y_{1:t}) dx_t$.
This can be estimated as $\sum_{i=1}^n w_{t,i} I_A(X_i)$.

\paragraph{Density.}
Suppose $p(x_t | y_{1:t})$ has a density function.
Let $J(x)$ be a smoothing kernel satisfying $\int J(x) dx = 1$ and $J(x) \geq 0$.
Let $h > 0$ and define $J_h(x) := \frac{1}{h^d} J\left( \frac{x}{h} \right)$.
Then the density of $p(x_t | y_{1:t})$ can be estimated as
\begin{equation} \label{eq:densityestimate}
 \hat{p}(x_t|y_{1:t}) = \sum_{i=1}^n w_{t,i} J_h(x_t - X_i), 
\end{equation}
with an appropriate choice of $h$.

\paragraph{Mode.}
The mode may be obtained by finding a point that maximizes (\ref{eq:densityestimate}).
However, this requires a careful choice of $h$.
Instead, we may use  $X_{i_{\rm max}}$ with $i_{\rm max} := \arg \max_{i} w_{t,i}$ as a mode estimate: this is the point in $\{X_1,\dots,X_n \}$ that is associated with the maximum weight in $w_{t,1}, \dots, w_{t,n}$. 
This point can be interpreted as the point that maximizes (\ref{eq:densityestimate}) in the limit of $h \to 0$.

\paragraph{Other methods.}
Other ways of using (\ref{eq:posterior_decoding}) include the pre-image computation and fitting of Gaussian mixtures. See, e.g.,\ \cite{song2009,FukSonGre13,MccOcaRam13}.

\section{Theoretical analysis} \label{sec:theory}
In this section, we analyze the sampling procedure of the prediction step in Section \ref{sec:algorithm}.
Specifically, we derive an upper-bound on the error of the estimator (\ref{eq:prior}).
We also discuss in detail how the resampling step in Section \ref{sec:resampling} works as a pre-processing step of the prediction step.

To make our analysis clear, we slightly generalize the setting of the prediction step, and discuss the sampling and resampling procedures in this setting.
\subsection{Error bound for the prediction step} \label{sec:upper_bound}
Let $\X$ be a measurable space, and $P$ be a probability distribution on $\X$.
Let $p(\cd|x)$ be a conditional distribution on $\X$ conditioned on $x \in \X$.
Let $Q$ be a marginal distribution on $\X$ defined by $Q(B) = \int p(B|x)dP(x)$ for all measurable $B \subset \X$.
In the filtering setting of Section \ref{sec:filter}, the space $\X$ corresponds to the state space, and the distributions $P$, $p(\cd|x)$, and $Q$ correspond to the posterior $p(x_{t-1} | y_{1:t-1})$ at time $t-1$, the transition model $p(x_t | x_{t-1})$, and the prior $p(x_{t} | y_{1:t-1})$ at time $t$, respectively.

Let $k_\X$ be a positive definite kernel on $\X$, and $\H_\X$ be the RKHS associated with $k_\X$.
Let $m_P = \int k_\X(\cd,x)dP(x)$ and $m_Q = \int k_\X(\cd,x)dQ(x)$ be the kernel means of $P$ and $Q$, respectively.
Suppose that we are given an empirical estimate of $m_P$ as
\begin{equation} \label{eq:embed_X}
\hm_P := \sum_{i=1}^n w_i k_\X(\cd,X_i),
\end{equation}
where $w_1,\dots,w_n \in \R$ and $X_1,\dots,X_n \in \X$. 
Considering this weighted sample form enables us to explain the mechanism of the resampling step.

The prediction step can then be cast as the following procedure: 
for each sample $X_i$, we generate a new sample $X'_i$ with the conditional distribution $X'_i \sim p(\cd|X_i)$.
Then we estimate $m_Q$ by
\begin{equation}
\label{eq:KMC}
\hm_Q := \sum_{i=1}^n w_i k_\X(\cd,X'_i),
\end{equation}
which corresponds to the estimate (\ref{eq:prior}) of the prior kernel mean at time $t$.

The following theorem provides an upper-bound on the error of (\ref{eq:KMC}), and reveals properties of (\ref{eq:embed_X}) that affect the error of the estimator (\ref{eq:KMC}).
The proof is given in Appendix \ref{sec:appendix_proof}. 

\begin{theo} \label{theo:finite_sample_bound}
Let $\hm_P$ be a fixed estimate of $m_P$ given by (\ref{eq:embed_X}).
Define a function $\theta$ on $\X \times \X$ by $\theta(x_1,x_2) = \int \int k_\X(x'_1,x'_2) dp(x'_1|x_1)dp(x'_2|x_2), \forall x_1,x_2 \in \X \times \X$, and assume that $\theta$ is included in the tensor RKHS $\H_\X \otimes \H_\X$.\footnote{The tensor RKHS $\H_\X \otimes \H_\X$ is the RKHS of a product kernel $k_{\X \times \X}$ on $\X \times \X$ defined as $k_{\X \times \X} ((x_a,x_b), (x_c,x_d) ) = k_\X(x_a,x_c) k_\X(x_b,x_d), \forall (x_a,x_b), (x_c,x_d) \in \X \times \X$.
This space $\H_\X \otimes \H_\X$ consists of smooth functions on $\X \times \X$, if the kernel $k_\X$ is smooth (e.g.,\ if $k_\X$ is Gaussian; see Sec.\ 4 of \cite{SteChr2008}).
In this case, we can interpret this assumption as requiring that  $\theta$ be smooth as a function on $\X \times \X$.

The function $\theta$ can be written as the inner product between the kernel means of the conditional distributions: $\theta(x_1,x_2) = \left< m_{p(\cd | x_1)},  m_{p(\cd | x_2)} \right>_{\H_\X}$, where $m_{p(\cd | x)} := \int k_\X(\cd, x') dp(x' | x)$.
Therefore the assumption may be further seen as requiring that the map $x \to m_{p(\cd | x)}$ be smooth.
Note that while similar assumptions are common in the literature on kernel mean embeddings (e.g., Theorem 5 of \cite{FukSonGre13}), we may relax this assumption by using approximate arguments in learning theory (e.g.,\ Theorem 2.2 and 2.3 of \cite{ElbSte13}). 
This analysis remains a topic for future research.
}
The estimator $\hm_Q$ (\ref{eq:KMC}) then satisfies
\begin{eqnarray}
 && \E_{X'_1,\dots,X'_n}[\| \hm_{Q} - m_{Q} \|_{\H_\X}^2 ] \nonumber  \\
 &&\leq \sum_{i=1}^n w_i^2  ( \E_{X'_i}[k_\X(X'_i,X'_i)] - \E_{X'_i, \tX'_i}[k_\X(X'_i,\tX'_i)] ) \label{eq:theorem_first} \\
 && \quad + \| \hm_P - m_P \|^2_{\H_\X} \| \theta \|_{\H_\X \otimes \H_\X} \label{eq:theorem_second}, 
\end{eqnarray}
where $X'_i \sim p(\cd| X_i)$ and $\tX'_i$ is an independent copy of $X'_i$.
\end{theo}

From Theorem \ref{theo:finite_sample_bound}, we can make the following observations.
First, the second term (\ref{eq:theorem_second})  of the upper-bound shows that the error of the estimator (\ref{eq:KMC}) is likely to be large if the given estimate (\ref{eq:embed_X}) has large error $\| \hm_P - m_P \|^2_{\H_\X}$, which is reasonable to expect.

Second, the first term (\ref{eq:theorem_first}) shows that the error of (\ref{eq:KMC}) can be large if the distribution of $X'_i$ (i.e.\ $p(\cd | X_i)$) has large variance.
For example, suppose $X'_i = f(X_i) + \varepsilon_i$, where $f: \X \to \X$ is some mapping and $\varepsilon_i$ is a random variable with mean $0$. Let $k_\X$ be the Gaussian kernel: $k_\X(x,x') = \exp(- \| x - x' \| / 2\alpha)$ for some $\alpha >0$.
Then $\E_{X'_i}[k_\X(X'_i,X'_i)] - \E_{X'_i, \tX'_i}[k_\X(X'_i,\tX'_i)]$ increases from $0$ to $1$, as the variance of $\varepsilon_i$ (i.e.\ the variance of $X'_i$) increases from $0$ to infinity.
Therefore in this case (\ref{eq:theorem_first}) is upper-bounded at worst by $\sum_{i=1}^n w_i^2$.
Note that $\E_{X'_i}[k_\X(X'_i,X'_i)] - \E_{X'_i, \tX'_i}[k_\X(X'_i,\tX'_i)]$ is always non-negative.\footnote{To show this, it is sufficient to prove that $\int \int k_\X(x,\tx) dP(x) dP(\tx) \leq \int k_\X(x,x) dP(x)$ for any probability $P$. This can be shown as follows. $\int \int k_\X(x,\tx) dP(x) dP(\tx) = \int \int \left< k_\X(\cd,x), k_\X(\cd,\tx) \right>_{\H_\X} dP(x) dP(\tx) \leq  \int \int \sqrt{ k_\X(x,x) } \sqrt{ k_\X(\tx,\tx)} dP(x) dP(\tx) \leq  \int k_\X(x,x) dP(x)$. Here we used the reproducing property, the Cauchy-Schwartz inequality and Jensen's inequality}

\paragraph{Effective sample size.}
Now let us assume that the kernel $k_\X$ is bounded, i.e.,\ there is a constant $C > 0$ such that $\sup_{x \in \X} k_\X(x,x) < C$. 
Then the inequality of Theorem \ref{theo:finite_sample_bound} can be further bounded as 
\begin{equation} \label{eq:bound_with_boundedness_assumption}
   \E_{X'_1,\dots,X'_n}[\| \hm_{Q} - m_{Q} \|_{\H_\X}^2 ]  \leq 2C \sum_{i=1}^n w_i^2 + \| \hm_P - m_P \|^2_{\H_\X} \| \theta \|_{\H_\X \otimes \H_\X}. 
\end{equation}
This bound shows that  two quantities are important in the estimate $(\ref{eq:embed_X})$: (i) the sum of squared weights $\sum_{i=1}^n w_i^2$, and (ii) the error $\| \hm_P - m_P \|^2_{\H_\X}$.
In other words, the error of (\ref{eq:KMC}) can be large if the quantity $\sum_{i=1}^n w_i^2$ is large, regardless of the accuracy of $(\ref{eq:embed_X})$ as an estimator of $m_P$.
In fact, the estimator of the form $(\ref{eq:embed_X})$ can have large $\sum_{i=1}^n w_i^2$ even when $\| \hm_P - m_P \|^2_{\H_\X}$ is small, as shown in Section \ref{sec:exp_resampling}.

The inverse of the sum of the squared weights $1/\sum_{i=1}^n w_i^2$ can be interpreted as the effective sample size (ESS) of the empirical kernel mean (\ref{eq:embed_X}).
To explain this, suppose that the weights are normalized, i.e.,\ $\sum_{i=1}^n w_i = 1$. Then ESS takes its maximum $n$ when the weights are uniform, $w_1 = \cdots w_n = 1/n$. On the other hand, it becomes small when only a few samples have large weights (see the left figure in Figure \ref{fig:resampling}). Therefore the bound (\ref{eq:bound_with_boundedness_assumption}) can be interpreted as follows: to make (\ref{eq:KMC}) a good estimator of $m_Q$, we need to have (\ref{eq:embed_X}) such that the ESS is large and the error $\| \hm_P - m_P \|_{\H}$ is small. 
Here we borrowed the notion of ESS from the literature on particle methods, in which ESS has also been played  an important role; see, e.g.,\ Sec.\ 2.5.3 of \cite{Liu01} and Sec.\ 3.5 of \cite{DouJoh11}.

\subsection{Role of resampling} \label{sec:theory_resampling}
\begin{figure}[t]
\begin{center}
	\includegraphics[width=0.95\columnwidth, clip]{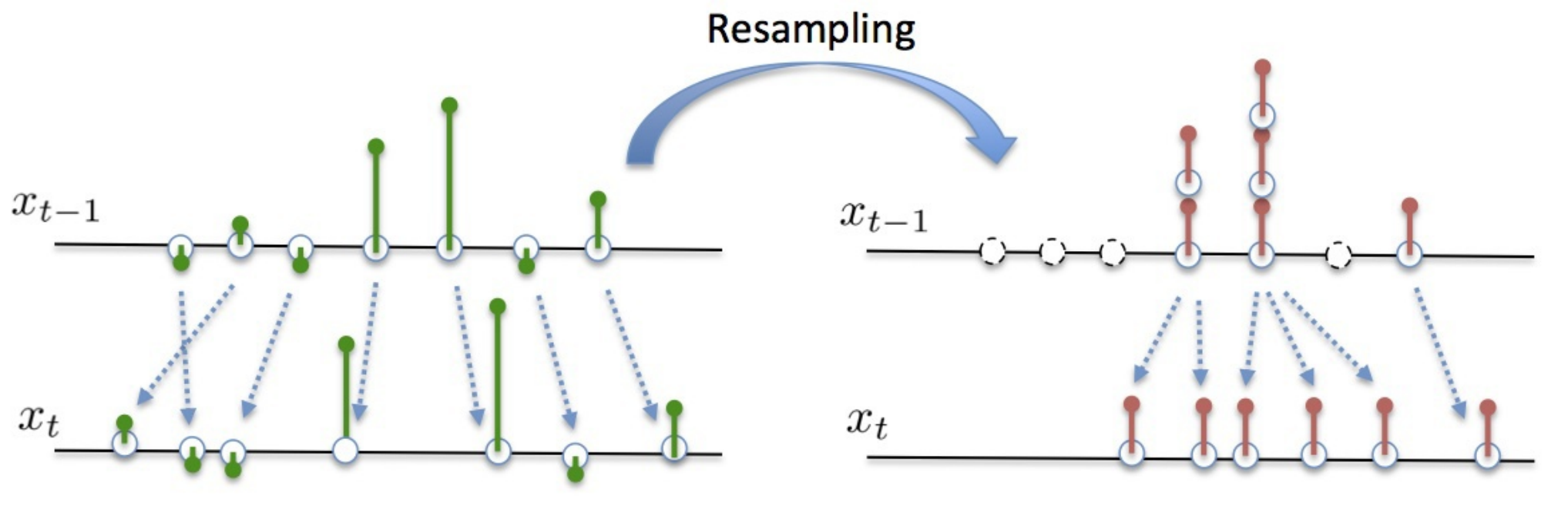}
\caption{An illustration of the sampling procedure with (right) and without (left) the resampling algorithm. The left figure corresponds to the kernel mean estimators (\ref{eq:embed_X}) (\ref{eq:KMC}) in Section \ref{sec:upper_bound}, and the right one corresponds to those (\ref{eq:m_P_resample}) (\ref{eq:m_Q_resample}) in Section \ref{sec:theory_resampling}  }
\label{fig:resampling}
\end{center}
\end{figure}

Based on these arguments, we explain how the resampling step in Section \ref{sec:resampling} works as a preprocessing step for the sampling procedure. 
Consider $\hm_P$ in (\ref{eq:embed_X}) as an estimate (\ref{eq:KBR}) given by the correction step at time $t-1$.
Then we can think of $\hm_Q$ (\ref{eq:KMC}) as an estimator of the kernel mean (\ref{eq:prior_embedding}) of the prior, {\em without} the resampling step.

The resampling step is application of Kernel Herding to $\hm_P$ to obtain samples $\bX_1,\dots,\bX_n$, which provide a new estimate of $m_P$ with uniform weights;
\begin{equation} \label{eq:m_P_resample}
\check{m}_P = \frac{1}{n} \sum_{i=1}^n k_\X(\cd, \bX_i).
\end{equation}
The subsequent prediction step is to generate a sample $\bX'_i \sim p(\cd | \bX_i)$ for each $\bX_i$ $(i=1,\dots,n)$, and estimate $m_Q$ as 
\begin{equation} \label{eq:m_Q_resample}
 \check{m}_Q = \frac{1}{n} \sum_{i=1}^n k_\X(\cd,\bX'_i).
\end{equation}
Theorem \ref{theo:finite_sample_bound} gives the following bound for this estimator that corresponds to (\ref{eq:bound_with_boundedness_assumption}):
\begin{equation} \label{eq:bound_resampling}
  \E_{\bX'_1,\dots,\bX'_n}[\| \check{m}_{Q} - m_{Q} \|_{\H_\X}^2 ] \leq  \frac{2C}{n} + \| \check{m}_P - m_P \|^2_\H \| \theta \|_{\H_\X \otimes \H_\X}. 
\end{equation}

A comparison of the upper-bounds of (\ref{eq:bound_with_boundedness_assumption}) and (\ref{eq:bound_resampling}) implies that the resampling step is beneficial when (i) $\sum_{i=1}^n w_i^2$ is large (i.e.,\ the ESS is small), and (ii) $\| \check{m}_P - \hm_P \|_{\H_\X}$ is small.
The condition on $\| \check{m}_P - \hm_P \|_{\H_\X}$ means that the loss by Kernel Herding (in terms of the RKHS distance) is small.
This implies $\| \hm_P - m_P \|_{\H_\X} \approx \| \check{m}_P - m_P \|_{\H_\X}$, so the second term of (\ref{eq:bound_resampling}) is close to that of (\ref{eq:bound_with_boundedness_assumption}).
On the other hand,  the first term of (\ref{eq:bound_resampling}) will be much smaller than that of (\ref{eq:bound_with_boundedness_assumption}), if $\sum_{i=1}^n w_i^2 \gg 1/n$.
In other words, the resampling step improves the accuracy of the sampling procedure, by increasing the ESS of the kernel mean estimate $\hm_P$. This is illustrated in Figure \ref{fig:resampling}.

The above observations lead to the following procedures:
\paragraph{When to apply resampling.}
If $\sum_{i=1}^n w_i^2$ is not large, the gain by the resampling step will be small.
Therefore the resampling algorithm should be applied when $\sum_{i=1}^n w_i^2$ is above a certain threshold, say $2/n$. The same strategy has been commonly used in particle methods (see, e.g.,\ \cite{DouJoh11}).

Also, the bound (\ref{eq:theorem_first}) of Theorem \ref{theo:finite_sample_bound} shows that resampling is not beneficial if the variance of the conditional distribution $p(\cd | x)$ is very small (i.e.,\ if state transition is nearly deterministic). In this case, the error of the sampling procedure may increase due to the loss $\| \check{m}_P - \hm_P \|_{\H_\X}$ caused by Kernel Herding.

\paragraph{Reduction of computational cost.}
Algorithm \ref{al:resampling} generates $n$ samples $\bX_1,\dots,\bX_n$ with time complexity $O(n^3)$.
Suppose that the first $\ell$ samples $\bX_1,\dots,\bX_\ell$, where $\ell < n$, already approximate $\hm_P$ well:
$\| \frac{1}{\ell} \sum_{i=1}^\ell k_\X(\cd, \bX_i) -\hm_P \|_{\H_\X}$ is small.
We do not then need to generate the rest of samples $\bX_{\ell+1}, \dots,\bX_{n}$: we can make $n$ samples by copying the $\ell$ samples $n/\ell$ times (suppose $n$ can be divided by $\ell$ for simplicity, say $n = 2\ell$).
Let $\bX_1,\dots,\bX_n$ denote these $n$ samples. Then $\frac{1}{\ell} \sum_{i=1}^\ell k_\X(\cd, \bX_i) =  \frac{1}{n} \sum_{i=1}^n k_\X(\cd, \bX_i)$ by definition, so $\| \frac{1}{n} \sum_{i=1}^n k_\X(\cd, \bX_i) -\hm_P \|_{\H_\X}$ is also small.
This reduces the time complexity of Algorithm \ref{al:resampling} to $O(n^2 \ell)$.

One might think that it is unnecessary to copy $n/\ell$ times to make $n$ samples. 
This is not true, however.
Suppose that we just use the first $\ell$ samples to define $\check{m}_P = \frac{1}{\ell} \sum_{i=1}^\ell k_\X(\cd,\bX_i)$.
Then the first term of (\ref{eq:bound_resampling}) becomes $2C/\ell$, which is larger than $2C/n$ of $n$ samples.
This difference involves sampling with the conditional distribution:  $\bX'_i \sim p(\cd | \bX_i)$.
If we just use the $\ell$ samples, sampling is done $\ell$ times. If we use the copied $n$ samples, sampling is done $n$ times.
Thus the benefit of making $n$ samples comes from sampling with the conditional distribution many times.
This matches the bound of Theorem \ref{theo:finite_sample_bound}, where the first term involves the variance of the conditional distribution.

\subsection{Convergence rates for resampling} \label{sec:rates_resampling}

\begin{algorithm}[t]
\caption{Generalized version of Algorithm \ref{al:resampling}}
\label{al:resampling_gen}
\begin{algorithmic}[1]
\STATE
{\bf Input:} $\hm_P \in \H_\X$, $\{Z_1,\dots,Z_N \} \subset \X$, $\ell \in \mathbb{N}$.
\STATE {\bf Output:} $\bX_{1},\dots,\bX_{\ell} \in \{Z_1,\dots,Z_N \}$. 
\\ \hrulefill
	\STATE $ \bX_{1} \leftarrow \argmax_{x \in \{ Z_1,\dots, Z_N \}}  \hm_P(x)$.
	\FOR{$p=2$ to $\ell$}
		\STATE $\bX_{p} \leftarrow \argmax_{x \in \{ Z_1,\dots, Z_N \}}  \hm_P(x) - \frac{1}{p} \sum_{j=1}^{p-1} k_\X(x,\bX_{j})$
	\ENDFOR \\
\end{algorithmic}
\end{algorithm}

Our resampling algorithm (Algorithm \ref{al:resampling}) is an approximate version of Kernel Herding in Section \ref{sec:kernel_herding}: Algorithm \ref{al:resampling} searches for the solutions of the update equations (\ref{eq:herding_update1}) (\ref{eq:herding_update2}) from a finite set $\{X_1,\dots,X_n\} \subset \X$, not from the entire space $\X$. 
Therefore existing theoretical guarantees for Kernel Herding \citep{CheWelSmo10,BacJulObo12} do not apply to Algorithm \ref{al:resampling}.
Here we provide a theoretical justification. 

\paragraph{Generalized version.}
We consider a slightly generalized version shown in Algorithm \ref{al:resampling_gen}:
It takes as input (i) a kernel mean estimator $\hm_P$ of a kernel mean $m_P$, (ii) candidate samples $Z_1,\dots,Z_N$, and (iii) the number $\ell$ of resampling; It then outputs resampling samples $\bX_1,\dots,\bX_\ell \in \{Z_1,\dots,Z_N\}$, which form a new estimator $\check{m}_P := \frac{1}{\ell} \sum_{i=1}^\ell k_\X(\cd,\bX_i)$. Here $N$ is the number of the candidate samples.

Algorithm \ref{al:resampling_gen} searches for the solutions of the update equations (\ref{eq:herding_update1}) (\ref{eq:herding_update2}) from the candidate set $\{Z_1,\dots,Z_N\}$.
Note that here these samples $Z_1,\dots,Z_N$ can be different from those expressing the estimator $\hm_P$. 
If they are the same, i.e., if the estimator is expressed as $\hm_P = \sum_{i=1}^n w_{t,i} k(\cd,X_i)$ with $n= N$ and $X_i = Z_i\ (i=1,\dots,n)$, then Algorithm \ref{al:resampling_gen} reduces to Algorithm \ref{al:resampling}.
In fact, Theorem \ref{theo:resampling} below allows $\hm_P$ to be any element in the RKHS.

\paragraph{Convergence rates in terms of $N$ and $\ell$.} 
Algorithm \ref{al:resampling_gen} gives the new estimator $\check{m}_P$ of the kernel mean $m_P$. The error of this new estimator $\| \check{m}_P - m_P \|_{\H_\X}$ should be close to that of the given estimator, $\| \hm_P - m_P \|_{\H_\X}$.
Theorem \ref{theo:resampling} below guarantees this. 
In particular, it provides convergence rates of $\| \check{m}_P - m_P \|_{\H_\X}$ approaching $\| \hm_P - m_P \|_{\H_\X}$, as $N$ and $\ell$ go to infinity.
This theorem follows from Theorem \ref{theo:resampling_app} in Appendix \ref{sec:proof_resampling}, which holds under weaker assumptions.

\begin{theo} \label{theo:resampling}
Let $m_P$ be the kernel mean of a distribution $P$, and $\hm_P$ be any element in the RKHS $\H_\X$.
Let $Z_1, \dots, Z_N$ be an i.i.d.\ sample from a distribution with density $q$. 
Assume that $P$ has a density function $p$ such that $\sup_{x \in \X} p(x)/q(x) < \infty$.
Let $\bX_1,\dots,\bX_\ell$ be samples given by Algorithm \ref{al:resampling_gen} applied to $\hm_P$ with candidate samples $\{Z_1,\dots,Z_N\}$.
Then for $\check{m}_P := \frac{1}{\ell} \sum_{i=1}^\ell k(\cd,\bX_i)$ we have
\begin{equation} \label{eq:theo_resampling}
 \| \check{m}_P - m_P \|_{\H_\X}^2 = \left( \| \hm_P - m_P \|_{\H_\X} + O_p(N^{-1/2}) \right)^2  + O\left( \frac{\ln \ell}{\ell} \right).\quad (N,\ell \to \infty) \end{equation}
\end{theo}

Our proof in Appendix \ref{sec:proof_resampling} relies on the fact that Kernel Herding can be seen as the Frank-Wolfe optimization method \citep{BacJulObo12}. Indeed, the error $O(\ln \ell / \ell)$ in (\ref{eq:theo_resampling}) comes from the optimization error of the Frank-Wolfe method after $\ell$ iterations \cite[Bound 3.2]{FreGri14}. On the other hand, the error $O_p(N^{-1/2})$ is due to the approximation of the solution space by a finite set $\{Z_1,\dots,Z_N\}$.
These errors will be small if $N$ and $\ell$ are large enough and the error of the given estimator $\| \hm_P - m_P \|_{\H_\X}$ is relatively large. This is formally stated in Corollary \ref{coro:resampling} below.

Theorem \ref{theo:resampling} assumes that the candidate samples are i.i.d.\ with a  density $q$. The assumption $\sup_{x \in \X} p(x)/q(x) < \infty$ requires that the support of $q$ contains that of $p$. This is a formal characterization of the explanation in Section \ref{sec:resampling} that the samples $X_1,\dots,X_N$ should cover the support of $P$ sufficiently.
Note that the statement of Theorem \ref{theo:resampling} also holds for non i.i.d.\ candidate samples, as shown in Theorem \ref{theo:resampling_app} of Appendix \ref{sec:proof_resampling}.

\paragraph{Convergence rates as $\hm_P$ goes to $m_P$.}
Theorem \ref{theo:resampling} provides convergence rates when the estimator $\hm_P$ is fixed. In Corollary \ref{coro:resampling} below, we let $\hm_P$ approach $m_P$, and provide convergence rates for $\check{m}_P$ of Algorithm \ref{al:resampling_gen} approaching $m_P$. This corollary directly follows from Theorem \ref{theo:resampling}, since the constant terms in $O_p(N^{-1/2})$ and $O(\ln \ell / \ell)$ in (\ref{eq:theo_resampling}) do not depend on $\hm_P$, which can be seen from the proof in Section \ref{sec:proof_resampling}.

\begin{coro} \label{coro:resampling}
Assume that $P$ and $Z_1,\dots,Z_N$ satisfy the conditions in Theorem \ref{theo:resampling} for all $N$.
Let $\hm_P^{(n)}$ be an estimator of $m_P$ such that $\| \hm_P^{(n)} - m_P \|_{\H_\X} = O_p(n^{-b})$ as $n \to \infty$ for some constant $b > 0$.\footnote{Here the estimator $\hm_P^{(n)}$ and the candidate samples $Z_1,\dots,Z_N$ can be dependent.}
Let $N = \ell = \lceil n^{2b} \rceil$.
Let $\bX_1^{(n)},\dots,\bX_\ell^{(n)}$ be samples given by Algorithm \ref{al:resampling_gen} applied to $\hm_P^{(n)}$ with candidate samples $\{Z_1,\dots,Z_N\}$.
Then for $\check{m}_P^{(n)} := \frac{1}{\ell} \sum_{i=1}^\ell k_\X(\cd,\bX_i^{(n)})$, we have  
\begin{eqnarray}
&& \| \check{m}_P^{(n)} - m_P \|_{\H_\X} = O_p(n^{-b}) \quad  (n \to \infty).
\end{eqnarray} 
\end{coro}

Corollary \ref{coro:resampling} assumes that the estimator $\hm_P^{(n)}$ converges to $m_P$ at a rate $O_p(n^{-b})$ for some constant $b > 0$. Then the resulting estimator $ \check{m}_P^{(n)}$ by Algorithm \ref{al:resampling_gen} also converges to $m_P$ at the same rate $O(n^{-b})$, if we set $N = \ell = \lceil n^{2b} \rceil$. This implies that if we use sufficiently large $N$ and $\ell$, the errors $O_p(N^{-1/2})$ and $O(\ln \ell / \ell)$ in (\ref{eq:theo_resampling}) can be negligible, as stated earlier.
Note that $N = \ell = \lceil n^{2b} \rceil$ implies that $N$ and $\ell$ can be smaller than $n$, since typically we have $b \leq 1/2$ ($b=1/2$ corresponds to the convergence rates of parametric models). This provides a support for the discussion in Section \ref{sec:theory_resampling} (reduction of computational cost).

\paragraph{Convergence rates of sampling after resampling.} 
We can derive convergence rates of the estimator $\check{m}_Q$ (\ref{eq:m_Q_resample}) in Section \ref{sec:theory_resampling}. 
Here we consider the following construction of $\check{m}_Q$ as discussed in Section \ref{sec:theory_resampling} (reduction of computational cost): 
(i) First apply Algorithm \ref{al:resampling_gen} to $\hm_P^{(n)}$, and obtain resampling samples $\bX_1^{(n)},\dots,\bX_\ell^{(n)} \in \{Z_1,\dots,Z_N\}$; 
(ii) Copy these samples $\lceil n / \ell \rceil$ times, and let $\bX_1^{(n)},\dots,\bX_{ \ell \lceil n / \ell \rceil }^{(n)}$ be the resulting $\ell \times \lceil n / \ell \rceil$ samples;
(iii) Sample with the conditional distribution $\bX_i^{'(n)} \sim p(\cd | \bX_i)\ (i=1,\dots, \ell\lceil n / \ell \rceil )$, and define 
\begin{equation} \label{eq:mq_rate}
 \check{m}_Q^{(n)} := \frac{1}{ \ell \lceil n / \ell \rceil } \sum_{i=1}^{ \ell \lceil n / \ell \rceil } k_\X(\cd,\bX_i^{'(n)}).
\end{equation}

The following corollary is a consequence of Corollary \ref{coro:resampling}, Theorem \ref{theo:finite_sample_bound} and the bound (\ref{eq:bound_resampling}). Note that Theorem \ref{theo:finite_sample_bound} obtains convergence in expectation, which implies convergence in probability.

\begin{coro} \label{coro:resamp_pred}
Let $\theta$ be the function defined in Theorem \ref{theo:finite_sample_bound} and assume $\theta \in \H_\X \otimes \H_\X$.
Assume that $P$ and $Z_1,\dots,Z_N$ satisfy the conditions in Theorem \ref{theo:resampling} for all $N$.
Let $\hm_P^{(n)}$ be an estimator of $m_P$ such that $\| \hm_P^{(n)} - m_P \|_{\H_\X} = O_p(n^{-b})$ as $n \to \infty$ for some constant $b > 0$.
Let $N = \ell = \lceil n^{2b} \rceil$.
Then for the estimator $ \check{m}_Q^{(n)}$ defined as (\ref{eq:mq_rate}), we have
\begin{eqnarray*}
&& \| \check{m}_Q^{(n)} - m_Q \|_{\H_\X} = O_p(n^{- \min(b, 1/2)})  \quad  (n \to \infty).
\end{eqnarray*} 
\end{coro}

Suppose $b \leq 1/2$, which holds with basically any nonparametric estimators.
Then Corollary \ref{coro:resamp_pred} shows that the estimator $\hm_Q^{(n)}$ achieves the same convergence rate as the input estimator $\hm_P^{(n)}$. Note that without resampling, the rate becomes $O_p( \sqrt{\sum_{i=1}^n (w_i^{(n)})^2} + n^{-b} )$, where the weights are given by the input estimator $\hm_P^{(n)} := \sum_{i=1}^n w_i^{(n)} k_\X(\cd,X_i^{(n)})$ (see the bound (\ref{eq:bound_with_boundedness_assumption})).
Thanks to resampling, the sum of the weights in the case of Corollary \ref{coro:resamp_pred} becomes $1/( \ell \lceil n / \ell \rceil) \leq 1/\sqrt{n}$, which is usually smaller than $\sqrt{\sum_{i=1}^n (w_i^{(n)})^2}$ and is faster than or equal to $O_p(n^{-b})$. This shows the merit of resampling in terms of convergence rates; see also the discussions in Section \ref{sec:theory_resampling}.

\subsection{Consistency of the overall procedure} \label{sec:consistency_KMCF}
Here we show the consistency of the overall procedure in KMCF. 
This is based on Corollary \ref{coro:resamp_pred}, which shows the consistency of the resampling step followed by the prediction step, and on Theorem 5 of \cite{FukSonGre13}, which guarantees the consistency of Kernel Bayes' Rule in the correction step.
Thus we consider three steps in the following order: (i) resampling; (ii) prediction; (iii) correction. 
More specifically, we show consistency of the estimator (\ref{eq:KBR}) of the posterior kernel mean at time $t$, given that the one at time $t-1$ is consistent. 

To state our assumptions, we will need the following functions $\theta_{\rm pos}: \Y \times \Y \to \R$, $\theta_{\rm obs}: \X \times \X \to \R$, and $\theta_{\rm tr}: \X \times \X \to \R$:
\begin{eqnarray} 
\theta_{\rm pos} (y, \ty) &:=& \int \int k_\X(x_t, \tx_t) dp(x_t | y_{1:t-1}, y_t = y) dp(\tx_t | y_{1:t-1}, y_t = \ty) \label{eq:pos_func}, \\
\theta_{\rm obs} (x, \tx) &:=& \int \int k_\Y(y_t, \ty_t) dp(y_t | x_t = x) dp( \ty_t | x_t = \tx)  \label{eq:obs_func}, \\
\theta_{\rm tra} (x, \tx) &:=& \int \int k_\X(x_t, \tx_t) dp(x_t | x_{t-1} = x) dp(\tx_t | x_{t-1} = \tx ) \label{eq:tra_func}.
\end{eqnarray}
These functions contain the information concerning the distributions involved.
In (\ref{eq:pos_func}), the distribution $p(x_t | y_{1:t-1}, y_t = y)$ denotes the posterior of the state at time $t$, given that the observation at time $t$ is $y_t = y$. Similarly $p(\tx_t | y_{1:t-1}, y_t = \ty)$ is the posterior at time $t$, given that the observation is $y_t = \ty_t$. 
In (\ref{eq:obs_func}), the distributions $p(y_t|x_t = x)$ and $p(\ty_t | x_t = \tx)$ denote the observation model when the state is $x_t = x$ or $x_t = \tx$, respectively.
In (\ref{eq:tra_func}), the distributions $p(x_t | x_{t-1} = x)$ and $p(\tx_t | x_{t-1} = \tx)$ denote the transition model with the previous state given by $x_{t-1} = x$ or $x_{t-1} = \tx$, respectively.

For simplicity of presentation, we consider here ``$N = \ell = n$" for the resampling step. 
Below denote by $\mathcal{F} \otimes \mathcal{G}$ the tensor product space of two RKHSs $\mathcal{F}$ and $\mathcal{G}$. 

\begin{coro} \label{coro:filtering_consistency}
Let $(X_1,Y_1), \dots, (X_n,Y_n)$ be an i.i.d.\ sample with a joint density $p(x,y) := p(y|x)q(x)$, where $p(y|x)$ is the observation model. 
Assume that the posterior $p(x_t|y_{1:t})$ has a density $p$, and that $\sup_{x \in \X} p(x)/q(x) < \infty$.
Assume that the functions defined by (\ref{eq:pos_func}), (\ref{eq:obs_func}) and (\ref{eq:tra_func}) satisfy $\theta_{\rm pos} \in \H_\Y \otimes \H_\Y$, $\theta_{\rm obs} \in \H_\X \otimes \H_\X$ and $\theta_{\rm tra} \in \H_\X \otimes \H_\X$, respectively.
Suppose that  $\| \hm_{x_{t-1} | y_{1:t-1}} - m_{x_{t-1} | y_{1:t-1}} \|_{\H_\X} \to 0$ as $n \to \infty$ in probability.  
Then for any sufficiently slow decay of regularization constants $\varepsilon_n$ and $\delta_n$ of Algorithm \ref{al:KBR_simple}, we have 
\[ \| \hm_{x_t | y_{1:t}} - m_{x_t | y_{1:t}} \|_{\H_\X} \to 0 \quad (n \to \infty) \]
in probability. 
\end{coro}

Corollary \ref{coro:filtering_consistency} follows from Theorem 5 of \cite{FukSonGre13} and Corollary \ref{coro:resamp_pred}.
The assumptions $\theta_{\rm pos} \in \H_\Y \otimes \H_\Y$ and $\theta_{\rm obs} \in \H_\X \otimes \H_\X$ are due to Theorem 5 of \cite{FukSonGre13} for the correction step, while the assumption $\theta_{\rm tra} \in \H_\X \otimes \H_\X$ is due to Theorem \ref{theo:finite_sample_bound} for the prediction step, from which Corollary \ref{coro:resamp_pred} follows.
As we discussed in footnote 4 of Section \ref{sec:upper_bound}, these essentially assume that the functions $\theta_{\rm pos}$, $\theta_{\rm obs}$ and $\theta_{\rm tra}$ are smooth. Theorem 5 of \cite{FukSonGre13} also requires that the regularization constants $\varepsilon_n, \delta_n$ of Kernel Bayes' Rule should decay sufficiently slowly, as the sample size goes to infinity ($\varepsilon_n, \delta_n \to 0$ as $n \to \infty$). For details, see Sections 5.2 and 6.2 in \cite{FukSonGre13}.

It would be more interesting to investigate the convergence rates of the overall procedure. However, this requires a refined theoretical analysis of Kernel Bayes' Rule, which is beyond the scope of this paper. This is because currently there is no theoretical result on convergence rates of Kernel Bayes' Rule as an estimator of a posterior kernel mean (existing convergence results are for the expectation of function values; see Theorems 6 and 7 in \cite{FukSonGre13}). This remains a topic for future research. 

\section{Experiments}
\label{sec:experiment}
This section is devoted to experiments. 
In Section \ref{sec:exp_resampling}, we conduct basic experiments on the prediction and resampling steps, before going on to the filtering problem.
Here we consider the problem described in Section \ref{sec:theory}.
In Section \ref{sec:exp_synthetic}, the proposed KMCF (Algorithm \ref{al:KBRPF}) is applied to synthetic state-space models. Comparisons are made with existing methods applicable to the setting of the paper (see also Section \ref{sec:related}).
In Section \ref{sec:exp_robot}, we apply KMCF to the real problem of vision-based robot localization.

In the following, $\mathbb{N}(\mu,\sigma^2)$ denotes the Gaussian distribution with mean $\mu \in \R$ and variance $\sigma^2 > 0$.
\subsection{Sampling and resampling procedures} \label{sec:exp_resampling}
The purpose here is to see how the prediction and resampling steps work empirically.
To this end, we consider the problem described in Section \ref{sec:theory} with $\X = \R$ (see Section \ref{sec:upper_bound} for details).
Specifications of the problem are described below.

We will need to evaluate the errors $\| m_P - \hm_P \|_{\H_\X}$ and $\| m_Q - \hm_Q \|_{\H_\X}$, so we need to know the true kernel means $m_P$ and $m_Q$. 
To this end, we define the distributions and the kernel to be Gaussian: this allows us to obtain analytic expressions for $m_P$ and $m_Q$.

\paragraph{Distributions and kernel.}  
More specifically, we define the marginal $P$ and the conditional distribution $p(\cd | x)$ to be Gaussian: $P = \mathbb{N}(0,\sigma_P^2)$ and $p(\cd | x) =  \mathbb{N}(x, \sigma^2_{\rm cond})$. 
Then the resulting $Q = \int p(\cd|x)dP(x)$ also becomes Gaussian: $Q=\mathbb{N}(0,\sigma_P^2 + \sigma_{\rm cond}^2)$.
We define $k_\X$ to be the Gaussian kernel: $k_\X(x,x') = \exp( -(x-x')^2 / 2\gamma^2)$. 
We set $\sigma_P = \sigma_{\rm cond} = \gamma = 0.1$. 

\paragraph{Kernel means.} Due to the convolution theorem of Gaussian functions, the kernel means $m_P = \int k_\X(\cd,x)dP(x)$ and $m_Q = \int k_\X(\cd,x)dQ(x)$ can be analytically computed: $m_P(x) = \sqrt{\frac{\gamma^2}{\sigma^2 + \gamma^2}} \exp(-\frac{x^2}{2(\gamma^2 + \sigma_P^2)})$, $m_Q(x) = \sqrt{\frac{\gamma^2}{(\sigma^2 + \sigma_{\rm cond}^2 + \gamma^2)}} \exp(- \frac{x^2}{2(\sigma_P^2 + \sigma_{\rm cond}^2 + \gamma^2)})$.

\paragraph{Empirical estimates.}
We artificially defined an estimate $\hm_P = \sum_{i=1}^n w_i k_\X(\cd,X_i)$ as follows. 
First, we generated $n=100$ samples $X_1,\dots,X_{100}$ from a uniform distribution on $[-A,A]$ with some $A > 0$ (specified below).
We computed the weights $w_1,\dots,w_n$ by solving an optimization problem 
\[ \min_{w \in \R^n} \| \sum_{i=1}^n w_i k_\X(\cd,X_i) - m_P \|_\H^2 + \lambda \| w \|^2,\]
and then applied normalization so that $\sum_{i=1}^n w_i = 1$.
Here $\lambda > 0$ is a regularization constant, which allows us to control the tradeoff between the error $\| \hm_P - m_P \|_{\H_\X}^2$ and the quantity $ \sum_{i=1}^n w_i^2 = \| w \|^2$. 
If $\lambda$ is very small, the resulting $\hm_P$ becomes very accurate, i.e.,\ $\| \hm_P - m_P \|_{\H_\X}^2$ is small, but has large $\sum_{i=1}^n w_i^2$.
If $\lambda$ is large, the error $\| \hm_P - m_P \|_{\H_\X}^2$ may not be very small, but $\sum_{i=1}^n w_i^2$ becomes small.
This enables us to see how the error $\| \hm_Q - m_Q \|_{\H_\X}^2$ changes as we vary these quantities.

\paragraph{Comparison.}
Given $\hm_P = \sum_{i=1}^n w_i k_\X(\cd,X_i)$, we wish to estimate the kernel mean $m_Q$.
We compare three estimators:
\begin{itemize}
\item woRes: Estimate $m_Q$ without resampling. Generate samples $X'_i \sim p(\cd | X_i)$ to produce the estimate $\hm_Q = \sum_{i=1}^n w_i k_\X(\cd,X'_i)$. This corresponds to the estimator discussed in Section \ref{sec:upper_bound}.
\item Res-KH: First apply the resampling algorithm of Algorithm \ref{al:resampling} to $\hm_P$, yielding $\bX_1,\dots,\bX_n$. Then generate $\bX'_i \sim p(\cd | \bX_i)$ for each $\bX_i$, giving the estimate $\hm_Q = \frac{1}{n} \sum_{i=1}^n k(\cd,\bX'_i)$.
This is the estimator discussed in Section \ref{sec:theory_resampling}.
\item Res-Trunc: Instead of Algorithm \ref{al:resampling}, first truncate negative weights in $w_1,\dots,w_n$ to be $0$, and apply normalization to make the sum of the weights to be $1$. Then apply the multinomial resampling algorithm of particle methods, and estimate $\hm_Q$ as Res-KH.
\end{itemize}

\paragraph{Demonstration.}  \label{sec:demonstration_exp}
Before starting quantitative comparisons, we demonstrate how the above estimators work. 
Figure \ref{fig:demo1} shows demonstration results with $A=1$.
First, note that for $\hm_P = \sum_{i=1}^n w_i k(\cd,X_i)$, samples associated with large weights are located around the mean of $P$, as the standard deviation of $P$ is relatively small $\sigma_P = 0.1$. 
Note also that some of the weights are negative.
In this example, the error of $\hm_P$ is very small $\| m_P - \hm_P \|_{\H_\X}^2 = 8.49e-10$, while that of the estimate $\hm_Q$ given by woRes is  $\| \hm_Q - m_Q \|_{\H_\X}^2 = 0.125$.
This shows that even if $\| m_P - \hm_P \|_{\H_\X}^2$ is very small, the resulting $\| \hm_Q - m_Q \|_{\H_\X}^2$ may not be small, as implied by Theorem \ref{theo:finite_sample_bound} and the bound (\ref{eq:bound_with_boundedness_assumption}).

We can observe the following. First, Algorithm \ref{al:resampling} successfully discarded samples associated with very small weights. 
Almost all the generated samples $\bX_1,\dots,\bX_n$ are located in $[ -2\sigma_P, 2\sigma_P]$, where $\sigma_P$ is the standard deviation of $P$.
The error is $\| \check{m}_P - m_P \|_{\H_\X}^2 = 4.74e-5$, which is greater than $\| m_P - \hm_P \|_{\H_\X}^2$. This is due to the additional error caused by the resampling algorithm.
Note that the resulting estimate $\check{m}_Q$ is of the error $\| \check{m}_Q - m_Q \|_{\H_\X}^2 = 0.00827$. This is much smaller than the estimate $\hm_Q$ by woRes, showing the merit of the resampling algorithm.

Res-Trunc first truncated the negative weights in $w_1,\dots,w_n$. 
Let us see the region where the density of $P$ is very small, i.e.\ the region outside $[-2 \sigma_P, 2 \sigma_P]$.
We can observe that the absolute values of weights are very small in this region.
Note that there exist positive and negative weights. These weights maintain balance such that the amounts of positive and negative values are almost the same.
Therefore the truncation of the negative weights breaks this balance. As a result, the amount of the positive weights surpasses the amount needed to represent the density of $P$.
This can be seen from the histogram for Res-Trunc: some of the samples $\bX_1,\dots,\bX_n$ generated by Res-Trunc are located in the region where the density of $P$ is very small.
Thus the resulting error $\| \check{m}_P - m_P \|_{\H_\X}^2 = 0.0538$ is much larger than that of Res-KH.
This demonstrates why the resampling algorithm of particle methods is not appropriate for kernel mean embeddings, as discussed in Section \ref{sec:resampling}.

\begin{figure}[t]
\begin{center}
\raisebox{10mm}{\includegraphics[width=0.9\columnwidth]{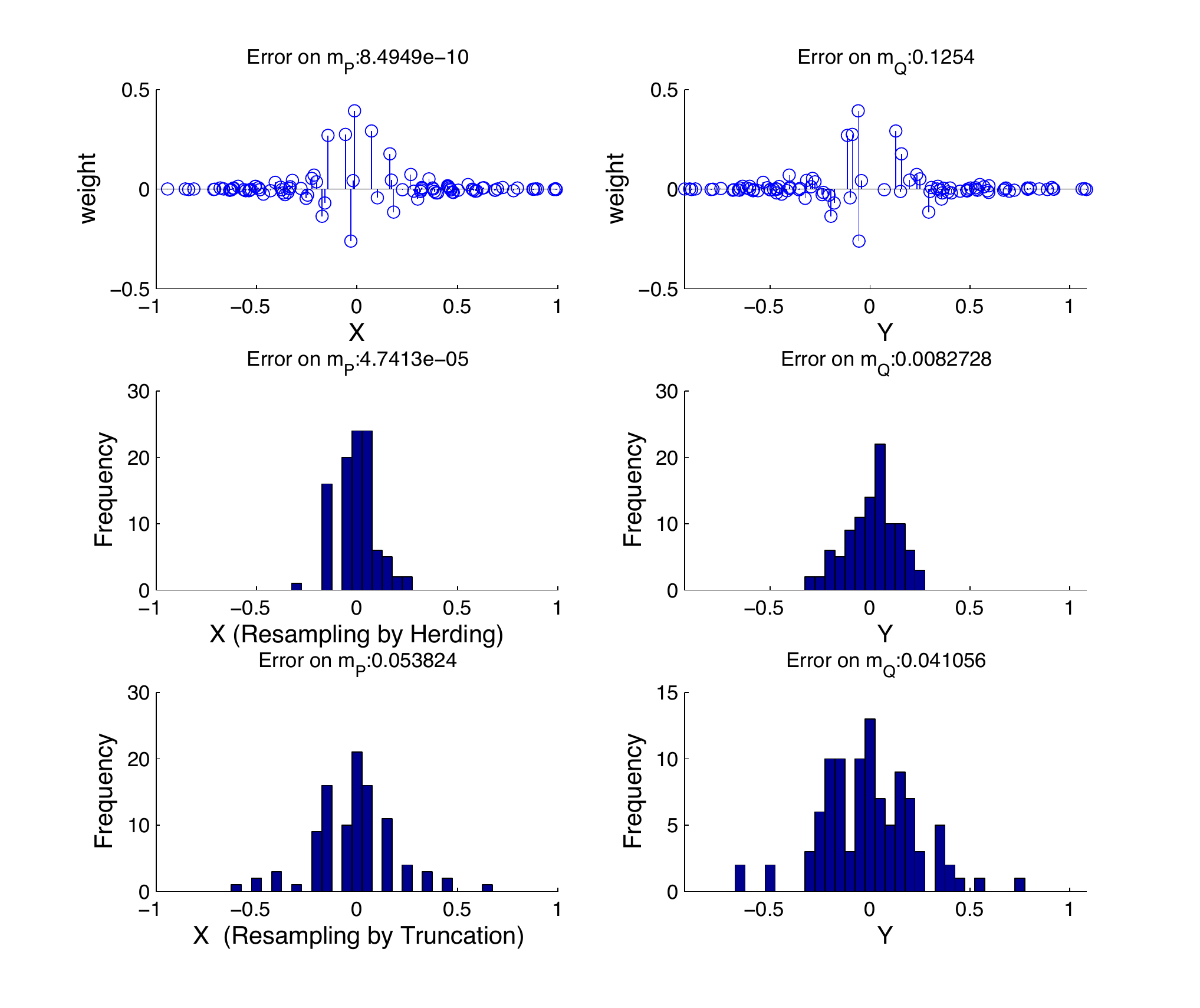}}
\caption{Results of the experiments from Section \ref{sec:demonstration_exp}. 
Top left and right: sample-weight pairs of $\hm_P = \sum_{i=1}^n w_i k_\X(\cd,X_i)$ and $\hm_Q = \sum_{i=1}^n w_i k(\cd,X'_i)$.
Middle left and right: histogram of samples $\bX_1,\dots,\bX_n$ generated by Algorithm \ref{al:resampling}, and that of samples $\bX'_1,\dots,\bX'_n$ from the conditional distribution.
Bottom left and right: histogram of samples generated with multinomial resampling after truncating negative weights, and that of samples from the conditional distribution.}
\label{fig:demo1}
\end{center}
\end{figure}

\paragraph{Effects of the sum of squared weights.}
The purpose here is to see how the error $\| \hm_Q - m_Q \|_{\H_\X}^2$ changes as we vary the quantity $\sum_{i=1}^n w_i^2$  (recall that the bound (\ref{eq:bound_with_boundedness_assumption}) indicates that $\| \hm_Q - m_Q \|_{\H_\X}^2$ increases as $\sum_{i=1}^n w_i^2$ increases).
To this end, we made $\hm_P = \sum_{i=1}^n w_i k_\X(\cd,X_i)$ for several values of the regularization constant $\lambda$ as described above.
For each $\lambda$, we constructed $\hm_P$, and estimated $m_Q$ using each of the three estimators above.
We repeated this $20$ times for each $\lambda$, and averaged the values of $\| \hm_P - m_P \|_{\H_\X}^2$, $\sum_{i=1}^n w_i^2$ and the errors $\| \hm_Q - m_Q \|_{\H_\X}^2$ by the three estimators.
Figure \ref{fig:varyingSSW} shows these results, where the both axes are in the log scale.
Here we used $A = 5$ for the support of the uniform distribution.\footnote{This enables us to maintain the values for $\| \hm_P - m_P \|_{\H_\X}^2$ in almost the same amount, while changing the values for $\sum_{i=1}^n w_i^2$.}
The results are summarized as follows:
\begin{itemize}
\item The error of woRes (blue) increases proportionally to the amount of $\sum_{i=1}^n w_i^2$. This matches the bound (\ref{eq:bound_with_boundedness_assumption}).
\item The error of Res-KH are not affected by $\sum_{i=1}^n w_i^2$. Rather, it changes in parallel with the error of $\hm_P$. This is explained by the discussions in Section \ref{sec:theory_resampling} on how our resampling algorithm improves the accuracy of the sampling procedure.
\item Res-Trunc is worse than Res-KH, especially for large $\sum_{i=1}^n w_i^2$.
This is also explained with the bound (\ref{eq:bound_resampling}).  Here $\check{m}_P$ is the one given by Res-Trunc, so the error $\|\check{m}_P - m_P \|_{\H_\X}$ can be large due to the truncation of negative weights, as shown in the demonstration results.
This makes the resulting error $\| \check{m}_Q - m_Q \|_{\H_\X}$ large.

\end{itemize}
Note that $m_P$ and $m_Q$ are different kernel means, so it can happen that the errors $\| m_Q - \check{m}_Q \|_{\H_\X}$ by Res-KH are less than $\| m_p - \hm_P \|_{\H_\X}$, as in Figure \ref{fig:varyingSSW}.

\begin{figure}[t]
\begin{center}
\raisebox{10mm}{\includegraphics[width=0.7\columnwidth]{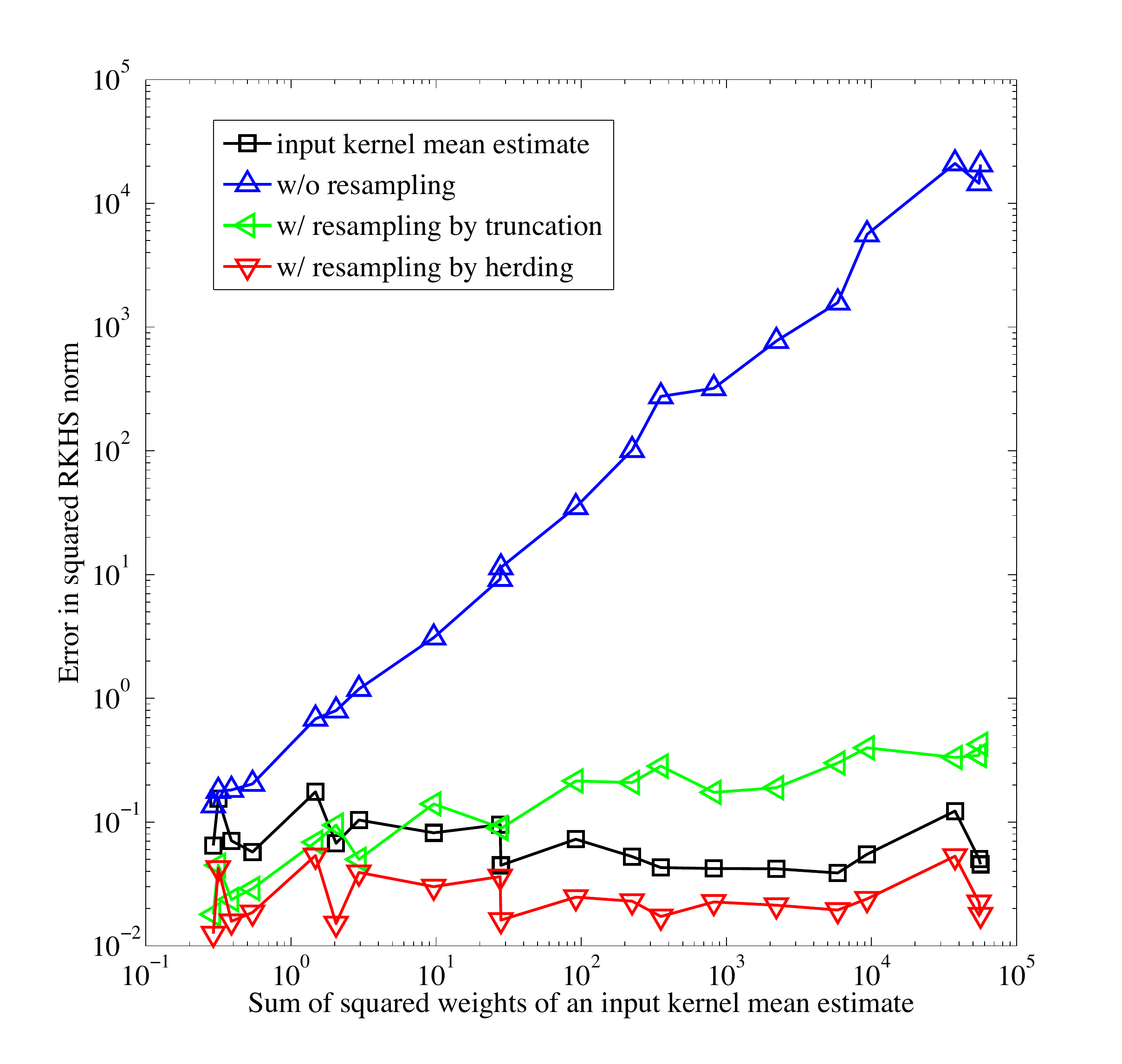}}
\caption{Results of synthetic experiments for the sampling and resampling procedure in Section \ref{sec:exp_resampling}. Vertical axis: errors in the squared RKHS norm. Horizontal axis: values of $\sum_{i=1}^n w_i^2$ for different $\hm_P$. 
Black: the error of $\hm_P$ ($\| \hm_P -m _P \|_{\H_\X}^2)$.
Blue, Green and Red: the errors on $m_Q$ by woRes, Res-KH and Res-Trunc, respectively.}
\label{fig:varyingSSW}
\end{center}
\end{figure}

\subsection{Filtering with synthetic state-space models} \label{sec:exp_synthetic}
Here we apply KMCF to synthetic state-space models. Comparisons were made with the following methods:

\paragraph{kNN-PF \citep{Vlassis2001}} This method uses $k$-NN-based conditional density estimation \citep{Sto77} for learning the observation model. First, it estimates the conditional density of the inverse direction $p(x|y)$ from the training sample $\{ (X_i,Y_i) \}$. 
The learned conditional density is then used as an alternative for the likelihood $p(y_t|x_t)$; this is a heuristic to deal with high-dimensional $y_t$. Then it applies Particle Filter (PF), based on the approximated observation model and the given transition model $p(x_t | x_{t-1})$.

\paragraph{GP-PF \citep{Ferris2006}} This method learns $p(y_t|x_t)$ from $\{ (X_i,Y_i) \}$ with Gaussian Process (GP) regression. Then Particle Filter is applied based on the learned observation model and the transition model.
We used the open-source code\footnote{\small{http://www.gaussianprocess.org/gpml/code/matlab/doc/}} for GP-regression in this experiment, so comparison in computational time is omitted for this method.

\paragraph{KBR filter \citep{FukSonGre11,FukSonGre13}} This method is also based on kernel mean embeddings, as is KMCF. It applies Kernel Bayes' Rule (KBR) in posterior estimation using the joint sample  $\{ (X_i,Y_i) \}$. This method assumes that there also exist training samples for the transition model. Thus in the following experiments, we additionally drew training samples for the transition model.
It was shown \citep{FukSonGre11,FukSonGre13} that this method outperforms Extended and Unscented Kalman Filters, when a state-space model has strong nonlinearity (in that experiment, these Kalman filters were given the full-knowledge of a state-space model). 
We use this method as a baseline.

We used state-space models defined in Table \ref{tb:SSM}, where SSM stands for State Space Model.
In Table \ref{tb:SSM}, $u_t$ denotes a control input at time $t$; $v_t$ and $w_t$ denote independent Gaussian noise: $v_t, w_t \sim \mathbb{N}(0,1)$; $W_t$ denotes 10 dimensional Gaussian noise: $ W_t \sim \mathbb{N}(0,I_{10})$.
We generated each control $u_t$ randomly from the Gaussian distribution $\mathbb{N}(0,1)$.

The state and observation spaces for SSMs \{1a, 1b, 2a, 2b, 4a, 4b\} are defined as $\X = \Y = \R$;  for SSMs \{3a, 3b\}, $\X = \R, \Y = \R^{10}$.
The models in SSMs \{1a, 2a, 3a, 4a\} and SSMs \{1b, 2b, 3b, 4b\} with the same number (e.g.,\ 1a and 1b) are almost the same; the difference is whether $u_t$ exists in the transition model. 
Prior distributions for the initial state $x_1$ for SSMs \{1a, 1b, 2a, 2b, 3a, 3b\} are defined as $p_{\rm init} = \mathbb{N}(0, 1/(1-0.9^2))$, and those for \{4a, 4b\} are defined as a uniform distribution on $[-3, 3]$.

\begin{table}[htdp]
\caption{State-space models (SSM) for synthetic experiments}
\label{tb:SSM}
\begin{center}
\begin{tabular}{| c |c | c|}
\hline 
SSM & transition model & observation model  \\ \hline \hline
1a & $x_t = 0.9 x_{t-1} + v_t$ & $y_t = x_t + w_t$ \\ \hline
1b & $x_t = 0.9 x_{t-1} + \frac{1}{ \sqrt{2}} ( u_t  + v_t )$ & $y_t = x_t + w_t$ \\ \hline
2a & $x_t = 0.9 x_{t-1} + v_t$ & $y_t = 0.5 \exp(x_t / 2) w_t$ \\ \hline
2b &$x_t = 0.9 x_{t-1} + \frac{1}{ \sqrt{2}} ( u_t  + v_t )$ & $y_t = 0.5 \exp(x_t / 2) w_t$ \\ \hline
3a & $x_t = 0.9 x_{t-1} + v_t$ & $y_t = 0.5 \exp( x_t / 2) W_t$ \\ \hline
3b &$x_t = 0.9 x_{t-1} + \frac{1}{ \sqrt{2}} ( u_t  + v_t )$ &$y_t = 0.5 \exp( x_t / 2) W_t$ \\ \hline
4a & $ a_t =  x_{t-1} +  \sqrt{2} v_t$ & $b_t = x_t + w_t$ \\ 
  & $ x_t = \begin{cases}  a_t\ \  ({\rm if}\ |a_t| \leq 3) \\ -3 \  ( {\rm otherwise})  \end{cases}$ & $ y_t = \begin{cases}  b_t\ \ ( {\rm if}\ |b_t| \leq 3) \\ b_t - 6 b_t/|b_t|\ \  ({\rm otherwise}) \end{cases} $  \\  \hline
4b & $ a_t =  x_{t-1} + u_t +  v_t$ & $b_t = x_t + w_t$ \\ 
  & $ x_t = \begin{cases}  a_t\ \  ({\rm if}\ |a_t| \leq 3) \\ -3 \  ( {\rm otherwise})  \end{cases}$ & $ y_t = \begin{cases}  b_t\ \ ( {\rm if}\ |b_t| \leq 3) \\ b_t - 6 b_t/|b_t|\ \  ({\rm otherwise}) \end{cases} $  \\  \hline \hline
\end{tabular}
\end{center}
\end{table}%

SSM 1a and 1b are linear Gaussian models.
SSM 2a and 2b are the so-called stochastic volatility models.
Their transition models are the same as those of SSM 1a and 1b.
On the other hand, the observation model has strong nonlinearity and the noise $w_t$ is multiplicative.
SSM 3a and 3b are almost the same as SSM 2a and 2b. 
The difference is that the observation $y_t$ is 10 dimensional, as $W_t$ is 10 dimensional Gaussian noise.
SSM 4a and 4b are more complex than the other models.
Both the transition and observation models have strong nonlinearities: states and observations located around the edges of the interval $[-3,3]$ may abruptly jump to distant places.

For each model, we generated the training samples $\{ (X_i, Y_i) \}_{i=1}^n$ by simulating the model.
Test data $\{ (x_t,y_t) \}_{t=1}^T$ was also generated by independent simulation (recall that $x_t$ is hidden for each method).
The length of the test sequence was set as $T=100$.
We fixed the number of particles in kNN-PF and GP-PF to $5000$; in primary experiments, we did not observe any improvements even when more particles were used.
For the same reason, we fixed the size of transition examples for KBR filter to $1000$.
Each method estimated the ground truth states $x_1,\dots,x_T$ by estimating the posterior means $\int x_t p(x_t | y_{1:t}) dx_t$ $(t=1,\dots,T)$.
The performance was evaluated with RMSE (Root Mean Squared Errors) of the point estimates, defined as $RMSE = \sqrt{ \frac{1}{T} \sum_{t=1}^T (\hat{x}_t - x_t)^2 }$, where $\hat{x}_t$ is the point estimate.

For KMCF and KBR filter, we used Gaussian kernels for each of $\X$ and $\Y$ (and also for controls in KBR filter).
We determined the hyper-parameters of each method by two-fold cross validation, by dividing the training data into two sequences. 
The hyper-parameters in the GP-regressor for PF-GP were optimized by maximizing the marginal likelihood of the training data. 
To reduce the costs of the resampling step of KMCF, we used the method discussed in Section \ref{sec:theory_resampling} with $\ell = 50$.
We also used the low rank approximation method (Algorithm \ref{al:KBR_lowrank}) and the subsampling method (Algorithm \ref{al:subsampling}) in Appendix \ref{sec:speed_up} to reduce the computational costs of KMCF.
Specifically, we used $r = 10, 20$ (rank of low rank matrices) for Algorithm \ref{al:KBR_lowrank} (described as KMCF-low10 and KMCF-low20 in the results below); $r = 50,100$ (number of subsamples) for Algorithm \ref{al:subsampling} (described as KMCF-sub50 and KMCF-sub100).
We repeated experiments $20$ times for each of different training sample size $n$.

Figure \ref{fig:artificial_woCont} shows the results in RMSE for SSMs \{1a, 2a, 3a, 4a\}, and Figure \ref{fig:artificial_Cont} shows those for SSMs \{1b, 2b, 3b, 4b\}.
Figure \ref{fig:synthetic_time} describes the results in computational time for SSM 1a and 1b; the results for the other models are similar, so we omit them.
We do not show the results of KMCF-low10 in Figure \ref{fig:artificial_woCont} and \ref{fig:artificial_Cont}, since they were numerically unstable and gave very large RMSEs.

\begin{figure}[t]
\vskip -0.05in
\begin{center}
	\subfigure[RMSE (SSM 1a)]{
			\includegraphics[width=0.43\columnwidth]{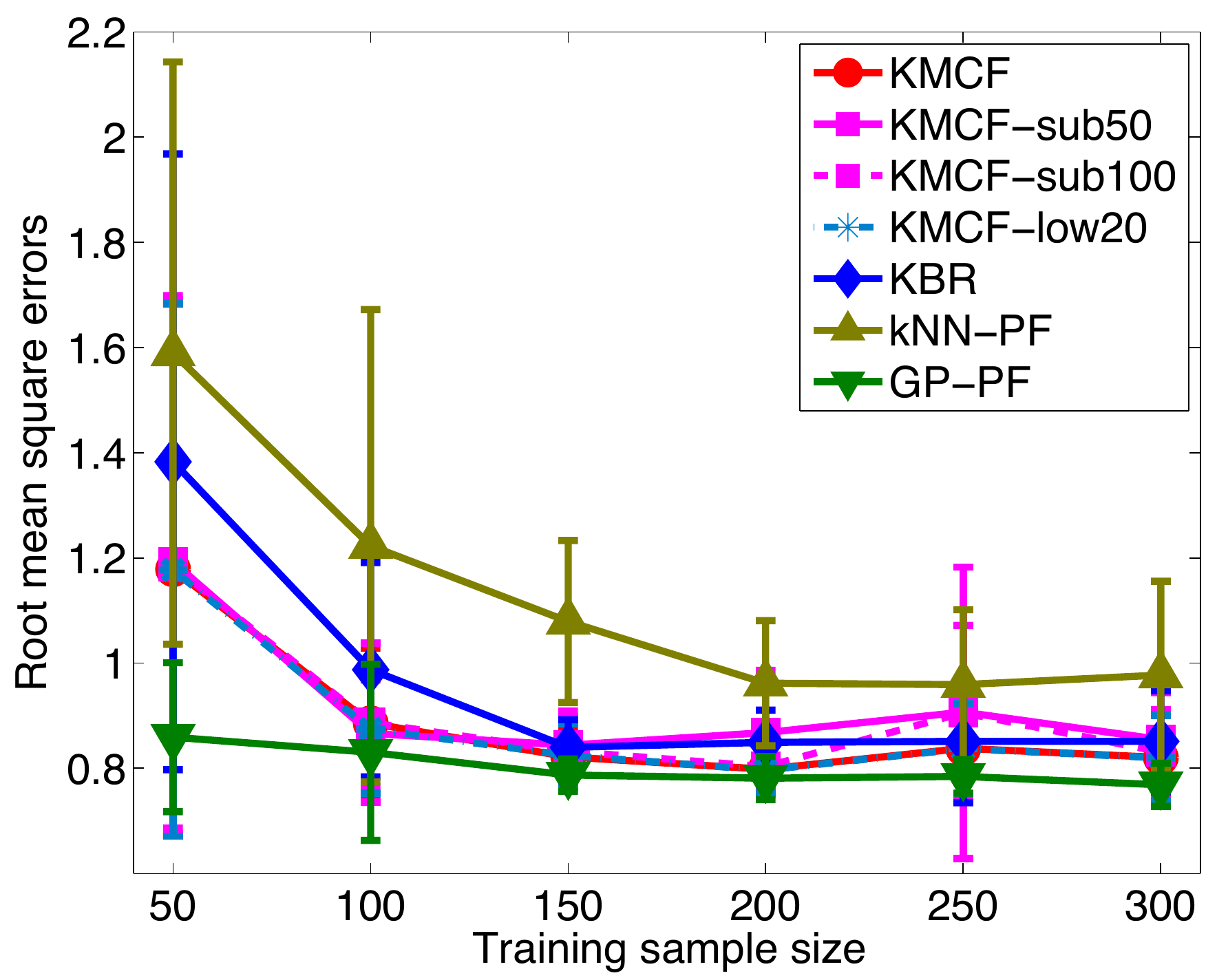}
	}%
	\subfigure[RMSE (SSM 2a)]{
			\includegraphics[width=0.43\columnwidth]{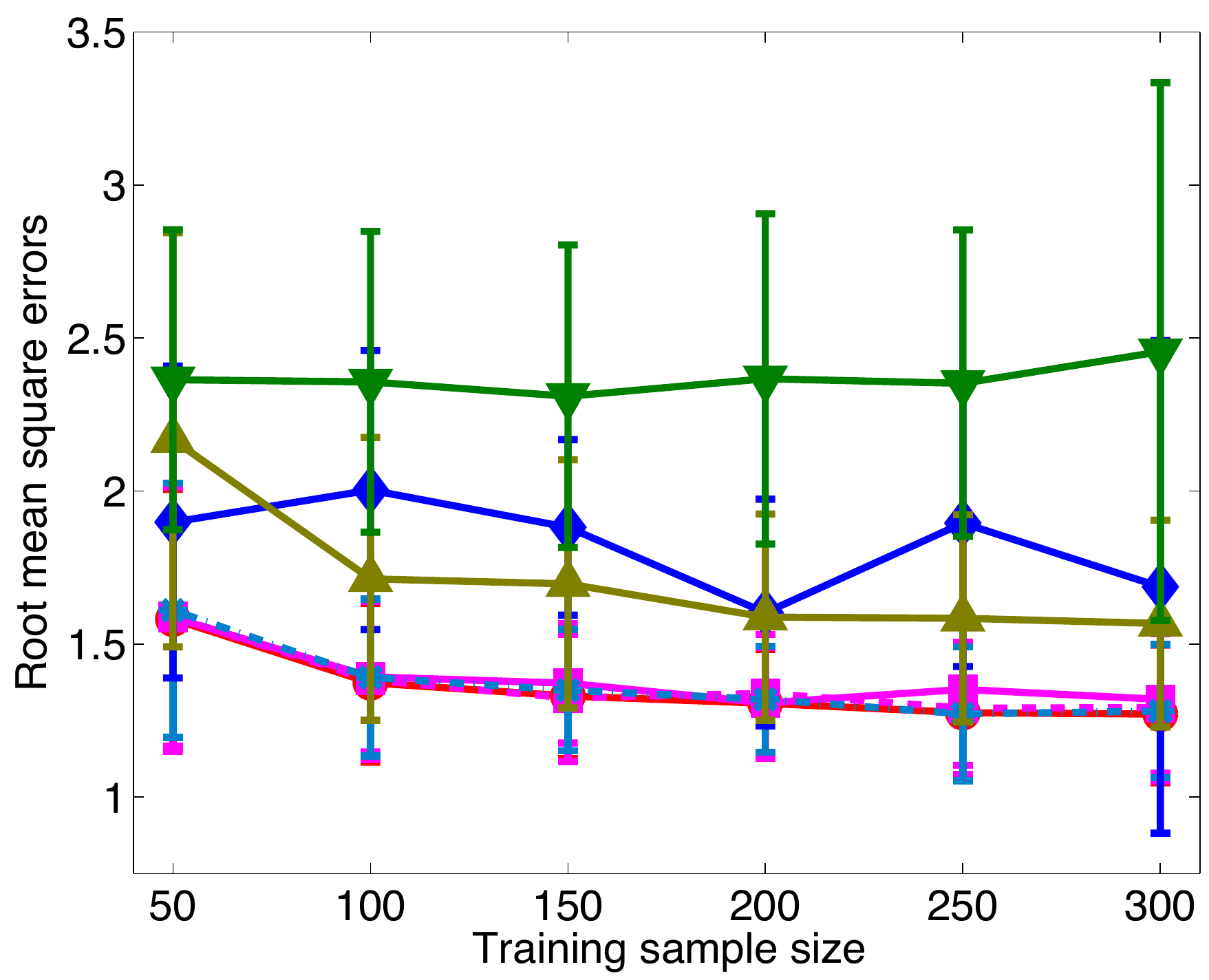}
	}%
	\\ \subfigure[RMSE (SSM 3a)]{
			\includegraphics[width=0.43\columnwidth]{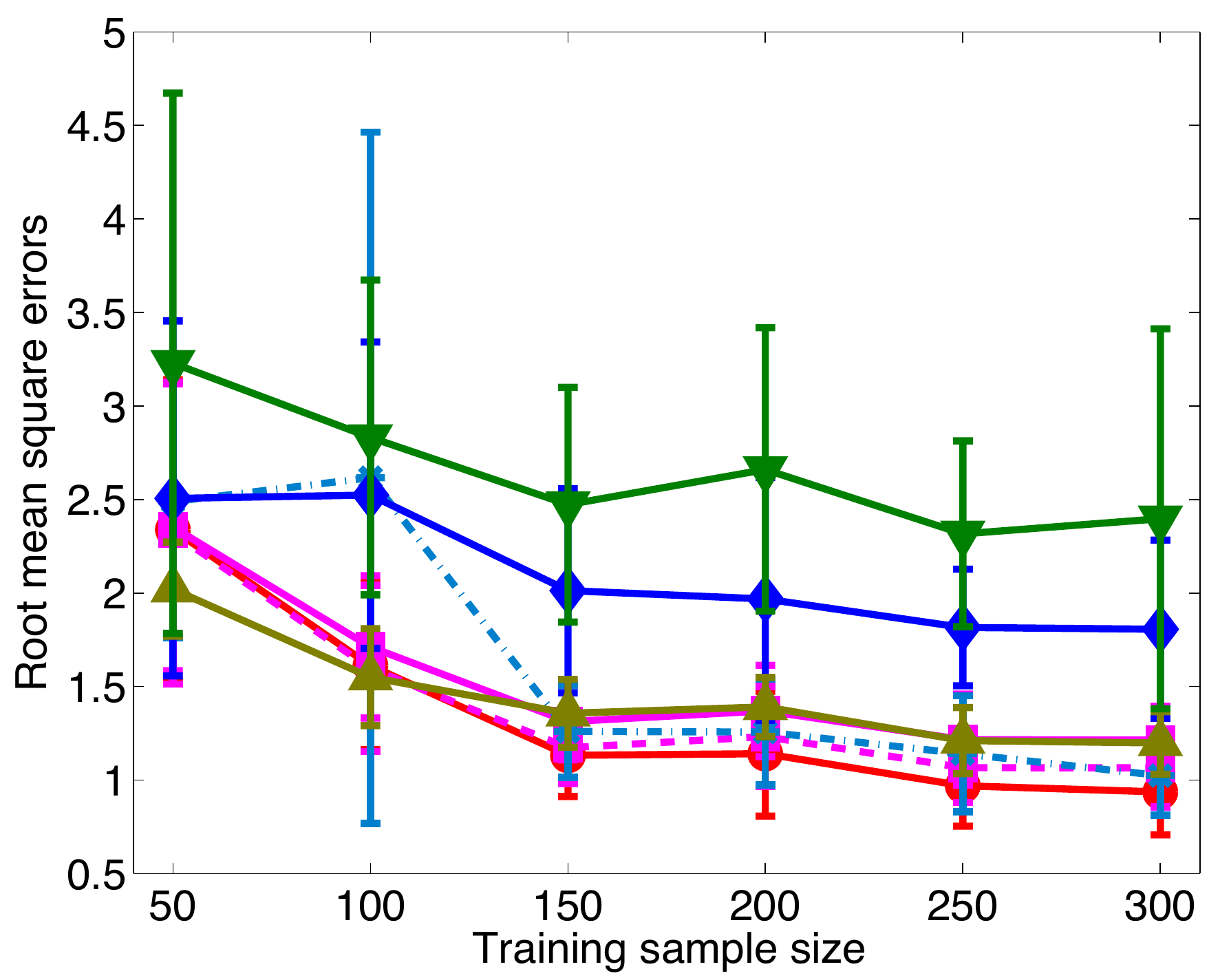}
	}%
	 \subfigure[RMSE (SSM 4a)]{
			\includegraphics[width=0.43\columnwidth]{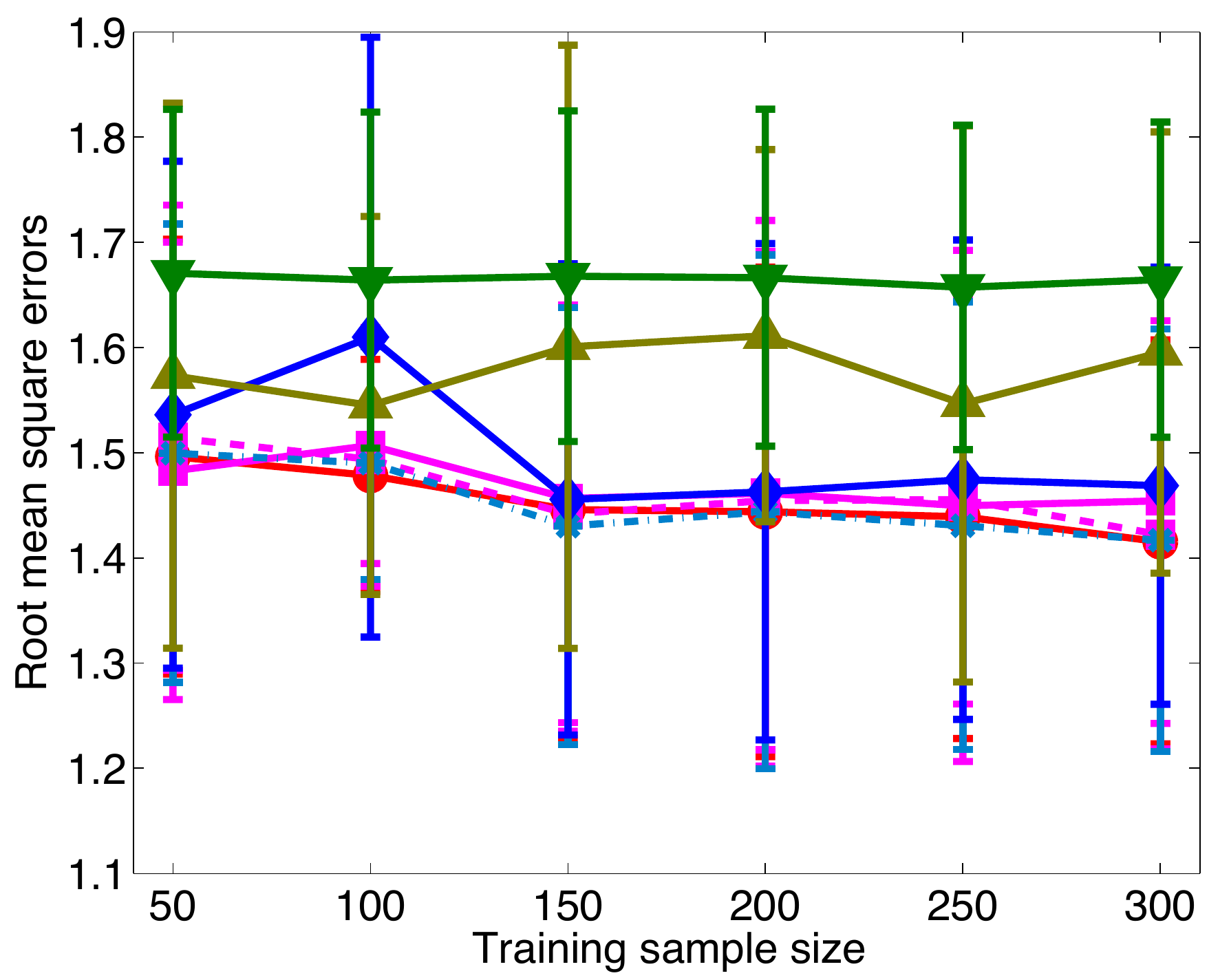}
	}%
\caption{RMSE of the synthetic experiments in Section \ref{sec:exp_synthetic}. The state-space models of these figures have no control in their transition models.}
\label{fig:artificial_woCont}
\end{center}
\end{figure}

\begin{figure}[t]
\vskip -0.05in
\begin{center}
	 \subfigure[RMSE (SSM 1b)]{
			\includegraphics[width=0.43\columnwidth]{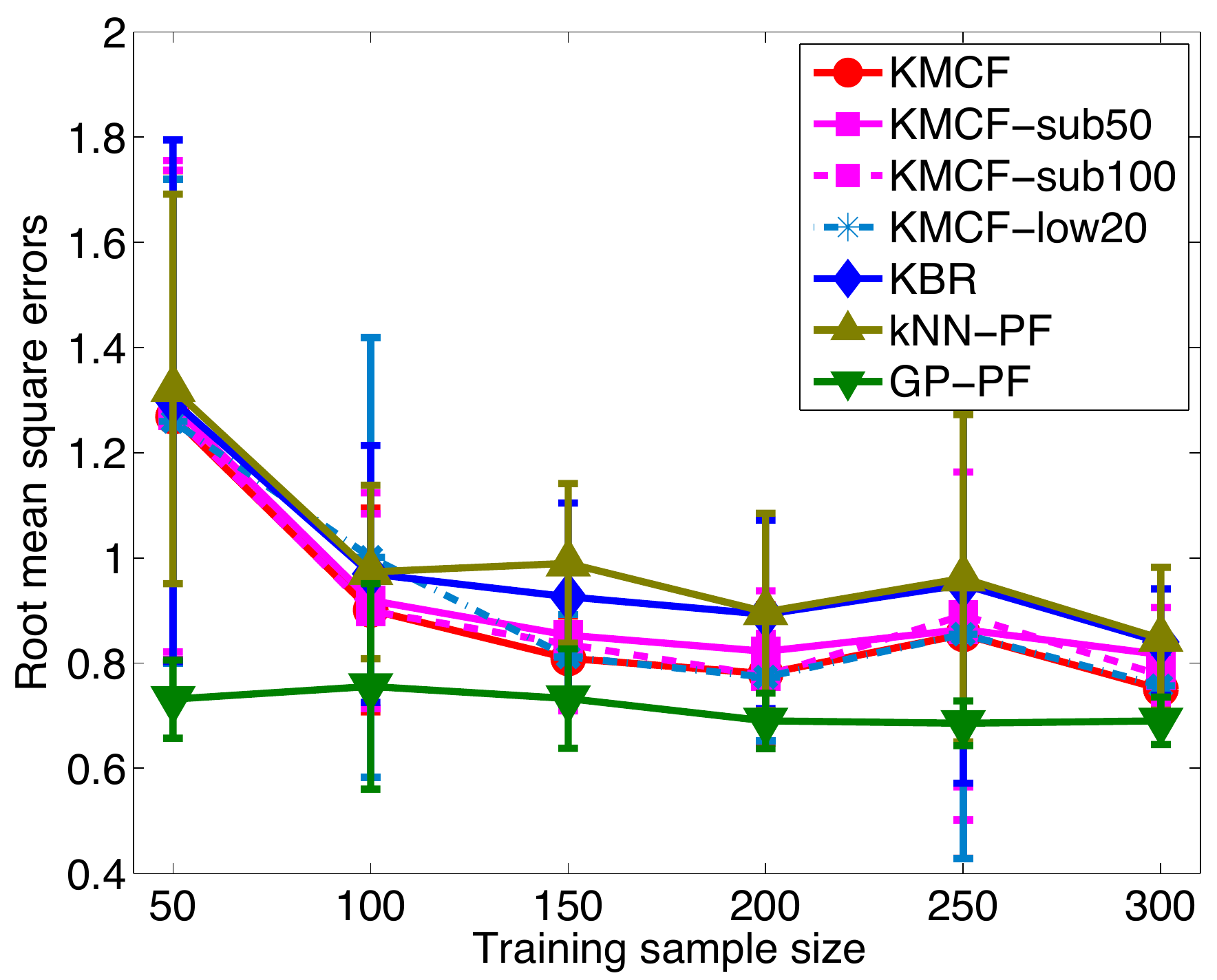}
	}%
	 \subfigure[RMSE (SSM 2b)]{
			\includegraphics[width=0.43\columnwidth]{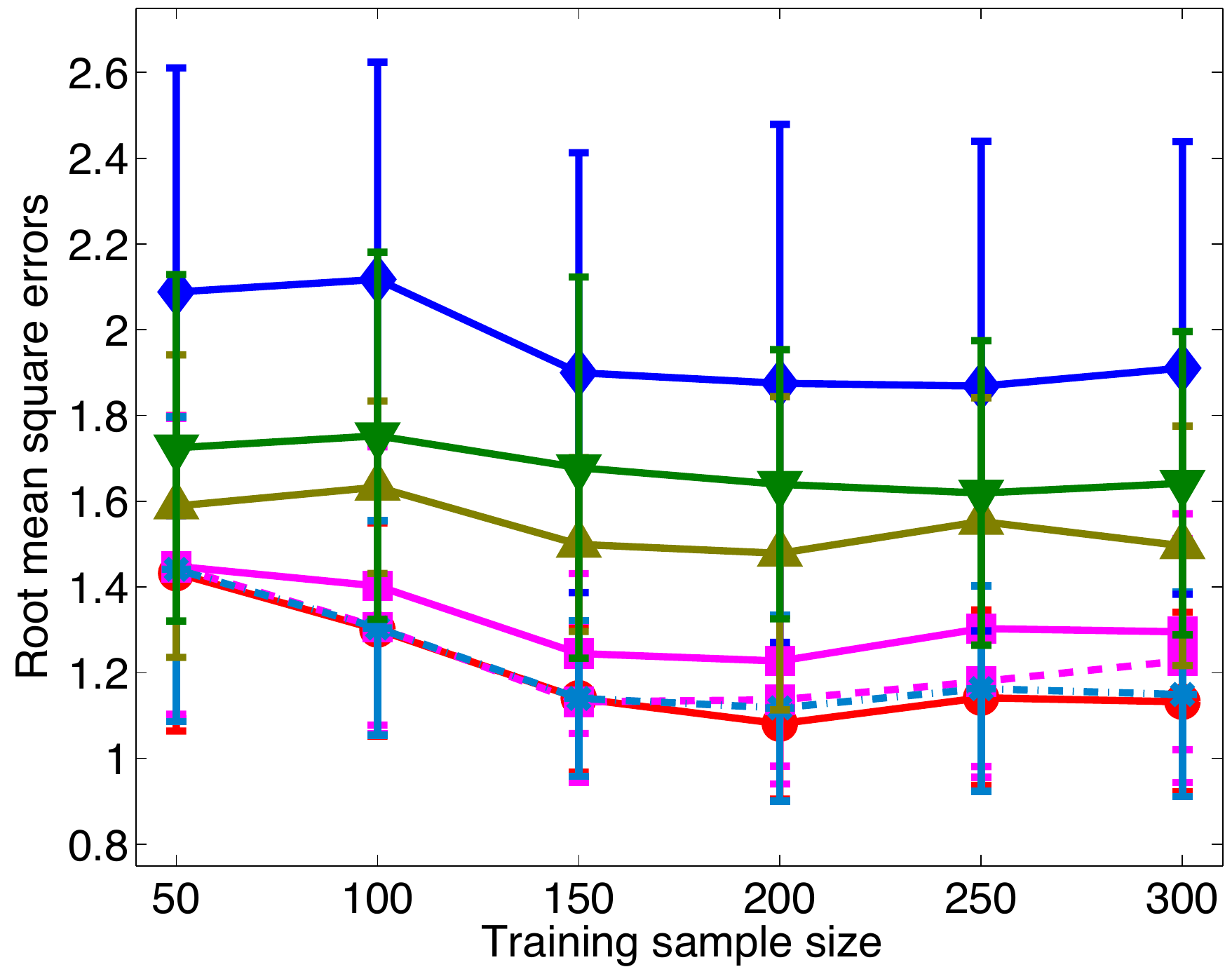}
	}%
	 \\\subfigure[RMSE (SSM 3b)]{
			\includegraphics[width=0.43\columnwidth]{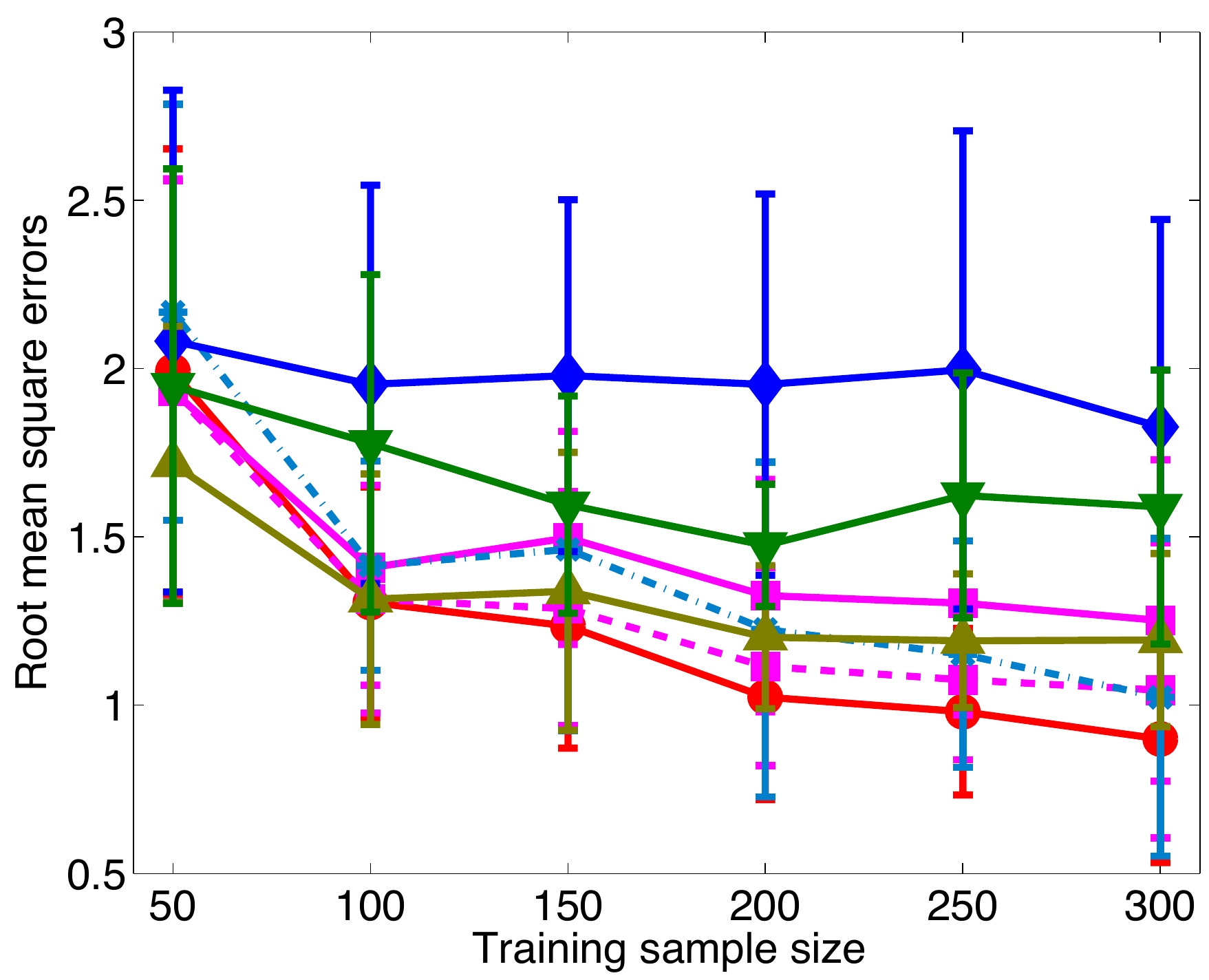}
	}%
	 \subfigure[RMSE (SSM 4b)]{
			\includegraphics[width=0.43\columnwidth]{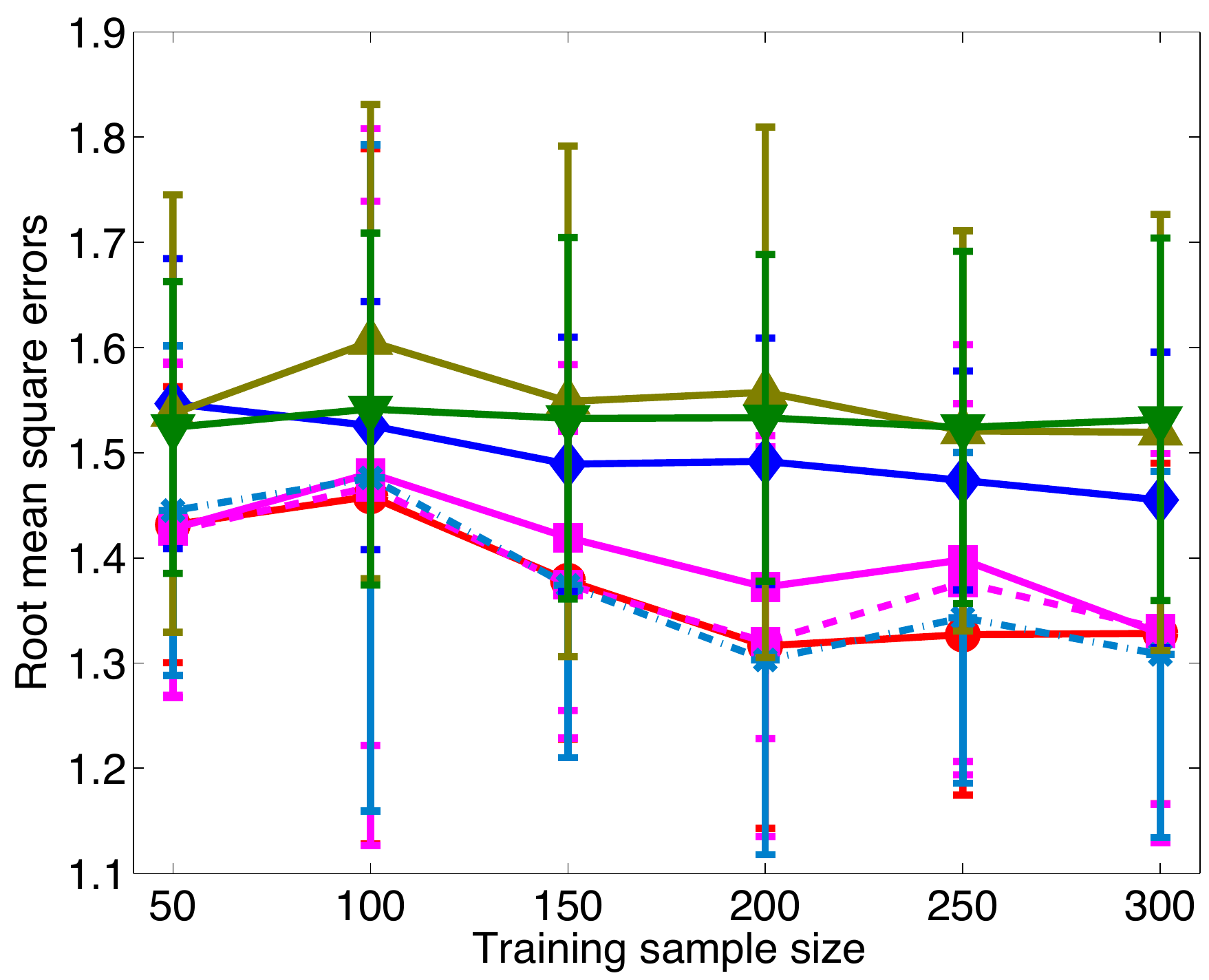}
	}%
\caption{RMSE of synthetic experiments in Section \ref{sec:exp_synthetic}. The state-space models of these figures include control $u_t$ in their transition models.}
\label{fig:artificial_Cont}
\end{center}
\end{figure}

\begin{figure}[t]
\begin{center}
\includegraphics[width=0.43\columnwidth]{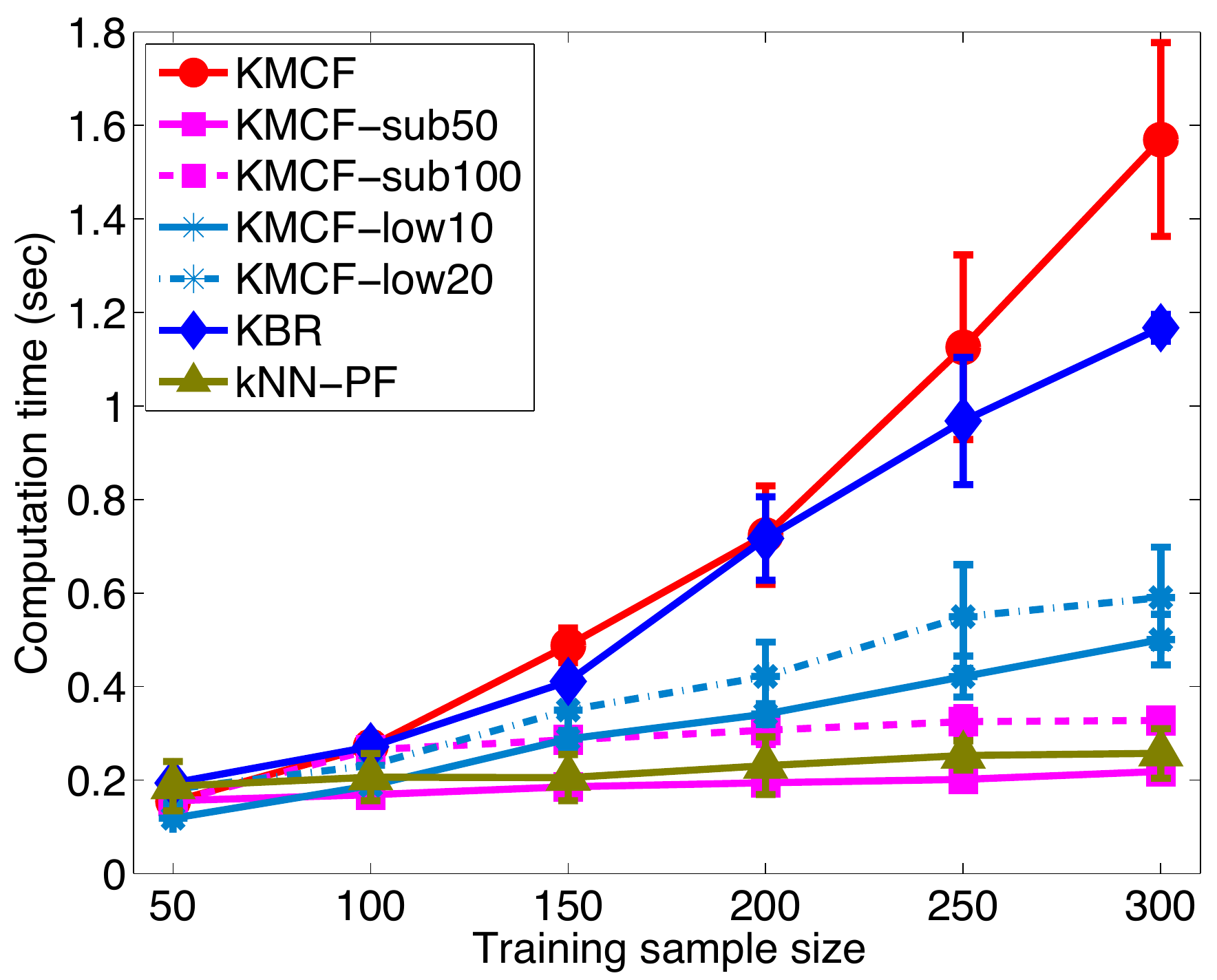}
\includegraphics[width=0.43\columnwidth]{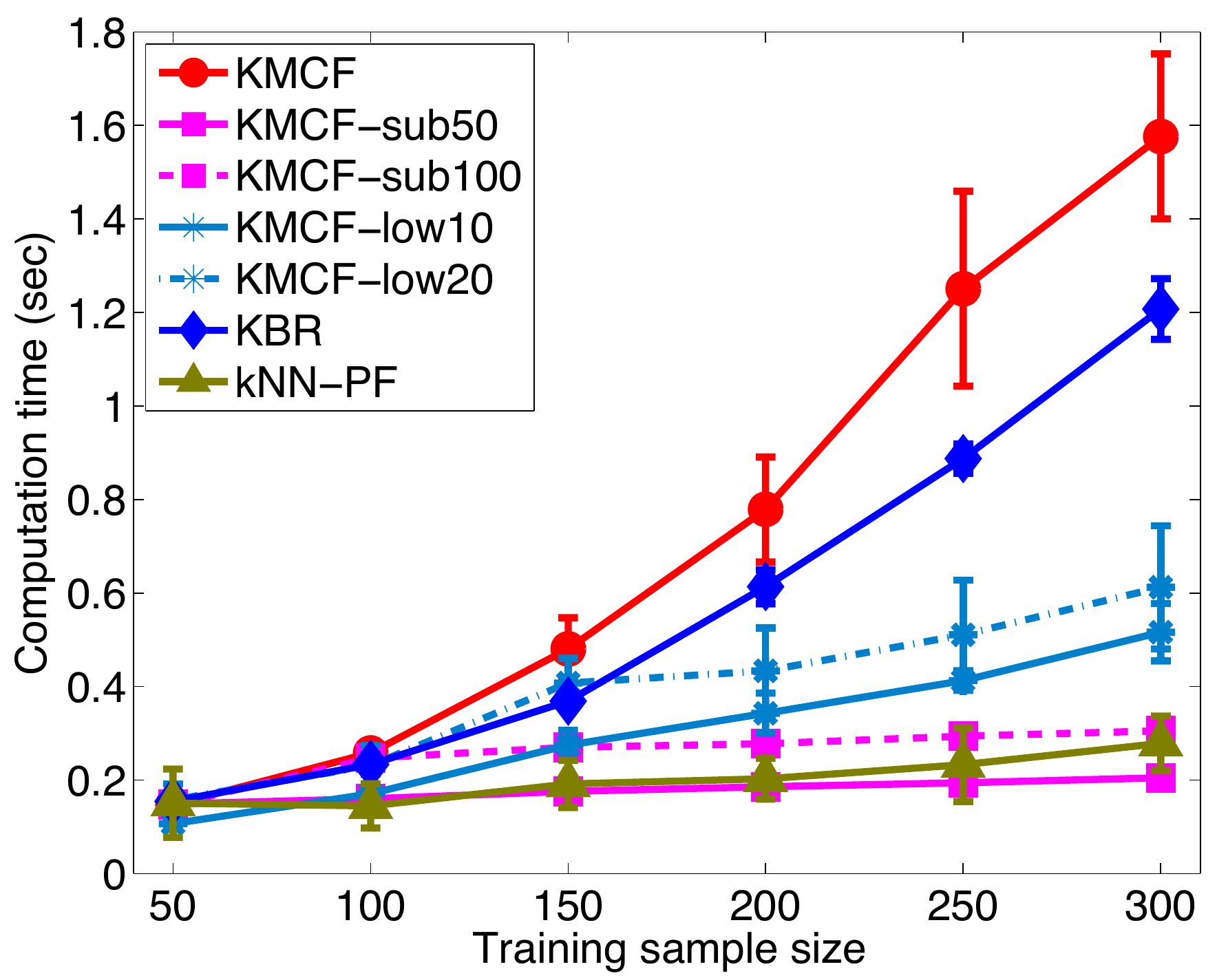}
\caption{Computation time of synthetic experiments in Section \ref{sec:exp_synthetic}. Left: SSM 1a. Right: SSM 1b.}
\label{fig:synthetic_time}
\end{center}
\end{figure}

GP-PF performed the best for SSM 1a and 1b. This may be because these models fit the assumption of GP-regression, as their noise are additive Gaussian.
For the other models, however, GP-PF performed poorly; the observation models of these models have strong nonlinearities and the noise are not additive Gaussian.
For these models, KMCF performed the best or competitively with the other methods.
This indicates that KMCF successfully exploits the state-observation examples $\{ (X_i,Y_i) \}_{i=1}^n$ in dealing with the complicated observation models.
Recall that our focus has been on situations where the relation between states and observations are so complicated that the observation model is not known; the results indicate that KMCF is promising for such situations.
On the other hand, KBR filter performed worse than KMCF for the most of the models.
KBF filter also uses Kernel Bayes' Rule as KMCF.
The difference is that KMCF makes use of the transition models directly by sampling, while KBR filter must learn the transition models from training data for state transitions.
This indicates that the incorporation of the knowledge expressed in the transition model is very important for the filtering performance.
This can also be seen by comparing Figure \ref{fig:artificial_woCont} and Figure \ref{fig:artificial_Cont}.
The performance of the methods other than KBR filter improved for SSMs \{1b, 2b, 3b, 4b\}, compared to the performance for the corresponding models in SSMs \{1a, 2a, 3a, 4a\}.
Recall that SSMs \{1b, 2b, 3b, 4b\} include control $u_t$ in their transition models. The information of control input is helpful for filtering in general.
Thus the improvements suggest that KMCF, kNN-PF and GP-PF successfully incorporate the information of controls: they achieve this simply by sampling with $p(x_t | x_{t-1},u_t)$.
On the other hand, KBF filter must learn the transition model $p(x_t | x_{t-1}, u_t)$; this can be harder than learning the transition model $p(x_t | x_{t-1})$ that has no control input.

We next compare computation time (Figure \ref{fig:synthetic_time}).
KMCF was competitive or even slower than the KBR filter.
This is due to the resampling step in KMCF.
The speeding up methods (KMCF-low10, KMCF-low20, KMCF-sub50 and KMCF-sub100) successfully reduced the costs of KMCF.
KMCF-low10 and KMCF-low20 scaled linearly to the sample size $n$; this matches the fact that Algorithm \ref{al:KBR_lowrank} reduces the costs of Kernel Bayes' Rule to $O(n r^2)$.
On the other hand, the costs of KMCF-sub50 and KMCF-sub100 remained almost the same amounts over the difference sample sizes.
This is because they reduce the sample size itself from $n$ to $r$, so the costs are reduced to $O(r^3)$ (see Algorithm \ref{al:subsampling}).
KMCF-sub50 and KMCF-sub100 are competitive to kNN-PF, which is fast as it only needs kNN searches to deal with the training sample $\{ (X_i,Y_i) \}_{i=1}^n$.
In Figure \ref{fig:artificial_woCont} and \ref{fig:artificial_Cont}, KMCF-low20 and KMCF-sub100 produced the results competitive to KMCF for SSMs \{1a, 2a, 4a, 1b, 2b, 4b\}.
Thus for these models, such methods reduce the computational costs of KMCF without loosing much accuracy.
KMCF-sub50 was slightly worse than KMCF-100. This indicates that the number of subsamples cannot be reduced to this extent if we wish to maintain the accuracy.
For SSM 3a and 3b, the performance of KMCF-low20 and KMCF-sub100 were worse than KMCF, in contrast to the performance for the other models.
The difference of SSM 3a and 3b from the other models is that the observation space is 10-dimensional: $\Y = \R^{10}$.
This suggests that if the dimension is high, $r$ needs to be large to maintain the accuracy (recall that $r$ is the rank of low rank matrices in Algorithm \ref{al:KBR_lowrank}, and the number of subsamples in Algorithm \ref{al:subsampling}).
This is also implied by the experiments in the next subsection.

\subsection{Vision-based mobile robot localization} \label{sec:exp_robot}
We applied KMCF to the problem of vision-based mobile robot localization \citep{Vlassis2001,WolBurBur05,QuiStaCoaThr10}.
We consider a robot moving in a building. 
The robot takes images with its vision camera as it moves.
Thus the vision images form a sequence of observations $y_1,\dots,y_T$ in time series; each $y_t$ is an image.
On the other hand, the robot does not know its positions in the building;
we define state $x_t$ as the robot's position at time $t$.
The robot wishes to estimate its position $x_t$ from the sequence of its vision images $y_1,\dots,y_t$.
This can be done by filtering, i.e., by estimating the posteriors $p(x_t | y_1,\dots,y_t)$ $(t=1,\dots,T)$.
This is the robot localization problem.
It is fundamental in robotics, as a basis for more involved applications such as navigation and reinforcement learning \citep{thurun2002}.

The state-space model is defined as follows: 
the observation model $p(y_t | x_t)$ is the conditional distribution of images given position, which is very complicated and considered unknown. 
We need to assume position-image examples $\{ (X_i,Y_i) \}_{i=1}^n$; these samples are given in the dataset described below.
The transition model $p(x_t | x_{t-1}) := p(x_t | x_{t-1}, u_t)$ is the conditional distribution of the current position given the previous one.
This involves a control input $u_t$ that specifies the movement of the robot. 
In the dataset we use, the control is given as odometry measurements.
Thus we define $p(x_t | x_{t-1}, u_t)$ as the {\em odometry motion model}, which is fairly standard in robotics \citep{thurun2002}.
Specifically, we used the algorithm described in Table 5.6 of  \cite{thurun2002}, with all of its parameters fixed to $0.1$.
The prior $p_{\rm init}$ of the initial position $x_1$ is defined as a uniform distribution over the samples $X_1,\dots,X_n$ in $\{ (X_i,Y_i) \}_{i=1}^n$.

As a kernel $k_\Y$ for observations (images), we used the Spatial Pyramid Matching Kernel of \cite{Lazebnik2006}.
This is a positive definite kernel developed in the computer vision community, and is also fairly standard.
Specifically, we set the parameters of this kernel as suggested in \cite{Lazebnik2006}: this gives a 4200 dimensional histogram for each image.
We defined the kernel $k_\X$ for states (positions) as Gaussian.
Here the state space is the $4$-dimensional space: $\X = \R^4$: two dimensions for location, and the rest for the orientation of the robot.\footnote{We projected the robot's orientation in $[0,2\pi]$ onto the unit circle in $\R^2$.}

The dataset we used is the COLD database \citep{cold}, which is publicly available. 
Specifically, we used the dataset {\em Freiburg, Part A, Path 1, cloudy}.
This dataset consists of three similar trajectories of a robot moving in a building, each of which provides position-image pairs $\{ (x_t, y_t) \}_{t=1}^T$.
We used two trajectories for training and validation, and the rest for test.
We made state-observation examples $\{ (X_i,Y_i) \}_{i=1}^n$ by randomly subsampling the pairs in the trajectory for training.
Note that the difficulty of localization may depend on the time interval (i.e.,\ the interval between $t$ and $t-1$ in sec.) 
Therefore we made three test sets (and training samples for state transitions in KBR filter) with different time intervals: $2.27$ sec.\  ($T=168$), $4.54$ sec.\  ($T=84$) and $6.81$ sec.\  ($T=56$).

In these experiments, we compared KMCF with three methods: kNN-PF, KBR filter, and the naive method (NAI) defined below. 
For KBR filter, we also defined the Gaussian kernel on the control $u_t$, i.e.,\ on the difference of odometry measurements at time $t-1$ and $t$.
The naive method (NAI) estimates the state $x_t$ as a point $X_j$ in the training set $\{ (X_i,Y_i) \}$ such that the corresponding observation $Y_j$ is closest to the observation $y_t$.
We performed this as a baseline.
We also used the Spatial Pyramid Matching Kernel for these methods (for kNN-PF and NAI, as a similarity measure of the nearest neighbors search).
We did not compare with GP-PF, since it assumes that observations are real vectors and thus cannot be applied to this problem straightforwardly.
We determined the hyper-parameters in each method by cross validation.
To reduced the cost of the resampling step in KMCF, we used the method discussed in Section  \ref{sec:theory_resampling} with $\ell = 100$.
The low rank approximation method (Algorithm \ref{al:KBR_lowrank}) and the subsampling method (Algorithm \ref{al:subsampling}) were also applied to reduce the computational costs of KMCF.
Specifically, we set $r= 50, 100$ for Algorithm \ref{al:KBR_lowrank} (described as KMCF-low50 and KMCF-low100 in the results below), and $r = 150, 300$ for Algorithm \ref{al:subsampling} (KMCF-sub150 and KMCF-sub300).

Note that in this problem, the posteriors $p(x_t | y_{1:t})$ can be highly multimodal. 
This is because similar images appear in distant locations.
Therefore the posterior mean $\int x_t p(x_t | y_{1:t}) dx_t$ is not appropriate for point estimation of the ground-truth position $x_t$.
Thus for KMCF and KBR filter, we employed the heuristic for mode estimation explained in Section \ref{sec:decode}.
For kNN-PF, we used a particle with maximum weight for the point estimation.
We evaluated the performance of each method by RMSE of location estimates.
We ran each experiment 20 times for each training set of different size.

\paragraph{Results.}
First, we demonstrate the behaviors of KMCF with this localization problem. Figures \ref{fig:robot_demo1} and  \ref{fig:robot_demo2} show iterations of KMCF with $n=400$, applied to the test data with time interval $6.81$ sec.
Figure \ref{fig:robot_demo1} illustrates iterations that produced accurate estimates, while Figure \ref{fig:robot_demo2} describes situations where location estimation is difficult.

\begin{figure}[t]
\vskip -0.05in
\begin{center}
	\subfigure[$t=29$. $\| \hat{x}_t - x_t \| = 0.26378$.]{\includegraphics[width=0.49\columnwidth]{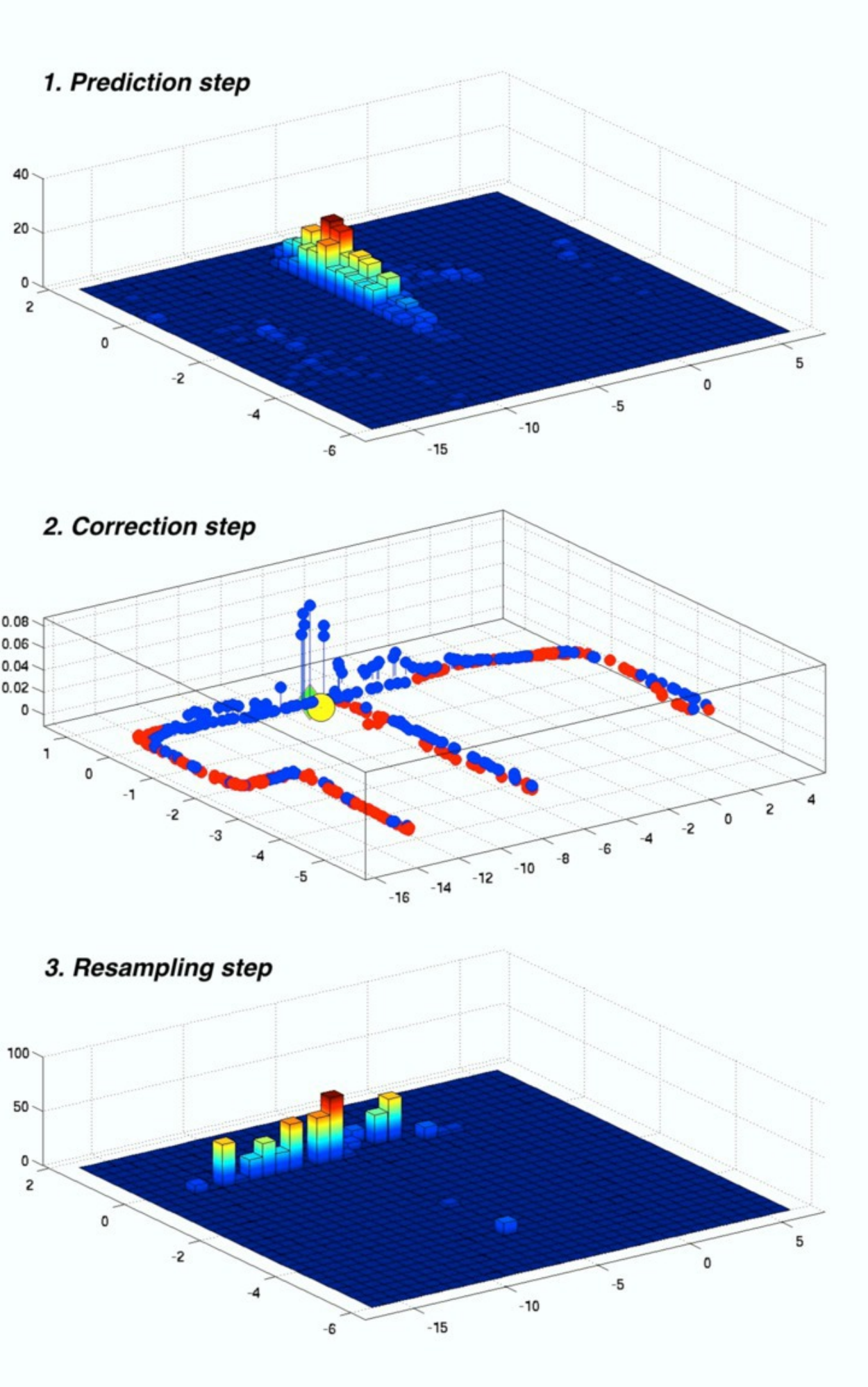}}%
	\subfigure[$t=43$. $\| \hat{x}_t - x_t \| = 0.26315$. ]{\includegraphics[width=0.49\columnwidth]{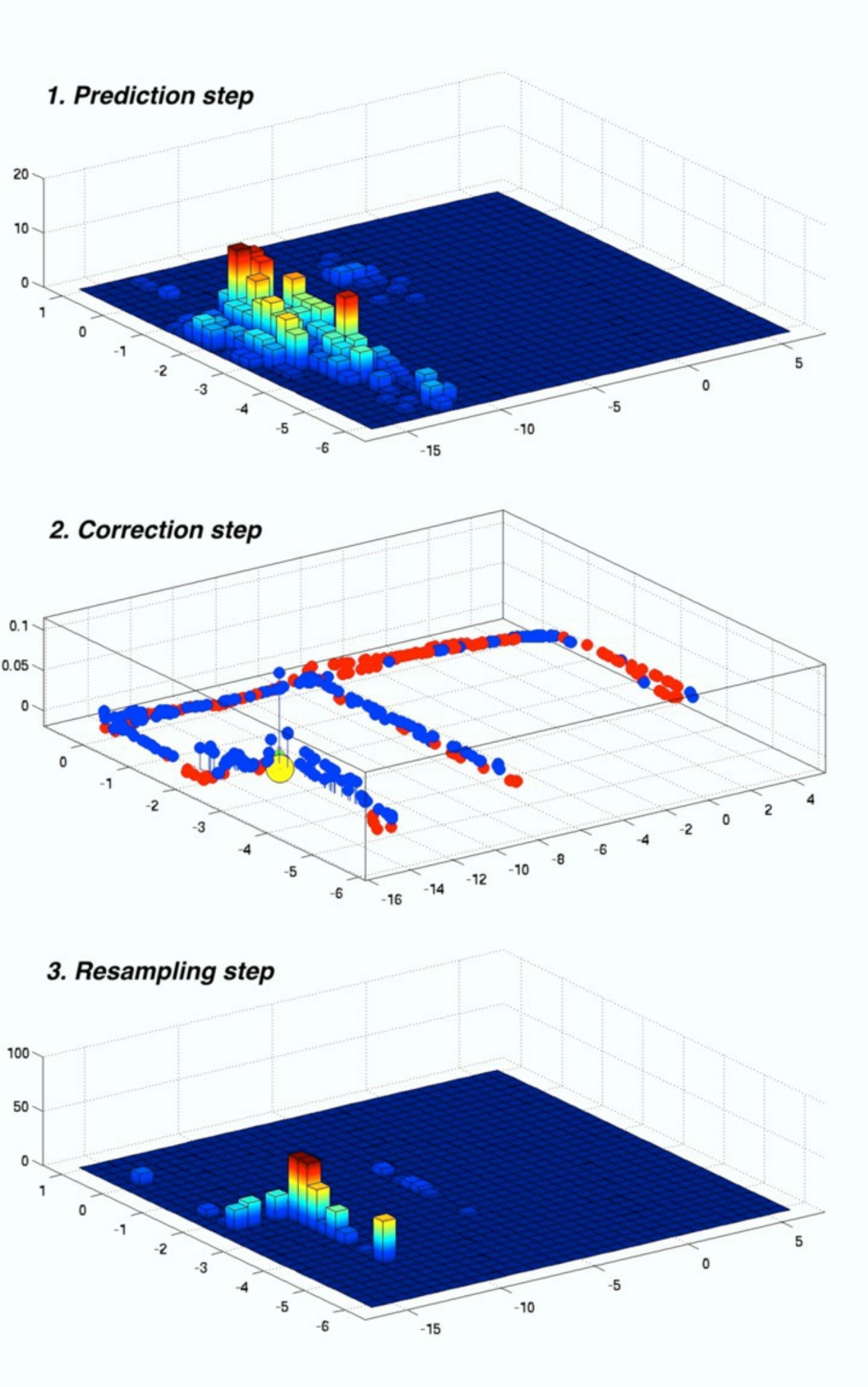}}
\caption{Demonstration results. Each column corresponds to one iteration of KMCF. Top (prediction step): histogram of samples for prior. Middle (correction step): weighted samples for posterior. The blue and red stems indicate positive and negative weights, respectively. The yellow ball represents the ground-truth location $x_t$, and the green diamond the estimated one $\hat{x}_t$. Bottom (resampling step): histogram of samples given by the resampling step. }
\label{fig:robot_demo1}
\end{center}
\end{figure}

\begin{figure}[t]
\vskip -0.05in
\begin{center}
	\subfigure[$t=11$. $\| \hat{x}_t - x_t \| = 2.3443$.]{\includegraphics[width=0.49\columnwidth]{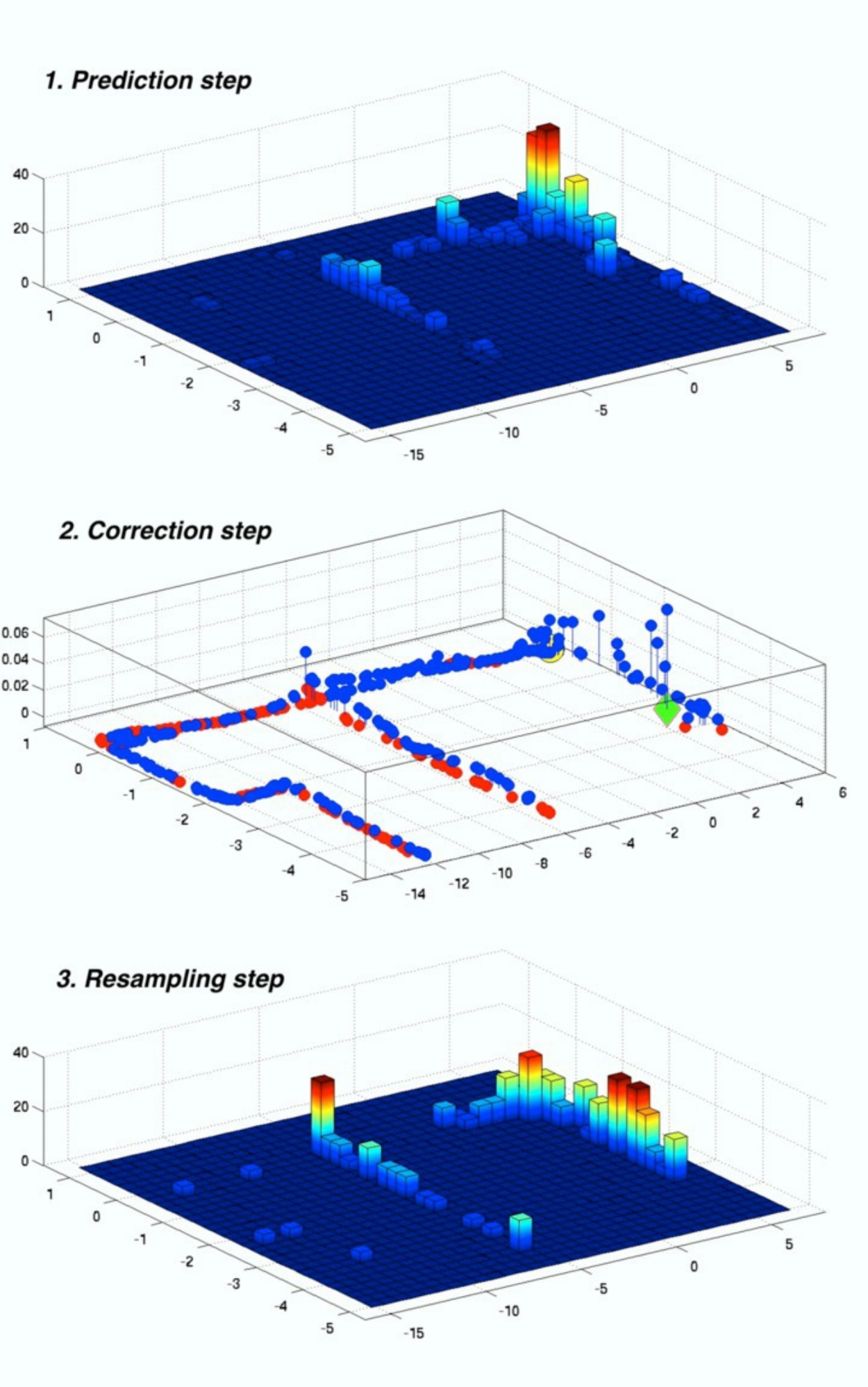}} %
	\subfigure[$t=40$. $\| \hat{x}_t - x_t \| = 0.3273$. ]{\includegraphics[width=0.49\columnwidth]{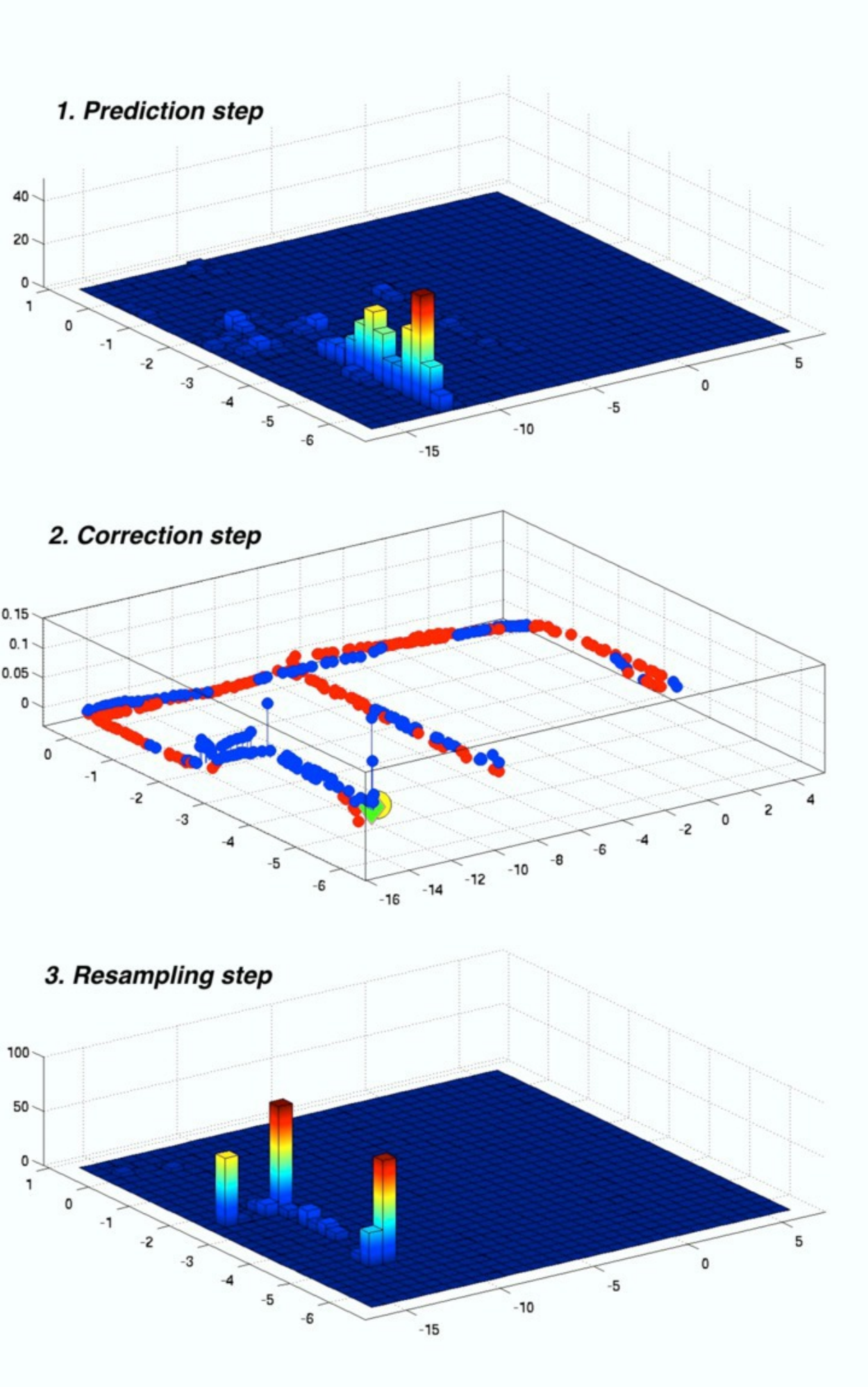}}
\caption{Demonstration results (see also the caption of Figure \ref{fig:demo1}). Here we show time points where observed images are similar to those in distant places. Such a situation often occurs at corners, and makes location estimation difficult.   (a) The prior estimate is reasonable, but the resulting posterior has modes in distant places. This makes the location estimate (green diamond) far from the true location (yellow ball). (b) While the location estimate is very accurate, modes also appear at distant locations.}
\label{fig:robot_demo2}
\end{center}
\end{figure}

Figures \ref{fig:robot_RMSE}  and \ref{fig:robot_Time} show the results in RMSE and computational time, respectively.
For all the results KMCF and that with the computational reduction methods (KMCF-low50, KMCF-low100, KMCF-sub150 and KMCF-sub300) performed better than KBR filter.
These results show the benefit of directly manipulating the transition models with sampling.
KMCF was competitive with kNN-PF for the interval 2.27 sec.; note that kNN-PF was originally proposed for the robot localization problem.
For the results with the longer time intervals (4.54 sec.\  and 6.81 sec.), KMCF outperformed kNN-PF.

\begin{figure}[t]
\vskip -0.05in
\begin{center}
	\subfigure[RMSE (time interval: 2.27 sec; $T = 168$)]{
			\includegraphics[width=0.43\columnwidth]{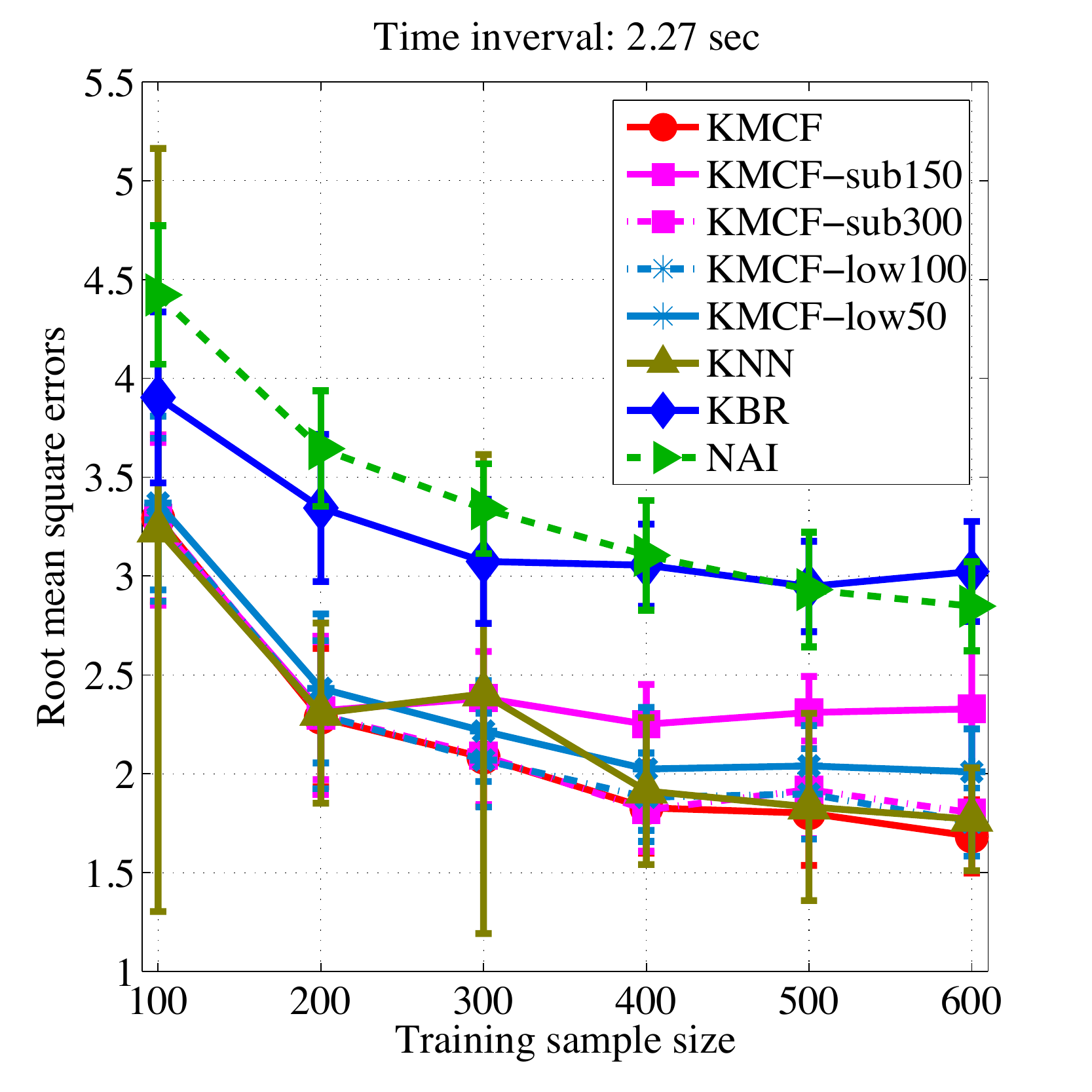}
	}%
	\\ \subfigure[RMSE (time interval 4.54 sec; $T=84$)]{
			\includegraphics[width=0.43\columnwidth]{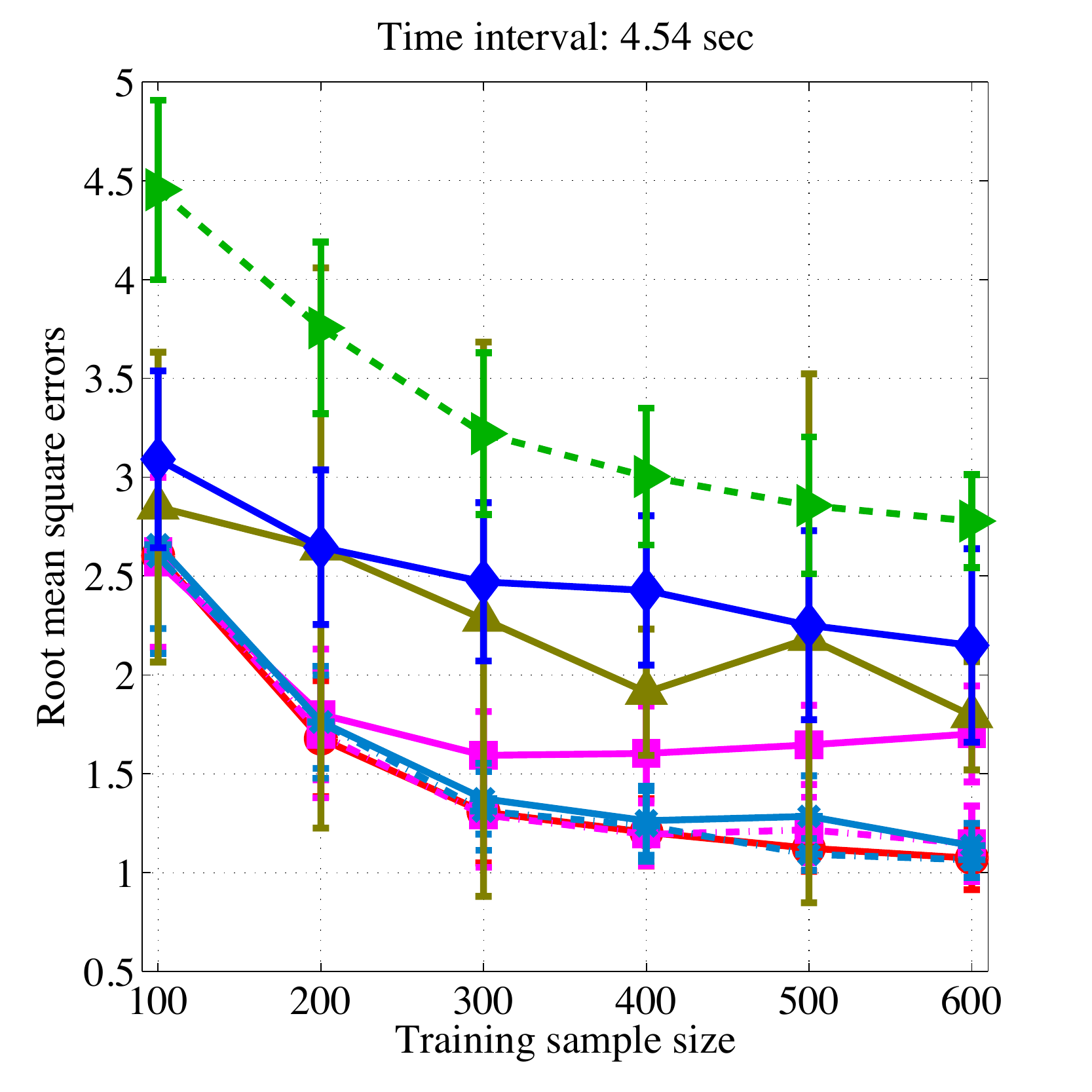}
	}%
	 \subfigure[RMSE (time interval 6.81 sec; $T=56$)]{
			\includegraphics[width=0.43\columnwidth]{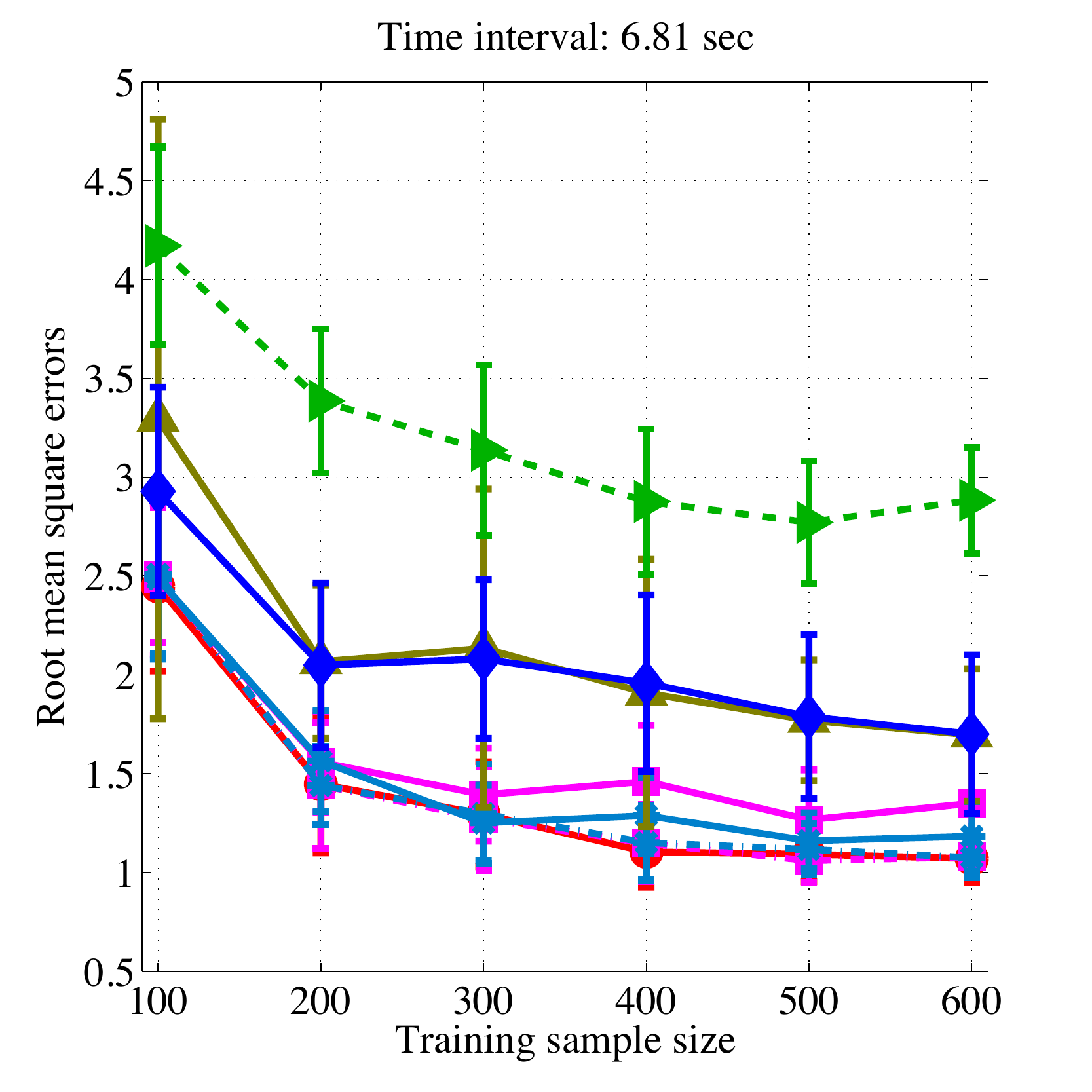}
	}%
\caption{RMSE of the robot localization experiments in Section \ref{sec:exp_robot}. (a), (b) and (c) show the cases for time interval 2.27 sec.\ , 4.54 sec.\  and 6.81 sec., respectively.  }
\label{fig:robot_RMSE}
\end{center}
\end{figure}

\begin{figure}[t]
\vskip -0.05in
\begin{center}
	\subfigure[Computation time (sec.) ($T = 168$)]{
			\includegraphics[width=0.43\columnwidth]{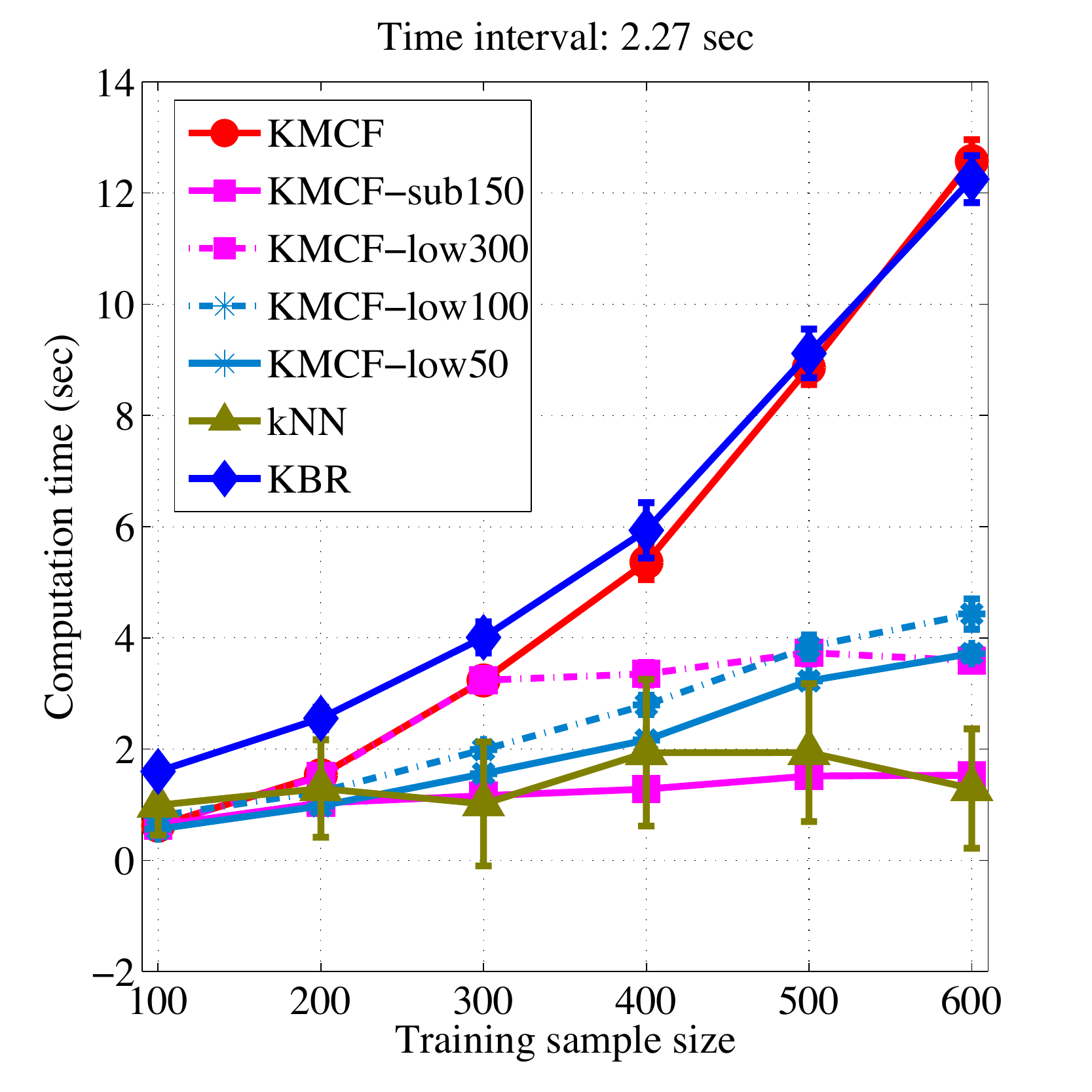}
	}%
	\\ \subfigure[Computation time (sec.) ($T = 84$)]{
			\includegraphics[width=0.43\columnwidth]{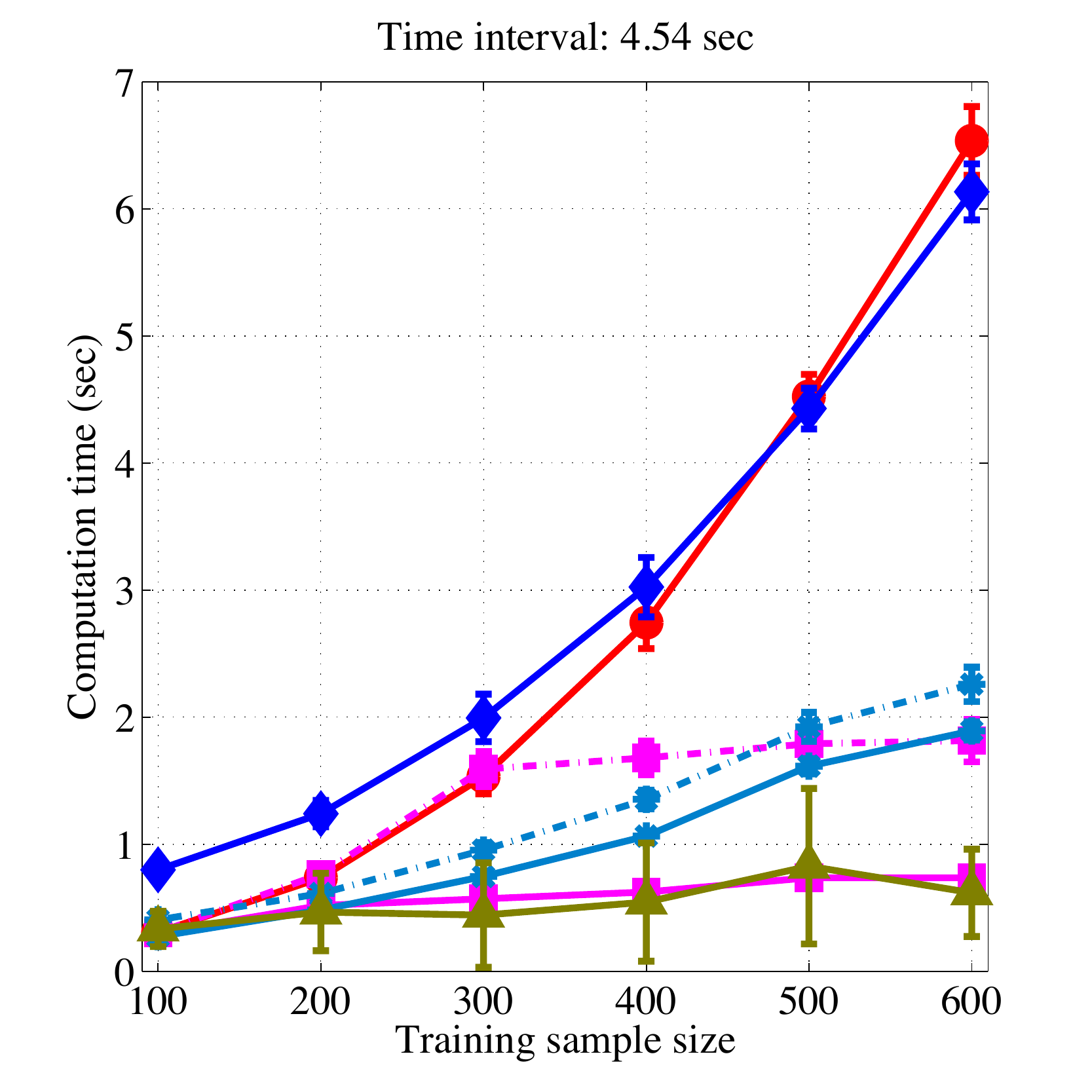}
	}%
	\subfigure[Computation time (sec.) ($T=56$)]{
			\includegraphics[width=0.43\columnwidth]{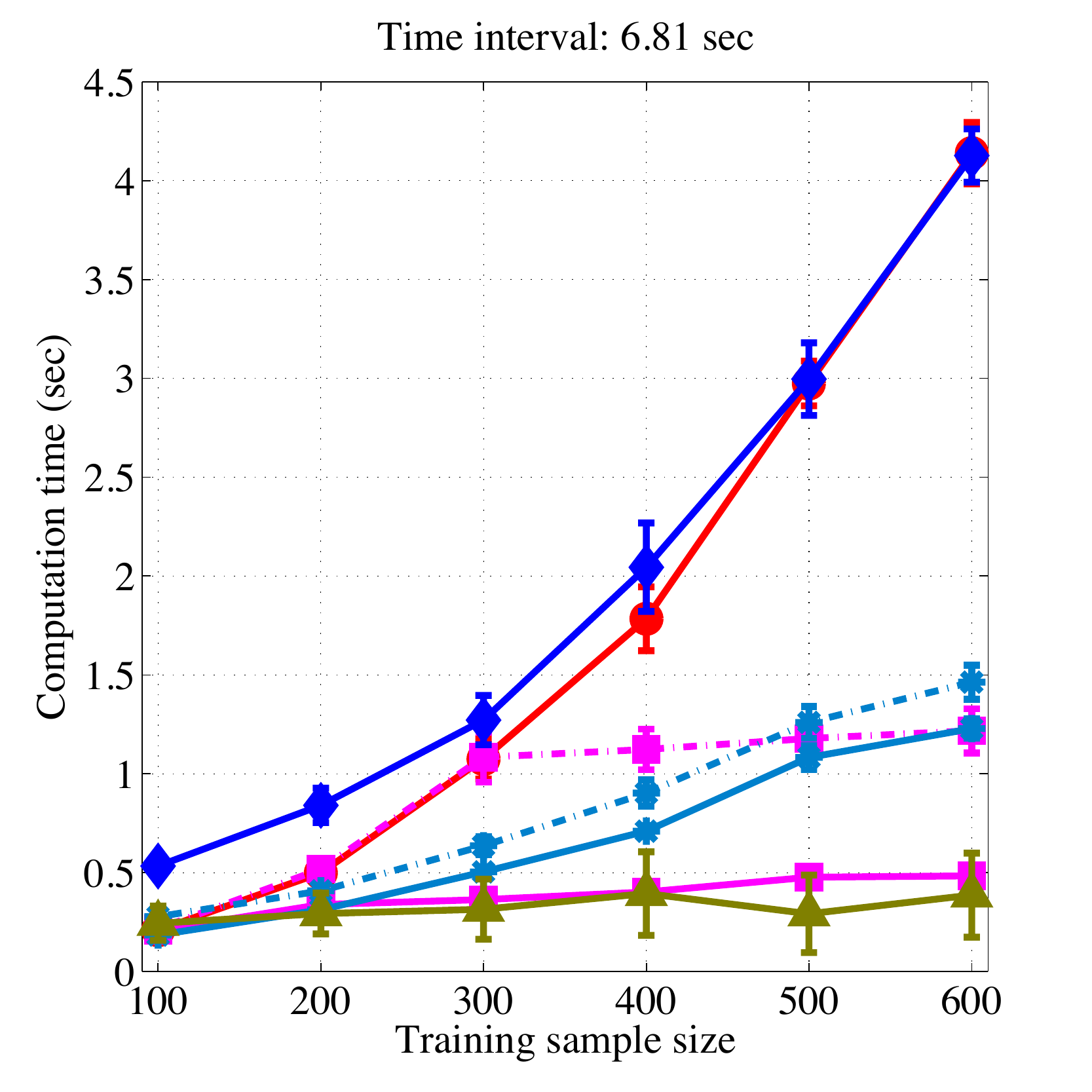}
	}%
\caption{Computation time of the localization experiments in Section \ref{sec:exp_robot}. (a), (b) and (c) show the cases for time interval 2.27 sec.\ , 4.54 sec.\  and 6.81 sec., respectively.  Note that the results show the run time of each method.}
\label{fig:robot_Time}
\end{center}
\end{figure}

We next investigate the effect on KMCF of the methods to reduce computational cost.
The performance of KMCF-low100 and KMCF-sub300 are competitive with KMCF; those of KMCF-low50 and KMCF-sub150 degrade as the sample size increases. 
Note that $r=50,100$ for Algorithm \ref{al:KBR_lowrank} are larger than those in Section \ref{sec:exp_synthetic}, though the values of the sample size $n$ are larger than those in Section \ref{sec:exp_synthetic}.
Also note that the performance of KMCF-sub150 is much worse than KMCF-sub300.
These results indicate that we may need large values for $r$ to maintain the accuracy for this localization problem.
Recall that the Spatial Pyramid Matching Kernel gives essentially a high-dimensional feature vector (histogram) for each observation.
Thus the observation space $\Y$ may be considered high-dimensional. 
This supports the hypothesis in Section \ref{sec:exp_synthetic} that if the dimension is high, the computational cost reduction methods may require larger $r$ to maintain accuracy.

Finally, let us look at the results in computation time (Figure \ref{fig:robot_Time}).
The results are similar to those in Section \ref{sec:exp_synthetic}.
Even though the values for $r$ are relatively large, Algorithm \ref{al:KBR_lowrank} and Algorithm \ref{al:subsampling} successfully reduced the computational costs of KMCF.

\section{Conclusions and future work}
This paper proposed Kernel Monte Carlo Filter, a novel filtering method for state-space models.
We have considered the situation where the observation model is not known explicitly or even parametrically, and where examples of the state-observation relation are given instead of the observation model.
Our approach was based on the framework of kernel mean embeddings, which enables us to deal with the observation model in a data-driven manner.
The resulting filtering method consists of the prediction, correction and resampling steps, all of which were realized in terms of kernel mean embeddings.
Methodological novelties lie in the prediction and resampling steps.
Thus we analyzed their behaviors, by deriving error bounds for the estimator of the prediction step.
The analysis revealed that the effective sample size of a weighted sample plays an important role, as in particle methods.
This analysis also explained how our resampling algorithm works.
We applied the proposed method to synthetic and real problems, confirming the effectiveness of our approach.

One interesting topic for future research would be parameter estimation for the transition model. In this paper we did not discuss this, and assumed that parameters are given and fixed, if exist. If the state observation examples $\{ (X_i,Y_i) \}_{i=1}^n$ are given as a sequence from the state-space model, then we can use the state samples $X_1,\dots,X_n$ for estimating those parameters. Otherwise, we need to estimate the parameters based on test data.
This might be possible by exploiting approaches for parameter estimation in particle methods (e.g.,\ Section IV in \cite{CapCodMou07}).

Another important topic is on the situation where the observation model in the test and training phases are different. As discussed in Section \ref{sec:overview_complexity}, this might be addressed by exploiting the framework of transfer learning \citep{PanYan10}. This would require extension of kernel mean embeddings to the setting of transfer learning, since there has been no work in this direction. We consider that such extension is interesting in its own right.

\subsection*{Acknowledgments}
We would like to express our gratitude to the associate editor and the anonymous reviewer for their time and helpful suggestions.
We also thank Masashi Shimbo, Momoko Hayamizu, Yoshimasa Uematsu and Katsuhiro Omae for their helpful comments.
This work has been supported in part by MEXT Grant-in-Aid for Scientific Research on Innovative Areas 25120012.
MK has been supported by JSPS Grant-in-Aid for JSPS Fellows 15J04406.

\section*{Appendix}
\appendix
\setcounter{lemma}{0}
\renewcommand{\thelemma}{\Alph{section}\arabic{lemma}}
\setcounter{prop}{0}
\renewcommand{\theprop}{\Alph{section}\arabic{prop}}
\renewcommand{\theequation}{\Alph{section}\arabic{equation}}
\setcounter{equation}{0}

\def\thesection{\Alph{section}}
\section{Proof of Theorem \ref{theo:finite_sample_bound}} \label{sec:appendix_proof}
Before going to the proof, we review some basic facts that will be needed.
Let $m_P = \int k_\X(\cd, x) dP(x)$ and $\hm_P = \sum_{i=1}^n w_i k_\X(\cd,X_i)$. By the reproducing property of the kernel $k_\X$, the following hold for any $f \in \H_\X$:
\begin{eqnarray}
\left< m_P, f \right>_{\H_\X} &=& \left< \int k_\X(\cd,x) dP(x), f\right>_{\H_\X} = \int \left< k_\X(\cd,x), f \right>_{\H_\X} dP(x) \nonumber \\
 &=&  \int f(x) dP(x) = \E_{X \sim P}[f(X)]. \label{eq:mP_repducing} \\
\left< \hm_P, f  \right>_{\H_\X} &=& \left< \sum_{i=1}^n w_i k_\X(\cd,X_i), f   \right>_{\H_\X} = \sum_{i=1}^n w_i f(X_i).\label{eq:hmP_repducing} 
\end{eqnarray}

For any $f, g \in \H_\X$, we denote by $f \otimes g \in \H_\X \otimes \H_\X$ the tensor product of $f$ and $g$ defined as 
\begin{equation}
f \otimes g (x_1,x_2) := f(x_1)g(x_2)\quad \forall x_1,x_2 \in \X \label{eq:tensor_def}.
\end{equation}
The inner product of the tensor RKHS $\H_\X \otimes \H_\X$ satisfies
\begin{equation} \label{eq:tensor_inner}
\left< f_1 \otimes g_1, f_2 \otimes g_2 \right>_{\H_\X \otimes \H_\X} = \left< f_1,f_2\right>_{\H_\X}  \left< g_1,g_2\right>_{\H_\X}\quad \forall f_1, f_2, g_1, g_2 \in \H_\X.
\end{equation}
Let $\{ \phi_i \}_{s=1}^I \subset \H_\X$ be complete orthonormal bases of $\H_\X$, where $I \in \mathbb{N} \cup \{ \infty \}$.
Assume $\theta \in \H_\X \otimes \H_\X$ (recall that this is an assumption of Theorem \ref{theo:finite_sample_bound}). Then $\theta$ is expressed as
\begin{equation} 
\theta = \sum_{s,t=1}^I \alpha_{s,t} \phi_s \otimes \phi_t \label{eq:theta_expansion}
\end{equation}
with $\sum_{s,t} | \alpha_{s,t} |^2 < \infty$ (see, e.g.,\ \cite{Aronszajn1950}). 

\begin{proof} [Proof of Theorem \ref{theo:finite_sample_bound}]
Recall that $\hm_{Q} = \sum_{i=1}^n w_i k_\X(\cd,X'_i)$, where $X'_i \sim p( \cd | X_i)\ (i=1,\dots,n)$. Then
\begin{eqnarray}
&& \E_{X'_1,\dots,X'_n}[\| \hm_{Q} - m_{Q} \|_{\H_\X}^2]  \nonumber  \\
&=&  \E_{X'_1,\dots,X'_n}[ \left<\hm_{Q}, \hm_{Q} \right>_{\H_\X} -  2 \left< \hm_{Q}, m_Q \right>_{\H_\X}  + \left< m_Q, m_Q \right>_{\H_\X} ] \nonumber \\
&=&	\sum_{i,j=1}^n w_i w_j \E_{X'_i,X'_j }[k_\X(X'_i,X'_j)] \nonumber \\
&&-2 \sum_{i=1}^n w_i  \E_{X' \sim {Q}, X'_i}[k_\X(X',X'_i)]  + \E_{X',\tX' \sim Q}[k_\X(X',\tX')] \nonumber \\
&=& \sum_{i \neq j} w_i w_j \E_{X'_i,X'_j}[k_\X(X'_i,X'_j)] + \sum_{i=1}^n w_i^2  \E_{X'_i}[k_\X(X'_i,X'_i)] \nonumber \label{eq:expression1} \\
&&-2 \sum_{i=1}^n w_i  \E_{X' \sim {Q}, X'_i}[k_\X(X',X'_i)]  + \E_{X',\tX' \sim Q}[k_\X(X',\tX')]   \label{eq:expression2},
\end{eqnarray} 
where $\tX'$ denotes an independent copy of $X'$.

Recall that $Q = \int p(\cd|x) dP(x)$ and $\theta(x,\tx) := \int \int k_\X(x',\tx') dp(x'|x) dp(\tx' | \tx)$. We can then rewrite terms in (\ref{eq:expression2}) as
\begin{eqnarray*}
&& \E_{X' \sim {Q}, X'_i} [k_\X(X',X'_i)]  \\
&=& \int \left( \int \int k_\X(x',x'_i) dp(x'|x)  dp(x'_i|X_i) \right) dP(x) \\
&=& \int  \theta(x,X_i) dP(x) = \E_{X \sim P} [\theta(X,X_i)]. \\
&& \E_{X', \tX' \sim Q} [k_\X(X',\tX')] \\
&=& \int \int  \left( \int \int k_\X(x',\tx') dp(x'|x)p(\tx'|\tx)  \right)  dP(x) dP(\tx) \\
&=& \int \int \theta(x,\tx)   dP(x) dP(\tx)   = \E_{X,\tX \sim P} [\theta(X,\tX)] .
\end{eqnarray*}
Thus  (\ref{eq:expression2}) is equal to
\begin{eqnarray}
&& \sum_{i=1}^n w_i^2  \left( \E_{X'_i}[k_\X(X'_i,X'_i)] - \E_{X'_i, \tX'_i}[k_\X(X'_i,\tX'_i)] \right) \nonumber \\
&+& \sum_{i,j=1}^n w_i w_j \theta(X_i,X_j) -2 \sum_{i=1}^n w_i \E_{X\sim P}[\theta(X,X_i)] + \E_{X,\tX \sim {P}}[\theta(X,\tX)]  \label{eq:terms_theta}
\end{eqnarray}
We can rewrite terms in (\ref{eq:terms_theta}) as follows, using the facts (\ref{eq:mP_repducing}) (\ref{eq:hmP_repducing}) (\ref{eq:tensor_def}) (\ref{eq:tensor_inner}) (\ref{eq:theta_expansion}):
\begin{eqnarray*}
&&\sum_{i,j} w_i w_j \theta(X_i,X_j) = \sum_{i,j} w_i w_j  \sum_{s,t} \alpha_{s,t} \phi_s (X_i) \phi_t(X_j) \\
&&\quad = \sum_{s,t} \alpha_{s,t} \sum_i w_i \phi_s(X_i) \sum_j w_j \phi_t(X_j) 
= \sum_{s,t} \alpha_{s,t} \left< \hm_P, \phi_s \right>_{\H_\X} \left< \hm_P, \phi_t \right>_{\H_\X} \\
&&\quad  = \sum_{s,t} \alpha_{s,t} \left< \hm_P \otimes \hm_P, \phi_s \otimes \phi_t \right>_{\H_\X \otimes \H_\X} 
= \left< \hm_P \otimes \hm_P, \theta \right>_{\H_\X \otimes \H_\X}. \\
&& \sum_i w_i \E_{X\sim P}[\theta(X,X_i)] =  \sum_i w_i \E_{X\sim P}[ \sum_{s,t} \alpha_{s,t} \phi_s(X) \phi_t(X_i)] \\
&&\quad =  \sum_{s,t} \alpha_{s,t} \E_{X\sim P}[ \phi_s(X)]  \sum_i w_i  \phi_t(X_i) =  \sum_{s,t} \alpha_{s,t} \left< m_P, \phi_s \right>_{\H_\X} \left< \hm_P, \phi_t \right>_{\H_\X} \\
&&\quad =  \sum_{s,t} \alpha_{s,t} \left< m_P \otimes \hm_P, \phi_s \otimes \phi_t \right>_{\H_\X \otimes \H_\X} = \left< m_P \otimes \hm_P,  \theta \right>_{\H_\X \otimes \H_\X}. \\
&& \E_{X,\tX \sim {P}}[\theta(X,\tX)] =  \E_{X,\tX \sim {P}}[\sum_{s,t} \alpha_{s,t} \phi_s(X) \phi_t(\tX)] \\
&&\quad = \sum_{s,t} \alpha_{s,t} \left< m_P, \phi_s \right>_{\H_\X} \left< m_P, \phi_t \right>_{\H_\X} =   \sum_{s,t} \alpha_{s,t} \left< m_P \otimes m_P, \phi_s \otimes \phi_t \right>_{\H_\X \otimes \H_\X} \\
&&\quad = \left< m_P \otimes m_P, \theta \right>_{\H_\X \otimes \H_\X}.
\end{eqnarray*}
Thus (\ref{eq:terms_theta}) is equal to
\begin{eqnarray*}
&& \sum_{i=1}^n w_i^2  \left( \E_{X'_i}[k_\X(X'_i,X'_i)] - \E_{X'_i, \tX'_i}[k_\X(X'_i,\tX'_i)] \right) \\
&& + \left< \hm_P \otimes \hm_P, \theta \right>_{\H_\X \otimes \H_\X} -2 \left< \hm_P \otimes m_P, \theta \right>_{\H_\X \otimes \H_\X} + \left< m_P \otimes m_P, \theta \right>_{\H_\X \otimes \H_\X} \\ 
&=&  \sum_{i=1}^n w_i^2  \left( \E_{X'_i}[k_\X(X'_i,X'_i)] - \E_{X'_i, \tX'_i}[k_\X(X'_i,\tX'_i)] \right) \\
&& + \left< (\hm_P - m_P)\otimes(\hm_P - m_P), \theta \right>_{\H_\X \otimes \H_\X}.
\end{eqnarray*}

Finally, the Cauchy-Schwartz inequality gives
\begin{eqnarray*}
 \left< (\hm_P - m_P)\otimes(\hm_P - m_P), \theta \right>_{\H_\X \otimes \H_\X} \leq   \| \hm_P - m_P \|_{\H_\X}^2  \| \theta \|_{\H_\X \otimes \H_\X}.
\end{eqnarray*}
This completes the proof.
\end{proof}

\section{Proof of Theorem \ref{theo:resampling}} \label{sec:proof_resampling}
Theorem \ref{theo:resampling} provides convergence rates for the resampling algorithm (Algorithm \ref{al:resampling_gen}). 
This theorem assumes that the candidate samples $Z_1,\dots,Z_N$ for resampling are i.i.d.\ with a density $q$. Here we prove Theorem \ref{theo:resampling} by showing that the same statement holds under weaker assumptions (Theorem \ref{theo:resampling_app} below). 

We first describe assumptions. Let $P$ be the distribution of the kernel mean $m_P$, and $L_2(P)$ be the Hilbert space of square-integrable functions on $\X$ with respect to $P$. For any $f \in L_2(P)$, we write its norm by $\| f \|_{L_2(P)} := \int f^2(x) dP(x)$.

\begin{assumption} \label{ass:resampling1}
The candidate samples $Z_1,\dots,Z_N$ are independent. There are probability distributions $Q_1,\dots, Q_N$ on $\X$, such that for any bounded measurable function $g: \X \to \R$, we have 
\begin{equation} \label{eq:resampling1}
 \E \left[ \frac{1}{N-1} \sum_{j \neq i} g(Z_j)  \right] = \E_{X \sim Q_i} [g(X)] \quad (i=1,\dots,N).
\end{equation}
\end{assumption}

\begin{assumption} \label{ass:resampling2}
The distributions $Q_1,\dots, Q_N$ have density functions $q_1, \dots, q_N$, respectively. Define $Q := \frac{1}{N} \sum_{i=1}^N Q_i$ and $q := \frac{1}{N} \sum_{i=1}^N q_i$. There is a constant $A > 0$ that does not depend on $N$, such that  
\begin{equation} \label{eq:resampling2}
 \left\| \frac{q_i}{q} - 1 \right\|_{L_2(P)}^2 \leq \frac{A}{\sqrt{N}} \quad (i=1,\dots,N).
\end{equation}
\end{assumption}

\begin{assumption} \label{ass:resampling3}
The distribution $P$ has a density function $p$ such that $\sup_{x \in \X} \frac{p(x)}{q(x)} < \infty$. 
There is a constant $\sigma > 0$ such that 
\begin{equation} \label{eq:resampling3}
 \sqrt{N}\left( \frac{1}{N} \sum_{i=1}^N \frac{p(Z_i)}{q(Z_i)} - 1 \right) \xrightarrow{D} \mathcal{N}(0,\sigma^2), 
\end{equation}
where $\xrightarrow{D}$ denotes convergence in distribution and $\mathcal{N}(0,\sigma^2)$ the normal distribution with mean $0$ and variance $\sigma^2$.
\end{assumption}

These assumptions are weaker than those in Theorem \ref{theo:resampling}, which require $Z_1,\dots,Z_N$ be i.i.d. For example, Assumption \ref{ass:resampling1} is clearly satisfied for the i.i.d.\ case, since in this case we have $Q = Q_1, = \cdots = Q_N$.
The inequality (\ref{eq:resampling2}) in Assumption \ref{ass:resampling2} requires that the distributions $Q_1,\dots,Q_N$ get similar, as the sample size increases. This is also satisfied under the i.i.d.\ assumption. 
Likewise, the convergence (\ref{eq:resampling3}) in Assumption \ref{ass:resampling3} is satisfied from the central limit theorem if $Z_1,\dots,Z_N$ are i.i.d.

We will need the following lemma.
\begin{lemma} \label{lemma:app_resampling}
Let $Z_1,\dots,Z_N$ be samples satisfying Assumption \ref{ass:resampling1}. Then the following holds for any bounded measurable function $g:\X \to \R$:
\[ \E \left[ \frac{1}{N} \sum_{i=1}^N g(Z_i) \right] =   \int g(x) dQ(x). \]
\begin{proof}
\begin{eqnarray*}
&&\E \left[ \frac{1}{N} \sum_{i=1}^N g(Z_i) \right] = \E\left[ \frac{1}{N (N-1)} \sum_{i=1}^N \sum_{j \neq i} g(Z_j) \right] \\
&=& \frac{1}{N} \sum_{i=1}^N  \E\left[ \frac{1}{N-1} \sum_{j \neq i} g(Z_j) \right] = \frac{1}{N} \sum_{i=1}^N  \int g(x) Q_i(x) = \int g(x) dQ(x).
\end{eqnarray*}
\end{proof}
\end{lemma}

The following theorem shows the convergence rates of our resampling algorithm. 
Note that it does not assume that the candidate samples $Z_1,\dots,Z_N$ are identical to those expressing the estimator $\hm_P$. 
\begin{theo} \label{theo:resampling_app}
Let $k$ be a bounded positive definite kernel, and $\H$ be the associated RKHS.
Let $Z_1,\dots,Z_N$ be candidate samples satisfying Assumptions \ref{ass:resampling1}, \ref{ass:resampling2} and \ref{ass:resampling3}. 
Let $P$ be a probability distribution satisfying Assumption 3, and let $m_P = \int k(\cd,x) dP(x)$ be the kernel mean. Let $\hm_P \in \H$ be any element in $\H$.
Suppose we apply Algorithm \ref{al:resampling_gen} to $\hm_P \in \H$ with candidate samples  $Z_1,\dots,Z_N$, and let  $\bar{X}_1,...,\bar{X}_\ell \in \{ Z_1, \dots, Z_N \}$ be the resulting samples. 
Then the following holds:
\begin{eqnarray*}
 \left\| m_P - \frac{1}{\ell} \sum_{i=1}^\ell k(\cd,\bar{X}_i) \right\|_\H^2 = \left( \| \hm_P - m_P \|_{\H_\X} + O_p(N^{-1/2}) \right)^2 +  O\left( \frac{\ln \ell}{\ell} \right).
\end{eqnarray*}
\end{theo}

\begin{proof}
Our proof is based on the fact \citep{BacJulObo12} that  Kernel Herding can be seen as the Frank-Wolfe optimization method with step size $1/(\ell+1)$ for the $\ell$-th iteration.  For details of the Frank-Wolfe method, we refer to \cite{Jag13,FreGri14} and references therein.

Fix the samples $Z_1, \dots, Z_N$. 
Let $\mathcal{M}_N$ be the convex hull of the set $\{ k(\cd,Z_1), \dots, k(\cd,Z_N) \} \subset \H$. Define a loss function $J: \H \to \R$ by
\begin{equation} \label{eq:def_loss}
J (g) = \frac{1}{2} \| g - \hm_P \|_\H^2,\quad g \in \H
\end{equation}
Then Algorithm \ref{al:resampling_gen} can be seen as the Frank-Wolfe method that iteratively minimizes this loss function over the convex hull $\mathcal{M}_N$:
\[ \inf_{g \in \mathcal{M}_N} J(g). \]
More precisely, the Frank-Wolfe method solves this problem by the following iterations:
\begin{eqnarray*}
s &:=& \argmin_{g \in \mathcal{M}_N} \left< g, \nabla J(g_{\ell - 1})  \right>_{\H} \\
g_{\ell} &:=& (1-\gamma) g_{\ell - 1} + \gamma s \quad (\ell \geq 1),
\end{eqnarray*}
where $\gamma$ is a step size defined as $\gamma = 1/\ell$, and $\nabla J(g_{\ell - 1})$ is the gradient of $J$ at $g_{\ell - 1}$: $\nabla J(g_{\ell - 1}) = g_{\ell - 1} - \hm_P$.  Here the initial point is defined as $g_0 := 0$. It can be easily shown that $g_\ell = \frac{1}{\ell} \sum_{i=1}^\ell k(\cd,\bar{X}_i)$, where $\bX_1,\dots,\bX_\ell$ are the samples given by Algorithm \ref{al:resampling_gen}. For details, see \cite{BacJulObo12}.

Let  $L_{J,\mathcal{M}_N} > 0$ be the Lipschitz constant of the gradient $\nabla J$ over $\mathcal{M}_N$, and ${\rm Diam}\ \mathcal{M}_N > 0$ be the diameter of $\mathcal{M}_N$:
\begin{eqnarray}
L_{J,\mathcal{M}_N} &:=& \sup_{g_1,g_2 \in \mathcal{M}_N} \frac{ \| \nabla J(g_1) - \nabla J(g_2) \|_{\H} }{\|g_1 - g_2 \|_{\H}} \nonumber \\
&=& \sup_{g_1,g_2 \in \mathcal{M}_N} \frac{ \| g_1 - g_2 \|_\H}{ \| g_1 - g_2 \|_\H } = 1, \label{eq:Lipschitz_grad} \\
{\rm Diam}\ \mathcal{M}_N &:=& \sup_{g_1,g_2 \in \mathcal{M}_N} \| g_1 - g_2 \|_{\H} \nonumber \\
&\leq& \sup_{g_1,g_2 \in \mathcal{M}_N} \|g_1\|_\H + \| g_2 \|_\H \leq 2C \label{eq:diam_bound},
\end{eqnarray}
where $C :=  \sup_{x \in \X} \| k(\cd,x) \|_\H = \sup_{x \in \X} \sqrt{k(x,x)} < \infty$.

From Bound 3.2 and Eq. (8) of \cite{FreGri14}, we then have
\begin{eqnarray}
J(g_\ell) - \inf_{g \in \mathcal{M}_N} J(g) 
&\leq& \frac{L_{J,\mathcal{M}_N} ({\rm Diam}\ \mathcal{M}_N)^2 (1+\ln \ell)}{2\ell} \\
&\leq& \frac{ 2 C^2 (1+\ln \ell)}{\ell} \label{eq:upper_FW},
\end{eqnarray}
where the last inequality follows from (\ref{eq:Lipschitz_grad}) and (\ref{eq:diam_bound}).

Note that the upper-bound of (\ref{eq:upper_FW}) does not depend on the candidate samples $Z_1,\dots, Z_N$.
Hence, combined with (\ref{eq:def_loss}), the following holds for any choice of $Z_1,\dots, Z_N$:
\begin{equation} \label{eq:resample_inter_proof}
 \left\| \hm_P - \frac{1}{\ell} \sum_{i=1}^\ell k(\cd,\bar{X}_i) \right\|_\H^2 \leq \inf_{g \in \mathcal{M}_N} \|\hm_P - g \|_\H^2 + \frac{ 4 C^2 (1+\ln \ell)}{\ell}.
\end{equation}

Below we will focus on bounding the first term of (\ref{eq:resample_inter_proof}). Recall here that $Z_1,\dots,Z_N$ are random samples. 
Define a random variable $S_N := \sum_{i=1}^N \frac{p(Z_i)}{q(Z_i)}$.
Since $\mathcal{M}_N$ is the convex hull of the $\{ k(\cd,Z_1), \dots, k(\cd,Z_N) \}$, we have
\begin{eqnarray*}
 && \inf_{g \in \mathcal{M}_N} \|\hm_P - g \|_\H \nonumber \\
 &=& \inf_{\alpha \in \R^N,\ \alpha \geq 0,\ \sum_i \alpha_i \leq 1} \| \hm_P - \sum_i \alpha_i k(\cd,Z_i) \|_\H \nonumber \\
&\leq&  \| \hm_P -  \frac{1}{S_N} \sum_i \frac{p(Z_i)}{q(Z_i)} k(\cd,Z_i) \|_\H \\
&\leq&   \| \hm_P - m_P \|_\H + \| m_P - \frac{1}{N} \sum_i \frac{p(Z_i)}{q(Z_i)} k(\cd,Z_i) \|_\H \nonumber \\
&&+   \|   \frac{1}{N} \sum_i \frac{p(Z_i)}{q(Z_i)} k(\cd,Z_i) -  \frac{1}{S_N} \sum_i \frac{p(Z_i)}{q(Z_i)} k(\cd,Z_i) \|_\H \nonumber.
\end{eqnarray*}
Therefore we have
\begin{eqnarray}
&& \| \hm_P - \frac{1}{\ell} \sum_{i=1}^\ell k(\cd,\bar{X}_i) \|_\H^2 \nonumber \\
&\leq& (  \| \hm_P - m_P \|_\H +  \| m_P - \frac{1}{N} \sum_i \frac{p(Z_i)}{q(Z_i)} k(\cd,Z_i) \|_\H \nonumber \\
&&+  \|   \frac{1}{N} \sum_i \frac{p(Z_i)}{q(Z_i)} k(\cd,Z_i) -  \frac{1}{S_N} \sum_i \frac{p(Z_i)}{q(Z_i)} k(\cd,Z_i) \|_\H )^2  +  O\left( \frac{ \ln \ell}{\ell} \right)\label{eq:proof_upper}.
\end{eqnarray}
Below we derive rates of convergence for the second and third terms.

\paragraph{Second term.}
We derive a rate of convergence in expectation, which implies a rate of convergence in probability.
To this end, we use the following fact: Let $f \in \H$ be any function in the RKHS.
By the assumption $\sup_{x \in \X} \frac{p(x)}{q(x)} < \infty$ and the boundedness of $k$, functions $x \to \frac{p(x)}{q(x)} f(x)$ and  $x \to \left( \frac{p(x)}{q(x)}\right)^2 f(x)$ are bounded.
\begin{eqnarray*}
&& \E[ \| m_P - \frac{1}{N} \sum_i \frac{p(Z_i)}{q(Z_i)} k(\cd,Z_i) \|_\H^2] \nonumber \\
&=& \| m_P \|_\H^2 - 2 \E [ \frac{1}{N} \sum_i \frac{p(Z_i)}{q(Z_i)} m_P(Z_i)  ] + \E[\frac{1}{N^2} \sum_i \sum_j \frac{p(Z_i)}{q(Z_i)} \frac{p(Z_j)}{q(Z_j)} k(Z_i,Z_j) ] \nonumber \\
&=& \| m_P \|_\H^2 - 2 \int \frac{p(x)}{q(x)} m_P(x) q(x) dx + \E [\frac{1}{N^2} \sum_i \sum_{j \neq i} \frac{p(Z_i)}{q(Z_i)} \frac{p(Z_j)}{q(Z_j)} k(Z_i,Z_j) ] \nonumber \\
&& + \E [ \frac{1}{N^2} \sum_i \left(\frac{p(Z_i)}{q(Z_i)}  \right)^2 k(Z_i,Z_i) ] \nonumber \\
&=& \| m_P \|_\H^2 - 2\| m_P \|_\H^2 + \E [ \frac{N-1}{N^2} \sum_i \frac{p(Z_i)}{q(Z_i)} \int \frac{p(x)}{q(x)} k(Z_i,x) q_i(x) dx ] \nonumber  \\
&&+ \frac{1}{N} \int \left(\frac{p(x)}{q(x)} \right)^2 k(x,x) q(x) dx \nonumber \\
&=& - \| m_P \|_\H^2 + \E [ \frac{N-1}{N^2} \sum_i \frac{p(Z_i)}{q(Z_i)} \int \frac{p(x)}{q(x)} k(Z_i,x) q_i(x) dx ]  + \frac{1}{N} \int \frac{p(x)}{q(x)} k(x,x) dP(x). \label{eq:second_bound}
\end{eqnarray*}

We further rewrite the second term of the last equality as follows:
\begin{eqnarray*}
&& \E [ \frac{N-1}{N^2} \sum_i \frac{p(Z_i)}{q(Z_i)} \int \frac{p(x)}{q(x)} k(Z_i,x) q_i(x) dx ] \\
&=& \E [ \frac{N-1}{N^2} \sum_i \frac{p(Z_i)}{q(Z_i)} \int \frac{p(x)}{q(x)} k(Z_i,x) ( q_i(x) - q(x) )dx ] \\
&&+ \E [ \frac{N-1}{N^2} \sum_i \frac{p(Z_i)}{q(Z_i)} \int \frac{p(x)}{q(x)} k(Z_i,x) q(x) dx ] \\
&=&  \E [ \frac{N-1}{N^2} \sum_i \frac{p(Z_i)}{q(Z_i)} \int \sqrt{p(x)} k(Z_i,x) \sqrt{p(x)} ( \frac{q_i(x)}{q(x)} - 1 )dx ] + \frac{N-1}{N} \| m_P \|_\H^2 \\
&\leq&   \E [ \frac{N-1}{N^2} \sum_i \frac{p(Z_i)}{q(Z_i)} \| k(Z_i, \cd) \|_{L_2(P)} \| \frac{q_i(x)}{q(x)} - 1  \|_{L_2(P)} ] + \frac{N-1}{N} \| m_P \|_\H^2 \\
&\leq&   \E [ \frac{N-1}{N^3} \sum_i \frac{p(Z_i)}{q(Z_i)} C^2 A ] + \frac{N-1}{N} \| m_P \|_\H^2 \\
&=&  \frac{C^2 A (N-1)}{N^2} + \frac{N-1}{N} \| m_P \|_\H^2,
\end{eqnarray*}
where the first inequality follows from Cauchy-Schwartz. Using this, we obtain
\begin{eqnarray*}
&& \E[ \| m_P - \frac{1}{N} \sum_i \frac{p(Z_i)}{q(Z_i)} k(\cd,Z_i) \|_\H^2 \\
&\leq&  \frac{1}{N} \left( \int \frac{p(x)}{q(x)} k(x,x) dP(x) - \| m_P \|_\H^2 \right) + \frac{C^2(N-1) A}{N^2} \\
&=& O(N^{-1}).
\end{eqnarray*}
Therefore we have 
\begin{equation} \label{eq:resampling_second}
 \| m_P - \frac{1}{N} \sum_i \frac{p(Z_i)}{q(Z_i)} k(\cd,Z_i) \|_\H = O_p(N^{-1/2}) \quad (N \to \infty).
\end{equation}

\paragraph{Third term.}
We can bound the third term as follows:
\begin{eqnarray*}
&& \left\|   \frac{1}{N} \sum_i \frac{p(Z_i)}{q(Z_i)} k(\cd,Z_i) -  \frac{1}{S_N} \sum_i \frac{p(Z_i)}{q(Z_i)} k(\cd,Z_i) \right\|_\H \\
&=&  \left\|  \frac{1}{N} \sum_i \frac{p(Z_i)}{q(Z_i)} k(\cd,Z_i) \left( 1 - \frac{N}{S_N} \right) \right\|_\H  \\
&=& \left| 1- \frac{N}{S_N} \right|  \left\|  \frac{1}{N} \sum_i \frac{p(Z_i)}{q(Z_i)} k(\cd,Z_i) \right\|_\H  \\
&\leq& \left| 1- \frac{N}{S_N}  \right|  C\ \| p/q \|_\infty  \\
&=& \left| 1 - \frac{1}{\frac{1}{N} \sum_{i=1}^N p(Z_i) / q(Z_i)} \right| C\ \| p/q \|_\infty,
\end{eqnarray*}
where $\| p/q \|_\infty := \sup_{x \in \X} \frac{p(x)}{q(x)} < \infty$.
Therefore the following holds by Assumption 3 and the Delta method:
\begin{equation} \label{eq:resampling_third}
 \left\|   \frac{1}{N} \sum_i \frac{p(Z_i)}{q(Z_i)} k(\cd,Z_i) -  \frac{1}{S_N} \sum_i \frac{p(Z_i)}{q(Z_i)} k(\cd,Z_i) \right\|_\H = O_p( N^{-1/2} ).
\end{equation}

The assertion of the theorem follows from (\ref{eq:proof_upper}) (\ref{eq:resampling_second}) (\ref{eq:resampling_third}).
\end{proof}


\section{Reduction of computational cost} \label{sec:speed_up}
We have seen in Section \ref{sec:overview_complexity} that the time complexity of KMCF in one time step is $O(n^3)$, where $n$ is the number of the state-observation examples $\{ (X_i, Y_i) \}_{i=1}^n$. 
This can be costly if one wishes to use KMCF in real-time applications with a large number of samples.
Here we show two methods for reducing the costs: one based on low rank approximation of kernel matrices, and one based on Kernel Herding. 
Note that Kernel Herding is also used in the resampling step.
The purpose here is different, however: we make use of Kernel Herding for finding a reduced representation of the data $\{ (X_i, Y_i) \}_{i=1}^n$.

\subsection{Low rank approximation of kernel matrices} \label{sec:low_rank}
Our goal is to reduce the costs of Algorithm \ref{al:KBR_simple} of Kernel Bayes' Rule.
Algorithm \ref{al:KBR_simple} involves two matrix inversions: $(G_X + n \varepsilon I_n)^{-1}$ in Line 3 and $( (\Lambda G_Y)^2 + \delta I_n)^{-1}$ in Line 4.
Note that $(G_X + n \varepsilon I_n)^{-1}$ does not involve the test data, so can be computed before the test phase. 
On the other hand, $( (\Lambda G_Y)^2 + \delta I_n)^{-1}$ depends on matrix $\Lambda$.
This matrix involves the vector ${\bf m}_\pi$, which essentially represents the prior of the current state (see Line 13 of Algorithm \ref{al:KBRPF}).
Therefore $( (\Lambda G_Y)^2 + \delta I_n)^{-1}$ needs to be computed for each iteration in the test phase.
This has complexity of $O(n^3)$.
Note that even if $(G_X + n \varepsilon I_n)^{-1}$ can be computed in the training phase, the multiplication $(G_X + n \varepsilon I_n)^{-1} {\bf m}_\pi$ in Line 3 requires $O(n^2)$.
Thus it can also be costly.
Here we consider methods to reduce both costs in Line 3 and 4.

Suppose that there exist low rank matrices $U, V \in \R^{n \times r}$, where $r < n$, that approximate the kernel matrices: $G_X \approx U U^T$, $G_Y \approx V V^T$. 
Such low rank matrices can be obtained by, for example, incomplete Cholesky decomposition with time complexity $O(n r^2)$ \citep{FinSch01,BacJor02}. 
Note that the computation of these matrices are only required once before the test phase. 
Therefore their time complexities are not the problem here.

\paragraph{Derivation.}
First, we approximate $(G_X + n \varepsilon I_n)^{-1} {\bf m}_\pi$ in Line 3 using $G_X \approx U U^T$.
By the Woodbury identity, we have
\begin{eqnarray*}
  (G_X + n \varepsilon I_n)^{-1} {\bf m}_\pi
  &\approx& (U U^T + n \varepsilon I_n)^{-1}  {\bf m}_\pi \\
  &=& \frac{1}{n \varepsilon} (I_n - U (n \varepsilon I_r + U^T U)^{-1} U^T)  {\bf m}_\pi,
\end{eqnarray*}
where $I_r \in \R^{r \times r}$ denotes the identity.
Note that $(n \varepsilon I_r + U^T U)^{-1}$ does not involve the test data, so can be computed in the training phase. 
Thus the above approximation of $\mu$ can be computed with complexity $O(n r^2)$.

Next, we approximate $w =  \Lambda G_Y ( (\Lambda G_Y)^2 + \delta I)^{-1} \Lambda {\bf k}_Y$ in Line 4 using $G_Y \approx V V^T$.
Define $B= \Lambda V \in \R^{n \times r}$, $C=V^T \Lambda V \in \R^{r \times r}$, and $D = V^T \in \R^{r \times n}$. Then $(\Lambda G_Y)^2 \approx (\Lambda V V^T)^2 = BCD$.
By the Woodbury identity, we obtain
\begin{eqnarray*} 
(\delta I_n + (\Lambda G_Y)^2 )^{-1} &\approx& (\delta I_n +  BCD )^{-1} \\
&=& \frac{1}{\delta} (I_n - B(\delta C^{-1} + DB)^{-1}D).
\end{eqnarray*}
Thus $w$ can be approximated as
\begin{eqnarray*}
w &=& \Lambda G_Y ( (\Lambda G_Y)^2 + \delta I)^{-1} \Lambda {\bf k}_Y \\
&\approx & \frac{1}{\delta} \Lambda V V^T (I_n - B(\delta C^{-1} + DB)^{-1} D) \Lambda {\bf k}_Y.
\end{eqnarray*}
The computation of this approximation requires $O(n r^2 + r^3) = O(n r^2)$. 
Thus in total, the complexity of Algorithm \ref{al:KBR_simple} can be reduced to $O(nr^2)$.
We summarize the above approximations in Algorithm \ref{al:KBR_lowrank}.

\paragraph{How to use.}
Algorithm \ref{al:KBR_lowrank} can be used with Algorithm \ref{al:KBRPF} of KMCF, by modifying Algorithm \ref{al:KBRPF} in the following manner:
(i) Compute the low rank matrices $U, V$ right after Line 4-5. This can be done by using, for example, incomplete Cholesky decomposition \citep{FinSch01,BacJor02}; (ii) Replace Algorithm \ref{al:KBR_simple} in Line 15 by Algorithm \ref{al:KBR_lowrank}.

\paragraph{How to select the rank.}
As discussed in Section \ref{sec:overview_complexity}, one way of selecting the rank $r$ is to use cross validation, by regarding $r$ as a hyper parameter of KMCF. 
Another way is to measure the approximation errors $\| G_X - UU^T \|$ and $\| G_Y -V V^T \|$ with some matrix norm, such as the Frobenius norm.
Indeed, we can compute the smallest rank $r$ such that these errors are below a prespecified threshold, and this can be done efficiently with time complexity $O(n r^2)$ \citep{BacJor02}.

\begin{algorithm}[t]
\caption{Low Rank Approximation of Kernel Bayes' Rule}
\label{al:KBR_lowrank}
\begin{algorithmic}[1]
\STATE
{\bf Input:} ${\bf k}_Y, {\bf m}_\pi \in \R^n$, $U, V \in \R^{n \times r}$, $\varepsilon,\delta > 0$.
\STATE {\bf Output:}  $w := (w_{1},\dots,w_{n})^T \in \R^n$.
\\ \hrulefill
\STATE  $\Lambda \leftarrow  {\rm diag} (  \frac{1}{n \varepsilon} (I_n - U (n \varepsilon I_r + U^T U)^{-1} U^T)  {\bf m} ) \in \R^{n \times n}$.
\STATE $B \leftarrow \Lambda V \in \R^{n \times r}$,\quad $C \leftarrow V^T \Lambda V \in \R^{r \times r}$,\quad $D \leftarrow V^T \in \R^{r \times n}$.
\STATE $w \leftarrow \frac{1}{\delta} \Lambda V V^T (I_n - B(\delta C^{-1} + DB)^{-1} D) \Lambda {\bf k}_Y \in \R^n$.

\end{algorithmic}
\end{algorithm}

\subsection{Data reduction with Kernel Herding} \label{sec:subsampling}
Here we describe an approach to reduce the size of the representation of the state-observation examples $\{ (X_i,Y_i) \}_{i=1}^n$ in an efficient way. 
By ``efficient", we mean that the information contained in $\{ (X_i,Y_i) \}_{i=1}^n$ will be preserved even after the reduction.
Recall that $\{ (X_i,Y_i) \}_{i=1}^n$ contains the information of the observation model $p(y_t | x_t)$ (recall also that $p(y_t | x_t)$ is assumed time-homogeneous; see Section \ref{sec:notation}).
This information is only used in Algorithm \ref{al:KBR_simple} of Kernel Bayes' Rule (Line 15, Algorithm \ref{al:KBRPF}). 
Therefore it suffices to consider how Kernel Bayes' Rule accesses the information contained in the joint sample $\{ (X_i,Y_i) \}_{i=1}^n$.

\paragraph{Representation of the joint sample.}
To this end, we need to show how the joint sample $\{ (X_i,Y_i) \}_{i=1}^n$ can be represented with a kernel mean embedding.
Recall that $(k_\X,\H_\X)$ and $(k_\Y,\H_\Y)$ are kernels and the associated RKHSs on the state space $\X$ and the observation space $\Y$, respectively.
Let $\X \times \Y$ be the product space of $\X$ and $\Y$.
Then we can define a kernel $k_{\X \times \Y}$ on $\X \times \Y$ as the product of $k_\X$ and $k_\Y$:
 $k_{\X \times \Y} ((x,y), (x',y')) = k_\X(x,x')k_\Y(y,y')$ for all $(x,y), (x',y') \in \X \times \Y$.
 This product kernel $k_{\X \times \Y}$ defines an RKHS of $\X \times \Y$: let $\H_{\X \times \Y}$ denote this RKHS.
As in Section \ref{sec:background}, we can use $k_{\X \times \Y}$ and $\H_{\X \times \Y}$ for a kernel mean embedding.
In particular, the empirical distribution $\frac{1}{n} \sum_{i=1}^n \delta_{(X_i,Y_i)}$ of the joint sample $\{ (X_i,Y_i) \}_{i=1}^n \subset \X \times \Y$ can be represented as an empirical kernel mean in $\H_{\X \times \Y}$:
\begin{equation} \label{eq:joint_embedding_estimate}
 \hm_{XY} := \frac{1}{n} \sum_{i=1}^n k_{\X \times \Y} ( (\cd,\cd), (X_i,Y_i)) \in \H_{\X \times \Y}.
\end{equation}
This is the representation of the joint sample $\{ (X_i,Y_i) \}_{i=1}^n$.

The information of $\{ (X_i,Y_i) \}_{i=1}^n$ is provided for Kernel Bayes' Rule essentially through this form (\ref{eq:joint_embedding_estimate}) \citep{FukSonGre11,FukSonGre13}.
Recall that (\ref{eq:joint_embedding_estimate}) is a point in the RKHS $\H_{\X \times \Y}$.
Any point close to (\ref{eq:joint_embedding_estimate}) in $\H_{\X \times \Y}$ would also contain information close to that contained in (\ref{eq:joint_embedding_estimate}).
Therefore, we propose to find a subset $\{ (\bX_1,\bY_1),\dots (\bX_r, \bX_r) \} \subset \{ (X_i,Y_i) \}_{i=1}^n$, where $r < n$, such that its representation in $\H_{\X \times \Y}$
\begin{equation} \label{eq:kernelmean_subsamples}
 \bar{m}_{XY} := \frac{1}{r} \sum_{i=1}^r  k_{\X \times \Y} ( (\cd,\cd), (\bX_i,\bY_i)) \in \H_{\X \times \Y}
\end{equation}
is close to (\ref{eq:joint_embedding_estimate}).
Namely, we wish to find subsamples such that $\| \bar{m}_{XY} - \hm_{XY} \|_{\H_{\X \times \Y}}$ is small.
If the error $\| \bar{m}_{XY} - \hm_{XY} \|_{\H_{\X \times \Y}}$ is small enough, (\ref{eq:kernelmean_subsamples}) would provide information close to that given by (\ref{eq:joint_embedding_estimate}) for Kernel Bayes' Rule.
Thus Kernel Bayes' Rule based on such subsamples $\{ (\bX_i,\bY_i) \}_{i=1}^r$ would not perform much worse than the one based on the entire set of samples $ \{ (X_i,Y_i) \}_{i=1}^n$.

\paragraph{Subsampling method.}
To find such subsamples, we make use of Kernel Herding in Section \ref{sec:kernel_herding}. 
Namely, we apply the update equations (\ref{eq:herding_update1}) (\ref{eq:herding_update2}) to approximate (\ref{eq:joint_embedding_estimate}), with kernel $k_{\X \times \Y}$ and RKHS $\H_{\X \times \Y}$.
We greedily find subsamples $\bar{\bf D}_r := \{ (\bX_1, \bY_1), \dots, (\bX_r,\bY_r) \}$ as
\begin{eqnarray*}
 (\bX_{r}, \bY_{r}) &=& \arg \max_{(x,y) \in {\bf D} / \bar{\bf D}_{r-1}  } \frac{1}{n} \sum_{i=1}^n k_{\X \times \Y} \left( (x,y), (X_i,Y_i) \right) - \frac{1}{r} \sum_{j=1}^{r-1} k_{\X \times \Y} \left( (x,r), (\bX_i,\bY_i) \right) \\
 &=&  \arg \max_{(x,y) \in {\bf D} / \bar{\bf D}_{r-1}  } \frac{1}{n} \sum_{i=1}^n k_\X(x,X_i) k_\Y(y,Y_i) - \frac{1}{r} \sum_{j=1}^{r-1} k_\X(x, \bX_j) k_\Y(y,\bY_j).
\end{eqnarray*}
The resulting algorithm is shown in Algorithm \ref{al:subsampling}. 
The time complexity is $O(n^2 r)$ for selecting $r$ subsamples.

\paragraph{How to use.}
By using Algorithm \ref{al:subsampling}, we can reduce the the time complexity of KMCF (Algorithm \ref{al:KBRPF}) in each iteration from $O(n^3)$ to $O(r^3)$. This can be done as follows: (i) Obtain subsamples $\{ (\bX_i, \bY_i) \}_{i=1}^r$ by applying Algorithm \ref{al:subsampling} to $\{ (X_i,Y_i) \}_{i=1}^n$; (ii) Replace $\{ (X_i,Y_i) \}_{i=1}^n$ in Requirement of Algorithm \ref{al:KBRPF} by $\{ (\bX_i, \bY_i) \}_{i=1}^r$, and use the number $r$ instead of $n$.

\paragraph{How to select the number of subsamples.}
The number $r$ of subsamples determine the tradeoff between the accuracy and computational time of KMCF.
It may be selected by cross validation, or by measuring the approximation error $\| \bar{m}_{XY} - \hm_{XY} \|_{\H_{\X \times \Y}}$, as for the case of selecting the rank of low rank approximation in Appendix \ref{sec:low_rank}.

\paragraph{Discussion.}
Recall that Kernel Herding generates samples such that they approximate a given kernel mean (see  Section \ref{sec:kernel_herding}).
Under certain assumptions, the error of this approximation is of $O(r^{-1})$ with $r$ samples, which is faster than that of i.i.d.\ samples $O(r^{-1/2})$.
This indicates that subsamples $\{ (\bX_i, \bY_i) \}_{i=1}^r$ selected with Kernel Herding may approximate (\ref{eq:joint_embedding_estimate}) well. 
Here, however, we find the solutions of the optimization problems (\ref{eq:herding_update1}) (\ref{eq:herding_update2}) from the finite set $\{ (X_i, Y_i) \}_{i=1}^n$, rather than the entire joint space $\X \times \Y$.
The convergence guarantee is only provided for the case of the entire joint space $\X \times \Y$.
Thus for our case the convergence guarantee is no longer provided. 
Moreover, the fast rate $O(r^{-1})$ is only guaranteed for finite dimensional RKHSs. 
Gaussian kernels, which we often use in practice, define infinite dimensional RKHSs. 
Therefore the fast rate is not guaranteed if we use Gaussian kernels.
Nevertheless, we can use Algorithm  \ref{al:subsampling} as a heuristic for data reduction.

\begin{algorithm}[t]
\caption{Subsampling with Kernel Herding}
\label{al:subsampling}
\begin{algorithmic}[1]
\STATE
{\bf Input:} (i)  ${\bf D} := \{(X_i,Y_i) \}_{i=1}^n$. (ii) size of subsamples $r$.
\STATE {\bf Output:} subsamples $\bar{\bf D}_r := \{ (\bX_1,\bY_1), \dots, (\bX_r, \bY_r) \} \subset {\bf D}$. 
\\ \hrulefill
	\STATE Select $(\bX_1, \bY_1)$ as follows and let $\bar{\bf D}_1 := \{ (\bX_1, \bY_1) \}$:
	\[ (\bX_1, \bY_1) =  \argmax_{(x,y) \in {\bf D}} \frac{1}{n} \sum_{i=1}^n k_\X(x,X_i) k_\Y(y,Y_i) \]
	\FOR{$N=2$ to $r$}
		\STATE Select $(\bX_N,\bY_N)$ as follows and let $\bar{\bf D}_N := \bar{\bf D}_{N-1}  \cup \{ (\bX_N, \bY_N) \}$:
		\[ (\bX_N,\bY_N) =   \argmax_{(x,y) \in {\bf D} / \bar{\bf D}_{N-1}  }  \frac{1}{n} \sum_{i=1}^n k_\X(x,X_i) k_\Y(u,Y_i) - \frac{1}{N} \sum_{j=1}^{N-1} k_\X(x, \bX_j) k_\Y(y,\bY_j) \]
	\ENDFOR \\
\end{algorithmic}
\end{algorithm}

\end{document}